\documentclass[11pt,letterpaper, logo]{mystyle}

\usepackage[utf8]{inputenc} % allow utf-8 input
\usepackage[T1]{fontenc}    % use 8-bit T1 fonts
\usepackage{hyperref}       % hyperlinks
\usepackage{url}            % simple URL typesetting
\usepackage{booktabs}       % professional-quality tables
\usepackage{amsfonts}       % blackboard math symbols
\usepackage{nicefrac}       % compact symbols for 1/2, etc.
\usepackage{microtype}      % microtypography
\usepackage{lipsum}
\usepackage{graphicx}
\usepackage{enumitem}
\usepackage[numbers]{natbib}
\usepackage{fontawesome5}
\usepackage{amsmath}
\usepackage{twemojis}
\usepackage{bxcoloremoji}
\usepackage[table]{xcolor}
\usepackage{multirow}
\usepackage{makecell}

\usepackage{subcaption}
\usepackage{wrapfig}
\usepackage{algorithm}
\usepackage{algorithmic}

\usepackage{booktabs,tabularx}

\newtheorem{theorem}{Theorem}
\newtheorem{assumption}[theorem]{Assumption}
\newtheorem{definition}[theorem]{Definition}
\newtheorem{lemma}[theorem]{Lemma}

\newtheorem{proposition}[theorem]{Proposition}

\newtcbox{\agentchip}{on line, arc=2pt, colback=blue!8, colframe=blue!25,
  boxsep=0.7pt, left=2pt, right=2pt, top=0.2ex, bottom=0.2ex}

\definecolor{AgentIndigo}{HTML}{3F51B5}
\definecolor{AgentIndigoLight}{HTML}{E8EAF6}

\tcbset{
  agentdef/.style={
    colback=AgentTealLight,
    colframe=AgentTeal,
    colbacktitle=AgentTeal!20!white,
    coltitle=black,
    boxrule=0.9pt, arc=2mm,
    left=3mm, right=3mm, top=2mm, bottom=2mm,
    fonttitle=\bfseries,
    title=Definition of Agentic Reasoning
  },
  agentscope/.style={
    colback=AgentIndigoLight,
    colframe=AgentIndigo,
    colbacktitle=AgentIndigo!20!white,
    coltitle=AgentIndigo,
    boxrule=0.9pt, arc=2mm,
    left=3mm, right=3mm, top=2mm, bottom=2mm,
    fonttitle=\bfseries,
    title=Survey Scope
  },
  agentcontrib/.style={
    colback=AgentAmberLight,
    colframe=AgentAmber,
    colbacktitle=AgentAmber!20!white,
    coltitle=black,
    boxrule=0.9pt, arc=2mm,
    left=3mm, right=3mm, top=2mm, bottom=2mm,
    fonttitle=\bfseries,
    title=Contributions
  },
   agentstruct/.style={
      colback=LARGBlue!12!white,
      colframe=LARGBlue,
      colbacktitle=LARGBlue!20!white,
      coltitle=LARGBlue,
    boxrule=0.9pt, arc=2mm,
    left=3mm, right=3mm, top=2mm, bottom=2mm,
    fonttitle=\bfseries,
    title=Survey Structure
  }
}

\newtcolorbox{takeawaybox}[1]{
  colback=white!98!black,
  colframe=white!86!black,
  title={\textcolor{black}{#1}},
  boxrule=0.8pt,
  arc=2pt,
  left=6pt,
  right=6pt,
  top=0pt,
  bottom=0pt,
  before skip=5pt,
}

\newtcolorbox{definitionbox}{
  colback=gray!10,
  colframe=black!70,
  boxrule=0.5pt,
  arc=4pt,
  left=6pt,
  right=6pt,
  top=6pt,
  bottom=6pt
}

\graphicspath{ {./images/} }

\setlist[itemize]{leftmargin=12pt}

\runningtitle{Heterogeneous Scientific Foundation Model Collaboration}

\title{
\textbf{Heterogeneous Scientific Foundation Model Collaboration}\\
\vspace{0.3em}
}

\date{\vspace{-3ex}}

\begin{document}

\author{
    \normalfont 
    \vspace{-3mm}
    \begin{minipage}{0.85\linewidth}
    \centering
    \textbf{
    Zihao Li, \quad Jiaru Zou, \quad Feihao Fang, \quad Xuying Ning, \quad
    Mengting Ai, \quad Tianxin Wei, \quad Sirui Chen, \quad Xiyuan Yang, \quad
    Jingrui He
    }
    \end{minipage}
   \\
   \vspace{6pt}

   \raisebox{1.0ex}{\includegraphics[height=1.1ex]{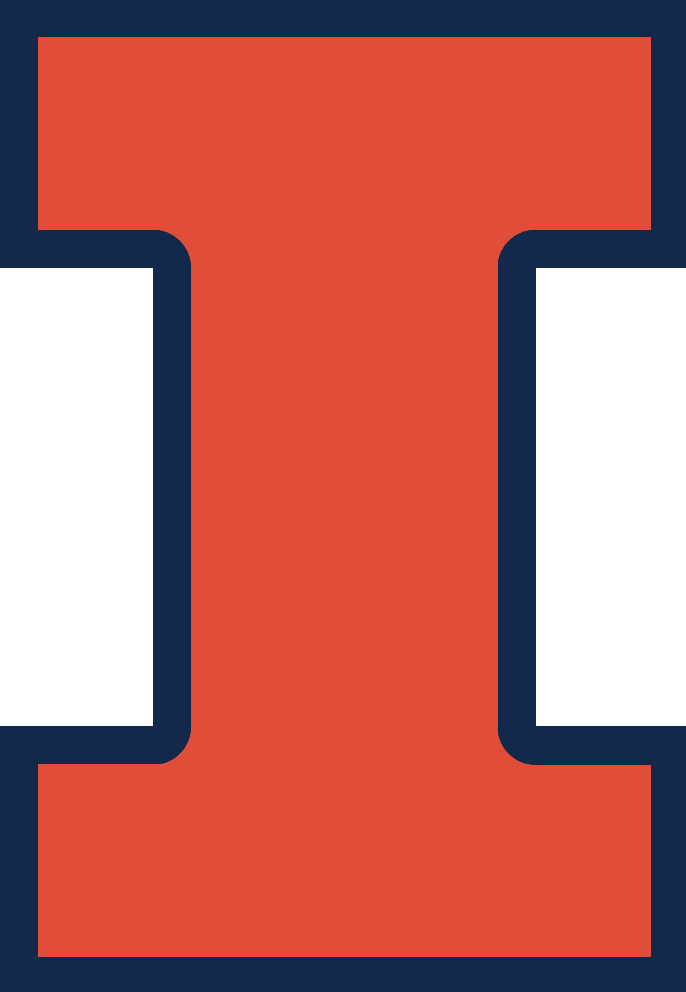}}\,University of Illinois Urbana-Champaign  \\

   \faGithub~\textbf{Code}: \href{https://github.com/Violet24K/Eywa}{https://github.com/Violet24K/Eywa} 
   \quad
   {\href{https://www.zihao.website/eywa.github.io/}{ [\faLink~Project Page]}}
   \vspace{-3mm}
}

\begin{abstract}

\textbf{\large Abstract:} 
Agentic large language model systems have demonstrated strong capabilities.
However, their reliance on language as the universal interface fundamentally limits their applicability to many real-world problems, especially in scientific domains where 
domain-specific foundation models have been developed to address specialized tasks beyond natural language. 
In this work, we introduce \textit{Eywa}, a heterogeneous agentic framework designed to extend language-centric systems to a broader class of scientific foundation models. The key idea of Eywa is to augment domain-specific foundation models with a language-model-based reasoning interface, enabling language models to guide inference over non-linguistic data modalities. This design allows predictive foundation models, which are typically optimized for specialized data and tasks, to participate in higher-level reasoning and decision-making processes within agentic systems.
Eywa can serve as a drop-in replacement for a single-agent pipeline (\textit{EywaAgent}) or be integrated into existing multi-agent systems by replacing traditional agents with specialized agents (\textit{EywaMAS}). We further investigate a planning-based orchestration framework in which a planner dynamically coordinates traditional agents and Eywa agents to solve complex tasks across heterogeneous data modalities (\textit{EywaOrchestra}).
We evaluate Eywa across a diverse set of scientific domains spanning physical, life, and social sciences. 
Experimental results demonstrate that Eywa improves performance on tasks involving structured and domain-specific data, while reducing reliance on language-based reasoning through effective collaboration with specialized foundation models.

\vspace{2mm}
${\coloremojicode{2709}}$ \textbf{Contact}: \href{mailto:zihaoli5@illinois.edu}{{\color{blue}zihaoli5@illinois.edu}}, \href{mailto:jingrui@illinois.edu}{{\color{blue}jingrui@illinois.edu}}

\end{abstract}
\maketitle

\begin{figure}[!h]
    \centering
    \includegraphics[width=\linewidth]{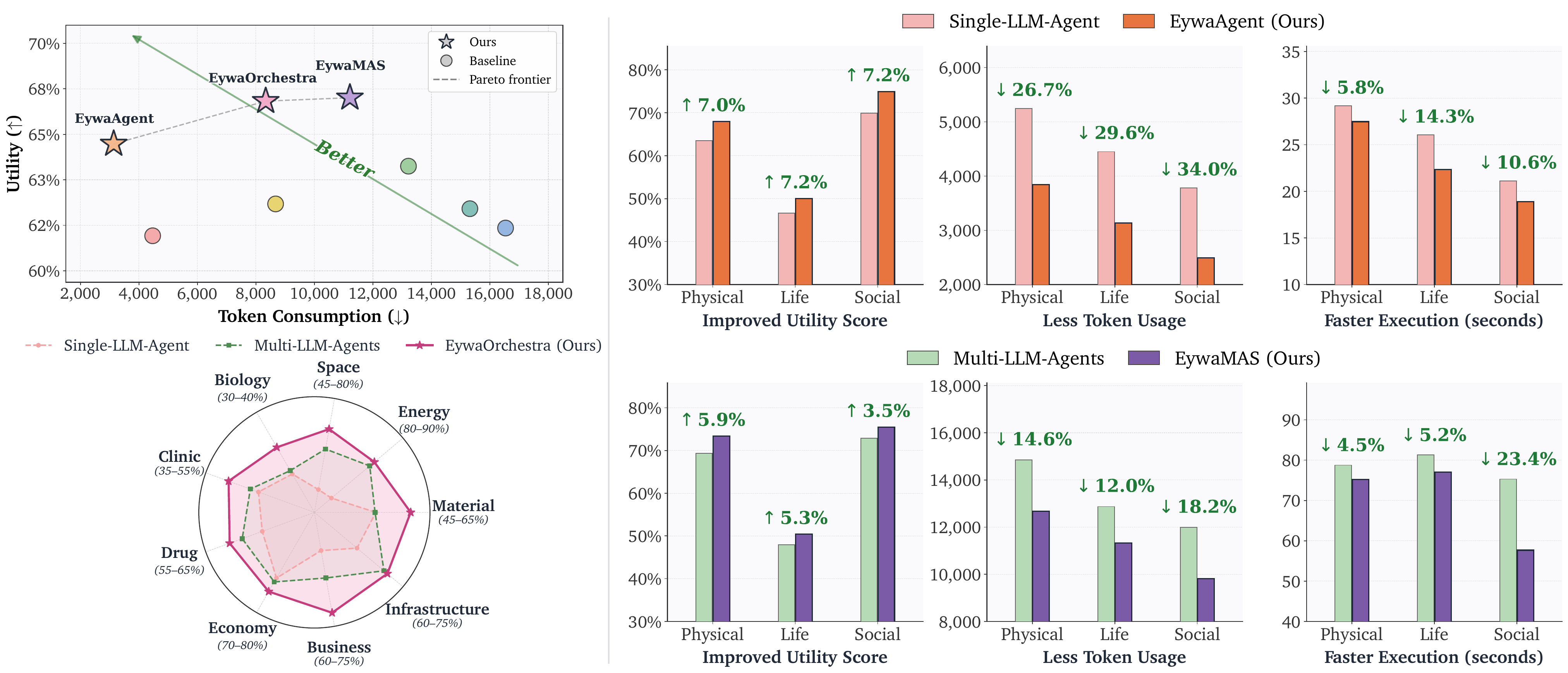}
    \vspace{-6mm}
    \caption{
    \textbf{\textit{Eywa} extends current agentic systems.}
    \textbf{(Left)} Overall comparisons show that \textit{EywaAgent}, \textit{EywaMAS}, and \textit{EywaOrchestra} achieve higher utility with lower token consumption than language-only baselines.
    \textbf{(Right)} Category-level results further show consistent gains in utility, token efficiency, and execution time across physical, life, and social science tasks.
    }
    \vspace{-3mm}
    \label{fig:hero}
\end{figure}

\addtocontents{toc}
{\protect\setcounter{tocdepth}{-1}}

\section{Introduction}
\label{sec:intro}

\begin{figure*}[t]
    \centering
    \includegraphics[width=\linewidth]{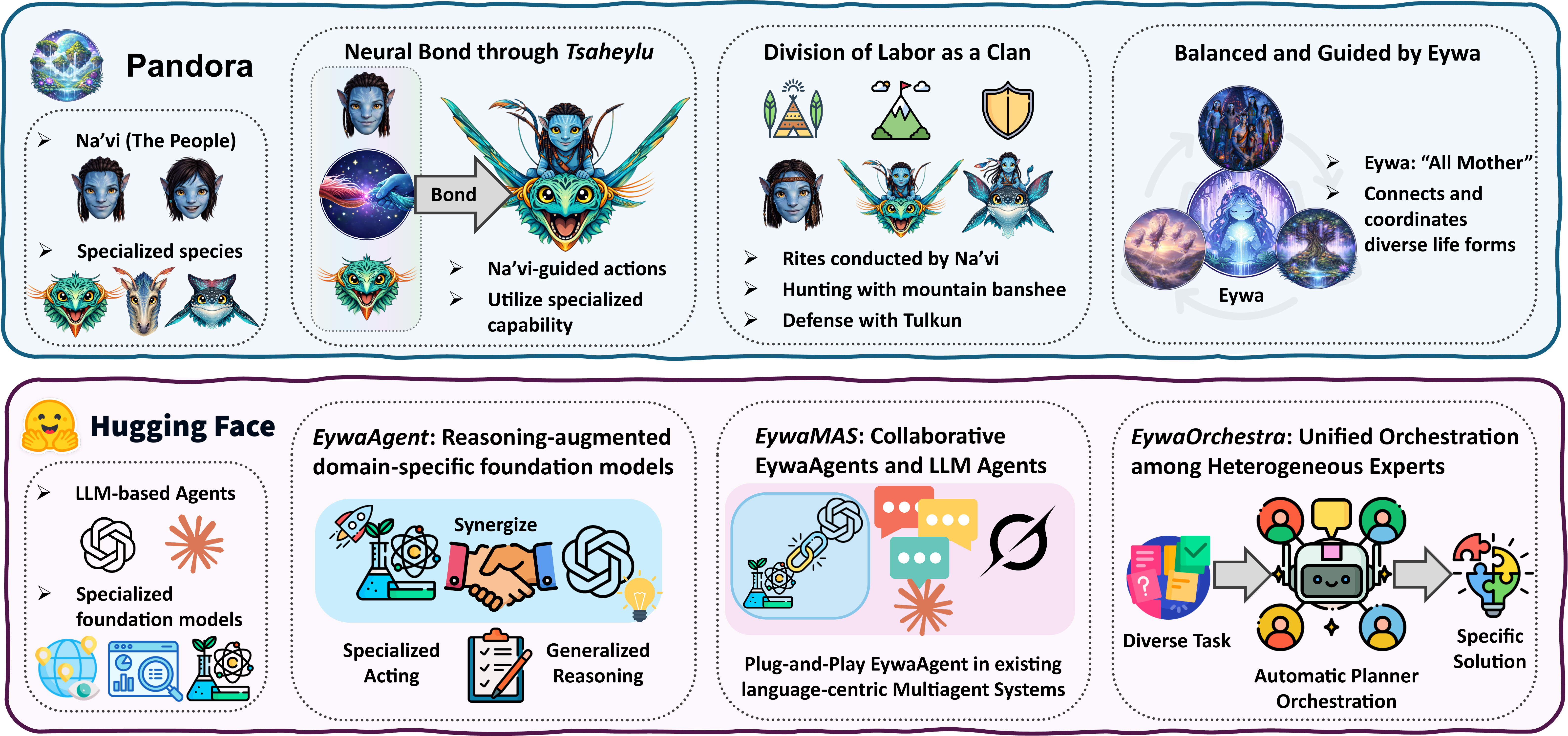}
    \caption{
Analogy between the Avatar Pandora ecosystem and the Agentic AI ecosystem.
In Pandora, specialized species are coordinated by Na'vi through \textit{Tsaheylu} (a neural bond for cross-species communication), collaborating under the global guidance of \textit{"All Mother"}.
Inspired by this analogy, we propose \textit{Eywa}, a three-stage agentic framework: 
(1) \textit{EywaAgent} builds an FM-LLM \textit{Tsaheylu} interface and augments domain-specific foundation models with language-based reasoning interfaces; 
(2) \textit{EywaMAS} enables collaboration between \textit{EywaAgents} and LLM agents; and 
(3) \textit{EywaOrchestra} dynamically orchestrates across heterogeneous experts.
    }
    \label{fig: main}
\end{figure*}

Recent advances in large language models (LLMs) have driven the emergence of agentic AI systems that spark paradigm shifts across numerous industries \cite{DBLP:journals/corr/abs-2303-08774, DBLP:journals/corr/abs-2503-19786, DBLP:journals/corr/abs-2407-21783, DBLP:journals/corr/abs-2601-12538, DBLP:journals/corr/abs-2512-16301, DBLP:journals/corr/abs-2511-20639}. 
Large-scale natural language pretraining endows these systems with general capabilities in perception, planning, reasoning, and decision-making over complex scenarios \cite{DBLP:journals/corr/abs-2307-03109, DBLP:journals/tmlr/FengJLZTCLY24, DBLP:journals/corr/abs-2211-09110}.
However, real-world problems are not limited to natural language. 
In scientific tasks involving specialized data types such as symbolic representations (e.g., formulas, equations) and structured data (e.g., time series) \cite{DBLP:journals/corr/abs-2508-21148, DBLP:journals/jcisd/Weininger88, hersbach2020era5, DBLP:conf/icml/ShojaeeNMFDR25, DBLP:journals/nar/Consortium23}, relying on natural language as the universal interface can become a significant bottleneck. This mismatch between language-centric pretraining and specialized scientific downstream tasks poses a key challenge to the development of agentic systems in scientific domains \cite{DBLP:journals/corr/abs-2508-21148, DBLP:journals/corr/abs-2508-14111, DBLP:journals/corr/abs-2505-19897, rios2026ai}. 

Meanwhile, this limitation has coincided with rapid progress in developing domain-specific foundation models that are optimized for specialized data and tasks \cite{DBLP:journals/nn/MenonMBPJKJ26, DBLP:journals/corr/abs-2312-03014, DBLP:journals/corr/abs-2307-13721, DBLP:journals/corr/abs-2504-04011, DBLP:journals/corr/abs-2108-07258, huang2024foundation}. These models, while not always equipped with language interfaces, are typically pretrained to capture domain-specific patterns and offer strong predictive capabilities within their respective domains \cite{DBLP:journals/corr/abs-2212-12794, DBLP:journals/corr/abs-2405-04285, hamann2024foundation, herzog2025olmoearth}. This creates an opportunity to extend language-centric systems by enabling these heterogeneous foundation models to participate directly in reasoning processes, motivating the following research question:

\vspace{-3mm}
\begin{center}{
\centering\textit{\textbf{
Can heterogeneous foundation models collaborate within agentic systems?
}}
}

\end{center}
\vspace{-3mm}

This challenge can be viewed as a limitation of communication. Unlike existing agentic collaboration systems, which typically assume that all agents communicate through natural language \cite{DBLP:conf/ijcai/GuoCWCPCW024, DBLP:conf/iclr/HongZCZCWZWYLZR24, DBLP:conf/uist/ParkOCMLB23, DBLP:journals/corr/abs-2502-14321, DBLP:journals/corr/abs-2501-06322}, 
many foundation models do not natively support language as an input or output modality. This discrepancy makes it challenging to directly incorporate such foundation models into agentic systems.

\begin{figure}[t]
\vspace{-3mm}
\centering
\begin{tcolorbox}[agentstruct,title=\textbf{Contributions}]
Our contributions are summarized as follows:
\begin{itemize}
    % \item We identify a fundamental limitation of language-centric agentic systems and formalize it as a communication bottleneck between language agents and domain-specific foundation models.
    \item We introduce \textit{Eywa}, a heterogeneous agentic framework that enables modality-native collaboration by augmenting foundation models with language-model-based reasoning interfaces.
    \item We introduce three instantiations of the framework: \textit{EywaAgent} (single-agent integration), \textit{EywaMAS} (multi-agent extension), and \textit{EywaOrchestra} (planning-based orchestration).
    \item We conduct extensive experiments across diverse scientific domains on our new benchmark \textit{EywaBench}, demonstrating improved performance by effectively integrating domain-specific foundation models into agentic systems.
\end{itemize}
\end{tcolorbox}
\vspace{-3mm}
\end{figure}

To better illustrate this limitation and our motivation, in Figure \ref{fig: main}, we draw an analogy from the movie \textit{Avatar}\footnote{\url{https://en.wikipedia.org/wiki/Avatar_(2009_film)}}. On Pandora, many species (e.g., Mountain Banshee\footnote{\url{https://james-camerons-avatar.fandom.com/wiki/Mountain_Banshee}} and the Direhorse\footnote{\url{https://james-camerons-avatar.fandom.com/wiki/Direhorse}}) possess highly specialized capabilities. Domain-specific foundation models are analogous to these specialized Pandora species. 
Yet such capabilities cannot be directly coordinated by Na'vi as they do not communicate through a shared symbolic language.
Instead, the Na'vi establish connections through \textit{Tsaheylu}\footnote{\url{https://james-camerons-avatar.fandom.com/wiki/Tsaheylu}}, forming a direct neural interface that enables interaction even across fundamentally different biological systems.

In this paper, we introduce \textit{Eywa}\footnote{\textit{Eywa}, a.k.a., "All Mother", represents a planetary-scale network that connects and coordinates diverse life forms. \url{https://james-camerons-avatar.fandom.com/wiki/Eywa}.}, a heterogeneous agentic framework designed to bridge language agents and domain-specific foundation models.
Inspired by the concept of \textit{Tsaheylu}, we propose to augment a domain-specific foundation model with a language model to create an \textit{EywaAgent} (like a Na'vi warrior bonded with a Banshee). This design allows language agents to guide inference, planning, and decision-making of the foundation models over their specialized tasks.
Building upon this primitive, we further extend Eywa to multi-agent settings. 
We introduce \textit{EywaMAS}, where \textit{EywaAgents} can replace existing language agents in multi-agent systems. 
Moreover, we propose \textit{EywaOrchestra}, a planning-based orchestration framework in which a central planner dynamically coordinates both language agents and EywaAgents to solve complex tasks. 
Through this design, Eywa enables modality-native collaboration, allowing heterogeneous models to participate in a unified reasoning process without requiring full translation into natural language.

We evaluate \textit{Eywa} across a diverse set of scientific domains spanning physical, life, and social sciences.
As shown in Figure~\ref{fig:hero}, \textit{Eywa} consistently improves the utility-cost trade-off over language-only baselines. 
Compared with the Single-LLM-Agent baseline, \textit{EywaAgent} improves utility by $\sim7\%$ across physical, life, and social science tasks, while reducing token usage by $\sim30\%$.
It also reduces execution time by $\sim10\%$.
Similarly, \textit{EywaMAS} improves utility while reducing token and time usage in multi-agent settings. Moreover, \textit{EywaOrchestra} dynamically orchestrates heterogeneous models and improves over single-agent and multi-agent baselines.
These results suggest that modality-native collaboration with specialized foundation models improves scientific task solving while reducing the token and runtime overhead of language-only reasoning. 

\section{Preliminary}

\textbf{LLM Agent ("LLM").} An LLM agent \cite{DBLP:journals/fcsc/WangMFZYZCTCLZWW24, DBLP:journals/corr/abs-2503-21460} is a policy
$A_{\mathrm{LLM}} : \mathcal{S} \to \Delta(\mathcal{M})$
where $\mathcal{S}$ is an internal state space, $\mathcal{M}$ is a response space of messages, actions, or tool invocations, and $\Delta(\mathcal{M})$ denotes a distribution over outputs.
We model the LLM agent as possessing strong general-purpose reasoning capabilities, while accessing non-linguistic inputs only indirectly through textualized representations.

\textbf{Domain-Specific Foundation Model ("FM").} While large language models are often considered a class of foundation models, in this work we use the term \emph{domain-specific foundation models} to refer to models that are primarily designed for specialized domains and do not necessarily provide a native language interface. 
An FM for domain $k$ is formulated as
$F_k : \mathcal{X}_k \times \mathcal{U}_k \to \mathcal{O}_k$
where $\mathcal{X}_k$ is the input, $\mathcal{U}_k$ is a space of structured user configuration arguments, and $\mathcal{O}_k$ is an output space.
$F_k$ is not assumed to natively operate over open-ended language. Rather, it provides faithful capability for a specialized domain.

\textbf{Multi-Agent Systems (MAS).}
A multi-agent system is defined as a tuple $\mathcal{M} = (\mathcal{A}, \mathcal{G})$ where $\mathcal{A} = \{A_1, A_2, \dots, A_n\}$ is a set of agents, and $\mathcal{G}$ denotes the communication topology.
Each agent $A_i$ is a policy $A_i : \mathcal{S}_i \to \Delta(\mathcal{M}_i)$
where $\mathcal{S}_i$ is the local state space and $\mathcal{M}_i$ is the message space.
Let $s_i^{(t)} \in \mathcal{S}_i$ denote the local state of agent $A_i$ at step $t$, and let $m_i^{(t)} \in \mathcal{M}_i$ denote the message produced by $A_i$ at step $t$; correspondingly, $m_{j,i}^{(t)}$ denotes the message sent from $A_j$ to $A_i$, and $m_{-i}^{(t)}$ collects all messages received by $A_i$ under topology $\mathcal{G}$. At each step, agent $A_i$ updates its state and produces a message:
\begin{equation}
    s_i^{(t)} = \mathrm{Update}_i\big(s_i^{(t-1)}, m_{-i}^{(t)}\big), 
    \quad
    m_i^{(t)} \sim A_i\big(s_i^{(t)}\big).
\end{equation}
The interaction proceeds iteratively until the system produces a final output $\hat{y}$ after a finite number of steps.

\textbf{Problem Formulation.}
Let $\mathcal{T}$ denote a family of tasks. Each task instance is represented by $\tau = (q, x, y^\star, \ell)$
where $q \in \mathcal{Q}$ is a natural-language instruction or high-level objective, $x \in \mathcal{X}$ is the task input, $y^\star \in \mathcal{Y}$ is the desired output, and $\ell$ is a task-specific loss function.
We make a reasonable assumption for scientific tasks that the input space factorizes as 
\begin{equation}
    \mathcal{X} = \mathcal{X}_{\mathrm{lng}} \times \mathcal{X}_{1} \times \cdots \times \mathcal{X}_{m},
\end{equation}

where $\mathcal{X}_{\mathrm{lng}}$ denotes language-observable context and each $\mathcal{X}_k$ denotes a domain-specific input.
An agentic system $G$ produces an output $\hat{y}_G(\tau) = G(q, x)$ for task $\tau$. The objective is to minimize the expected task loss over the task distribution:
\begin{equation}
    \min_G \; \mathbb{E}_{\tau \sim \mathcal{T}} \left[ \ell\big(\hat{y}_G(\tau), y^\star \big) \right].
\end{equation}

\begin{assumption}\label{assumption: domain advantage}
(Domain Advantage of Foundation Models).
Let $\pi_k : \mathcal{X} \to \mathcal{X}_k$ denote the projection onto the domain-specific component, i.e., $x_k = \pi_k(x)$.
For any task instance $\tau = (q, x, y^\star, \ell) \sim \mathcal{T}$ with $x_k = \pi_k(x)$ being informative, the foundation model $F_k$ achieves strictly better performance than any language-only model on the domain-specific component $x_k$, i.e.,
\begin{equation}
\mathbb{E}_{\tau = (q, x, y^\star, \ell) \sim \mathcal{T}} \left[ \ell_k(F_k(x_k), y^\star) \right]
<
\inf_{A_{\mathrm{LLM}}} 
\mathbb{E}_{\tau = (q, x, y^\star, \ell) \sim \mathcal{T}} \left[ \ell_k(A_{\mathrm{LLM}}(\text{serialize}(x_k)), y^\star) \right],
\end{equation}
where $\text{serialize}(\cdot)$ maps domain inputs into language tokens, $\ell_k$ is the sub-task loss for domain $k$.
\end{assumption}

\section{EywaAgent: Reasoning Foundation Model Agents}
\label{sec:eywaagent}

Our first step towards the Eywa agentic framework is to introduce \textit{EywaAgent}, a unified abstraction that augments a foundation model with a language-based reasoning interface. The key idea is similar to "Tsaheylu" in Avatar to create a strong bond between a language model that performs high-level planning and control and a domain-specific foundation model that provides specialized capabilities. 

\subsection{FM-LLM ``Tsaheylu'' Bond}

The objective of the FM--LLM ``Tsaheylu'' is to establish a robust and stable communication channel between a domain-specific foundation model $F_k : \mathcal{X}_k \times \mathcal{U}_k \to \mathcal{O}_k$ for domain $k$ and a language model $A_{\mathrm{LLM}} : \mathcal{S} \to \Delta(\mathcal{M})$. The Tsaheylu interface is designed to ensure that: (1) the LLM can correctly configure the control input $\mathcal{U}_k$ conditioned on the task state, (2) specialized computation is delegated to the foundation model, and (3) the resulting output $\mathcal{O}_k$ can be faithfully reintegrated into the language reasoning process.

To this end, we formalize the FM--LLM Tsaheylu as a bidirectional communication interface between the language model and the specialist. For each domain $k$, we define an interface pair $(\phi_k, \psi_k)$ where

\begin{itemize}
    \item $\phi_k : \mathcal{S} \to \mathcal{U}_k$ is a \emph{query compiler} that translates the task state into a structured FM invocation.
    \item $\psi_k : \mathcal{O}_k \to \mathcal{Z}_k$ is a \emph{response adapter} that converts the specialist output into a planner-consumable representation, where $\mathcal{Z}_k$ denotes a structured context space compatible with language reasoning.
\end{itemize} 

The resulting communication pipeline can be expressed as
\[
\text{input task} \,\, \tau 
\xrightarrow{\text{task interpretation by} \,\, A_{LLM}} 
s
\xrightarrow{\phi_k} 
u_k 
\xrightarrow{F_k(\cdot)} 
o_k 
\xrightarrow{\psi_k} 
z_k 
\xrightarrow{\text{reasoning \& synthesis by} \,\, A_{LLM}} 
\text{output response} \,\, \hat{y}
\]
which enables seamless integration of heterogeneous foundation models into the language-centric reasoning loop and forms the building block of our Eywa agentic framework.

\begin{wrapfigure}{r}{0.50\textwidth}
  \begin{center}
  \vspace{-10pt}
    \includegraphics[width=0.50\textwidth]{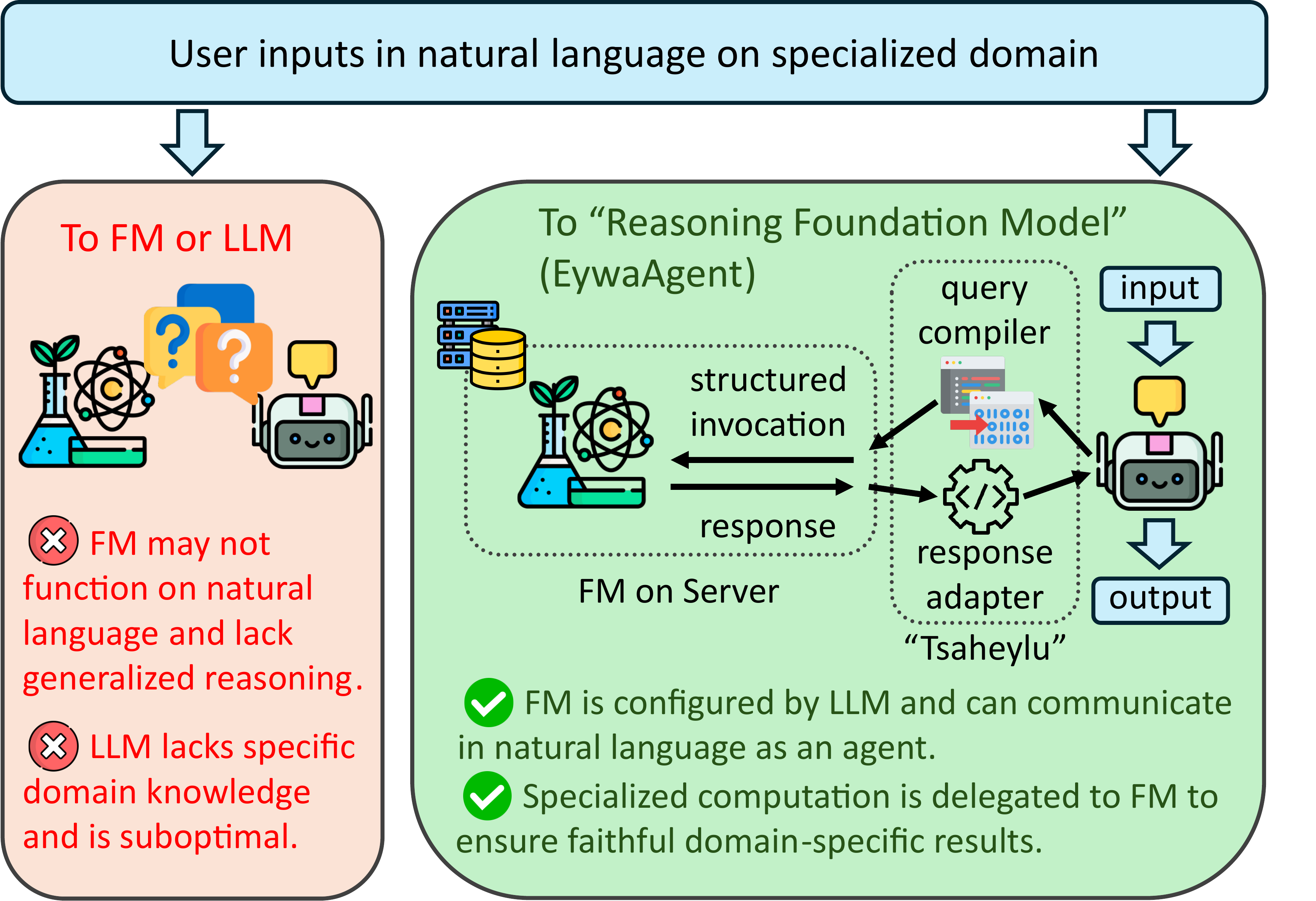}
  \end{center}
  \vspace{-11pt}
  \caption{Reasoning Foundation Model Agent (EywaAgent) leverages both generalized reasoning and specialized acting through FM-LLM "Tsaheylu" Bond.}
  \label{fig:eywaagent}
  \vspace{-11pt}
\end{wrapfigure}

\paragraph{Instantiation via Model Context Protocol.}
We implement the Tsaheylu interface through the \emph{Model Context Protocol (MCP)}~\cite{mcp2024}, which provides a standardized mechanism for structured interaction between language agents and external computational resources.
Within this framework, each foundation model $F_k$ is exposed as a \emph{remote service} with a well-defined schema over its input space $\mathcal{U}_k$ and output space $\mathcal{O}_k$. The query compiler $\phi_k$ is implemented as a structured tool call that specifies the target resource (e.g., dataset identifier or model endpoint), the invocation parameters (e.g., prediction horizon, conditioning variables), and any necessary execution constraints.

Upon invocation, the MCP server executes the requested operation by (i) retrieving domain-specific data $x_k$ from storage, (ii) applying the foundation model $F_k$ with the provided configuration $u_k$, and (iii) returning structured outputs $o_k$ to the agent.
The response adapter $\psi_k$ then transforms $o_k$ into a language-compatible representation $z_k$, which is appended to the subsequent reasoning.

\subsection{EywaAgent for Generalized Reasoning and Specialized Acting}

Building upon the FM--LLM Tsaheylu interface, we define \textit{EywaAgent} as a unified agent that combines language-based reasoning with domain-specific computation. Unlike conventional language agents that operate purely in the linguistic space, EywaAgent dynamically decides whether to invoke a foundation model based on task requirements. 
In our setting, the term ``agent'' does not refer to a single standalone model. Instead, an \textit{EywaAgent} is a coupled FM-LLM agentic unit. 
The two components are connected through the Tsaheylu interface and jointly define the agent's behavior.

\begin{tcolorbox}[
  agentscope,
  float,
  floatplacement=h,
  title=\textbf{EywaAgent}
]
\begin{definition}[EywaAgent]
An \textit{EywaAgent} is defined as a tuple
\begin{equation}
A_{\mathrm{eywa}} = (A_{\mathrm{LLM}}, F, \phi, \psi, \mathcal{C})
\end{equation}
where $A_{\mathrm{LLM}} : \mathcal{S} \to \Delta(\mathcal{M})$ is a language model operating over the agent state space $\mathcal{S}$, $F : \mathcal{X} \times \mathcal{U} \to \mathcal{O}$ is a domain-specific foundation model, and $(\phi, \psi)$ defines a bidirectional communication interface between the language space and the domain space, with $\phi : \mathcal{S} \to \mathcal{U}$ and $\psi : \mathcal{O} \to \mathcal{Z}$.
The control policy $\mathcal{C} : \mathcal{S} \to \{\texttt{invoke}, \texttt{skip}\}$ determines whether and how the foundation model is invoked at each step. If not specified, $\mathcal{C}$ is induced by the language model through its reasoning over the current state.
\end{definition}
\end{tcolorbox}

An illustration of EywaAgent is shown in Figure \ref{fig:eywaagent}. At each reasoning step $t$, given state $s^{(t)}$, the control policy produces a decision $ a^{(t)} \sim \mathcal{C}(s^{(t)}), \quad a^{(t)} \in \{ \texttt{invoke}, \texttt{skip} \}$.

If $a^{(t)} = \texttt{skip}$, the agent reduces to a standard language-only reasoning step:
\begin{equation}
z^{(t)} = A_{\mathrm{LLM}}(s^{(t)}).
\end{equation}
If $a^{(t)} = \texttt{invoke}$, the agent executes the Tsaheylu pipeline:
\begin{equation}
u = \phi(s^{(t)}), \quad
o = F(x, u), \quad
z^{(t)} = \psi(o).
\end{equation}
The updated state is given by $s^{(t+1)} = s^{(t)} \cup \{z^{(t)}\}$. This adaptive mechanism enables EywaAgent to seamlessly switch between generalized reasoning and specialized acting.
From this perspective, EywaAgent subsumes language-only agents as a special case (by choosing $\mathcal{C}$ to always skip), while strictly expanding the space of computable functions through access to domain-specific foundation models. Consequently, EywaAgent achieves both enhanced expressivity and improved task performance.
In particular, we show that EywaAgent attains a strictly lower optimal expected task loss compared to language-only agents under the domain advantage assumption, thereby expanding the class of tasks that can be effectively solved.
% Moreover, EywaAgent improves efficiency when handling domain-specific inputs. 
Moreover, by delegating computation to the foundation model, \textit{Eywa} avoids explicit token-level reasoning over structured data and reduces token usage. Formally, we have the following theorem.

\begin{theorem}\label{thm:eywaagent improvement}
(Improvement of EywaAgent over Language-only Agent)
Let $\mathcal{F}_{\mathrm{LLM}}$ and $\mathcal{F}_{\mathrm{Eywa}}$ denote the function classes induced by language-only agents and EywaAgent, respectively. Under Assumption \ref{assumption: domain advantage}, 
we have the following strict risk improvement of EywaAgent:
\begin{equation}
\inf_{f \in \mathcal{F}_{\mathrm{Eywa}}} 
\mathbb{E}_{\tau \sim \mathcal{T}}[\ell(f(x), y^\star)]
\;<\;
\inf_{f \in \mathcal{F}_{\mathrm{LLM}}} 
\mathbb{E}_{\tau \sim \mathcal{T}}[\ell(f(x), y^\star)].
\end{equation}
\end{theorem}
(Proof in Appendix \ref{app:expressivity}) 

Theorem~\ref{thm:eywaagent improvement} shows that the Tsaheylu interface preserves language-only reasoning through the skip branch, while enabling native foundation-model computation through the invoke branch, yielding lower optimal risk under the domain advantage assumption.

\section{Eywa Agentic Systems: Multi-Agent Composition and Orchestration}

With EywaAgent defined as a plug-and-play building block for agentic AI systems, we naturally extend this paradigm to multi-agent settings to enable more complex and heterogeneous collaborations. 
To this end, we introduce two complementary system-level abstractions. First, \textit{EywaMAS} generalizes EywaAgent to a distributed multi-agent setting, allowing multiple specialized agents to interact and collaborate. Second, \textit{EywaOrchestra} introduces a global orchestration mechanism that dynamically coordinates agents through structured planning and execution to solve complex tasks.

\subsection{EywaMAS: Plug-and-Play Composition with EywaAgent}

\begin{tcolorbox}[
  agentscope,
  float,
  floatplacement=h,
  title=\textbf{EywaMAS}
]
\begin{definition}[EywaMAS]
An \textit{EywaMAS} is defined as a multi-agent system
\begin{equation}
\mathcal{M}_{\mathrm{Eywa}} = (\mathcal{A}, \mathcal{G})
\end{equation}
where $\mathcal{A} = \{\mathcal{A}_1, \dots, \mathcal{A}_n\}$ is a set of heterogeneous agents, each of which is either an LLM agent or an EywaAgent, and $\mathcal{G}$ specifies the communication topology.
\end{definition}
\end{tcolorbox}

\begin{wrapfigure}{r}{0.30\textwidth}
  \begin{center}
  \vspace{-19pt}
    \includegraphics[width=0.30\textwidth]{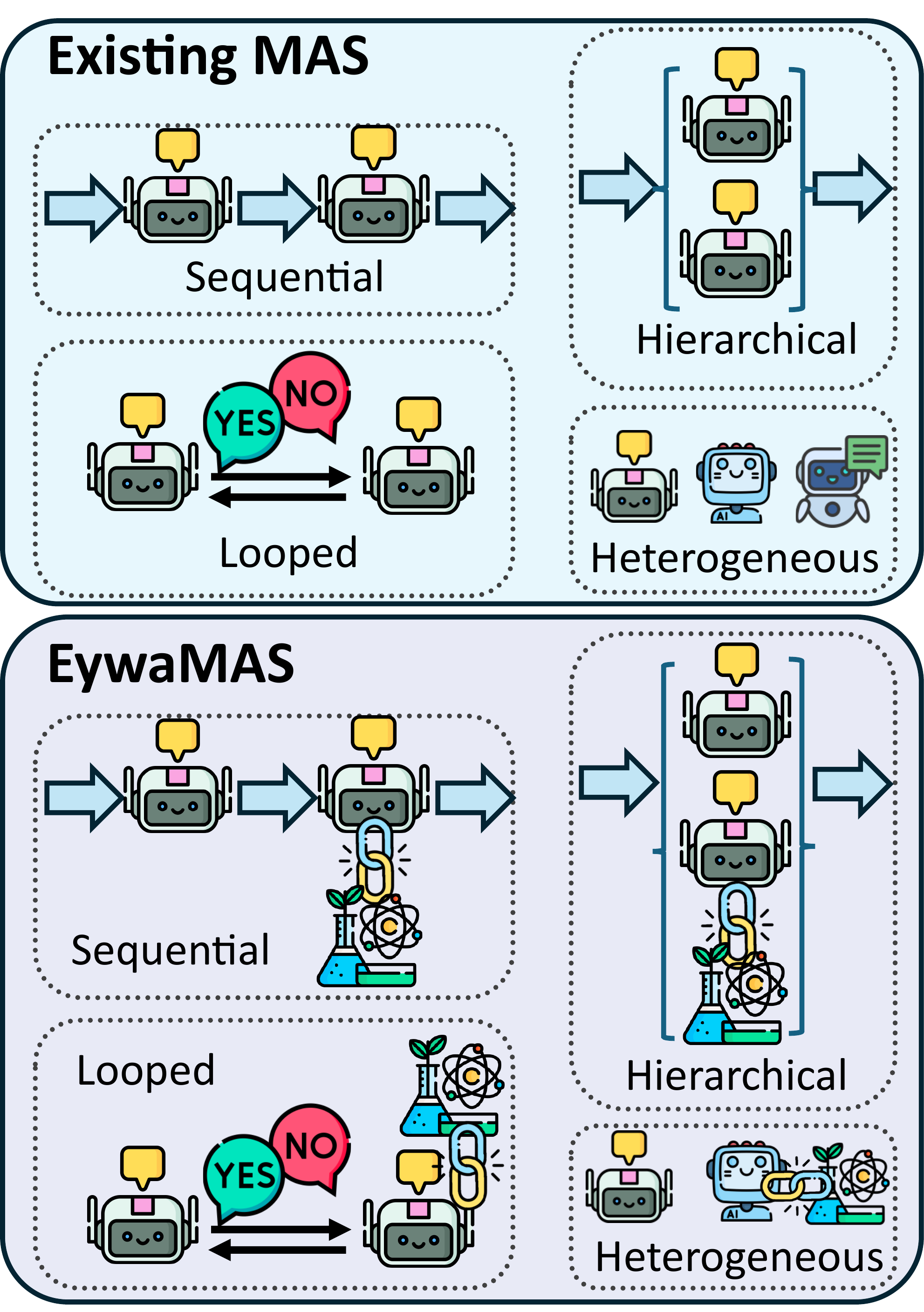}
  \end{center}
  \vspace{-15pt}
  \caption{EywaMAS generalizes existing multi-agent systems with EywaAgents.}
  \vspace{-10pt}
\end{wrapfigure}

\textit{EywaMAS} is a multi-agent system that composes heterogeneous agents in a plug-and-play manner. Unlike traditional multi-agent systems that rely solely on language-based agents, EywaMAS enables seamless integration of both language-only agents and EywaAgents within a unified framework. In practice, constructing an EywaMAS requires minimal modification to existing multi-agent architectures. Specifically, one can replace a subset of language-only agents with EywaAgents while keeping the overall system structure unchanged. For example, in a hierarchical multi-agent system consisting of a planner, multiple worker agents, and a summarizer, substituting some of the worker agents with EywaAgents directly yields an EywaMAS. This plug-and-play property allows domain-specific foundation models to be incorporated into existing agentic systems without redesigning communication protocols or system architectures.

Each EywaAgent $\mathcal{A}_k$ internally follows the mechanism defined in Section~\ref{sec:eywaagent}, operating in the language space while optionally invoking its associated foundation model through the Tsaheylu interface.
At the system level, EywaMAS follows the same communication and state update dynamics as standard multi-agent systems, with interactions governed by the communication topology $\mathcal{G}$. 
As a result, EywaMAS forms a heterogeneous agentic system that integrates language-based reasoning with specialized foundation model capabilities in a unified and modular framework. It supports flexible composition across (1) different language models with varying scales or capabilities, (2) different foundation models across domains, and (3) mixed agent types (LLM Agents and EywaAgents).
When $\mathcal{A} =\{\mathcal{A}_1, \dots, \mathcal{A}_n\}$ are all LLM agents, EywaMAS degenerates to existing multi-agent systems. Similar to Theorem~\ref{thm:eywaagent improvement}, we establish a theoretical result showing that EywaMAS strictly improves over language-only multi-agent systems in Appendix \ref{app:mas}.

Compared to single-agent settings, EywaMAS enables parallel specialization and cross-domain collaboration. However, as communication topology is fixed and coordination is decentralized, the system may suffer from suboptimal topology or human configuration inefficiency, motivating the need for automatic orchestration.

\subsection{EywaOrchestra: Dynamic Orchestration of Heterogeneous Experts}

Real-world agentic tasks are highly diverse, and the optimal multi-agent organization for one task may be suboptimal for another. In particular, different tasks may require different mixtures of language-only reasoning, domain-grounded prediction, and inter-agent collaboration patterns. To address this, we extend \textit{Eywa} from static heterogeneous agents and fixed multi-agent systems to a dynamic orchestration framework, termed \textit{EywaOrchestra}.

\begin{tcolorbox}[
  agentscope,
  float,
  floatplacement=h,
  title=\textbf{EywaOrchestra}
]
\begin{definition}[EywaOrchestra]
\label{def:eywaorchestra}
Given candidate language models $\mathcal{M}_{\mathrm{LLM}}$, candidate domain-specific foundation models $\mathcal{M}_{\mathrm{FM}}$, and a topology pool $\Pi$, \textit{EywaOrchestra} is a dynamically instantiated heterogeneous multi-agent system:
\begin{equation}
\mathcal{O} = (\mathcal{C}, P)
\end{equation}
where $\mathcal{C}$ is the configuration space induced by $(\mathcal{M}_{\mathrm{LLM}}, \mathcal{M}_{\mathrm{FM}}, \Pi)$, $P$ is the conductor.
\end{definition}
\end{tcolorbox}

Planning in \textit{EywaOrchestra} is achieved by a \emph{conductor} that, conditioned on the input task, dynamically instantiates a heterogeneous multi-agent system by deciding: 
(i) the role and type of each agent, e.g., whether an agent should be a language-only agent or an EywaAgent;
(ii) the backbone language model used by each agent;
(iii) the domain-specific foundation model attached to each EywaAgent; and
(iv) the communication topology of the overall multi-agent system.
For tractability in this initial study, we assume a finite \emph{topology pool} of candidate multi-agent structures. In the current implementation, the \emph{conductor} is instantiated as a large language model that maps the input task to a system configuration from this pool. After the conductor selects a configuration, the instantiated system is executed to solve the task.

\begin{algorithm}[t]
\caption{EywaOrchestra: Dynamic Orchestration of Heterogeneous Experts}
\label{alg:eywaorchestra}
\begin{algorithmic}[1]
\REQUIRE Task input $\tau = (q, x, y^\star, \ell)$, configuration space $\mathcal{C}$, conductor $P$
\ENSURE Output $\hat{y}$

\STATE Select a system configuration $c \leftarrow P(q,x)$

\STATE Instantiate the heterogeneous agent system specified by $c$

\STATE Execute the induced multi-agent system on $(q,x)$ and return output $\hat{y}$
\end{algorithmic}
\end{algorithm}

The benefit of \textit{EywaOrchestra} can be understood through the gap between adaptive orchestration and any fixed configuration. Let $F_c$ denote the agent system instantiated by configuration $c$. Define the best fixed configuration risk as
\begin{equation}
\mathcal{R}_{\mathrm{fixed}}^\star
=
\min_{c\in\mathcal{C}}
\mathbb{E}_{\tau\sim\mathcal{T}}
\big[\ell(F_c(q,x),y^\star)\big],
\end{equation}
and the oracle adaptive risk as
\begin{equation}
\mathcal{R}_{\mathrm{oracle}}
=
\mathbb{E}_{\tau\sim\mathcal{T}}
\Big[
\min_{c\in\mathcal{C}}
\mathbb{E}\big[\ell(F_c(q,x),y^\star)\big]
\Big].
\end{equation}
By construction, $\mathcal{R}_{\mathrm{oracle}} \le \mathcal{R}_{\mathrm{fixed}}^\star$, with strict inequality whenever different regions of the task distribution favor different system configurations. This observation highlights the limitation of any fixed multi-agent design across various tasks. The motivation of \textit{EywaOrchestra} is precisely to move beyond such static designs by enabling task-adaptive system construction, jointly leveraging \emph{model adaptivity}, through selecting language and domain-specific foundation models, and \emph{structural adaptivity}, through selecting the communication topology itself.

\section{Experiments}
\label{sec:experiment}

\subsection{Eywabench: A Scalable Multi-task Multi-domain Scientific Benchmark}

Current scientific benchmarks are often limited to a narrow task family \cite{DBLP:journals/corr/WuRFGGPLP17, DBLP:journals/corr/abs-2502-14739}, a single domain \cite{DBLP:conf/emnlp/DaiYSZGHLZTGH25, DBLP:conf/emnlp/DaiYSZGHLZTGH25, DBLP:journals/corr/abs-2504-16074}, or one data format \cite{DBLP:journals/corr/abs-2505-19501, DBLP:journals/corr/abs-2507-03578, DBLP:conf/nips/JohnsonFGMBSF23, DBLP:conf/cvpr/0004WGWWCML0Z0025}, and therefore may not fully reflect the capability requirements of scientific agentic systems.
More specifically, two important scientific modalities, time series and tabular data, are often ignored or poorly evaluated in existing benchmarks.
To provide a holistic evaluation of agentic systems in scientific settings, we introduce \textit{Eywabench}, a scalable benchmark for multi-task and multi-domain scientific reasoning across heterogeneous modalities.
\textit{Eywabench} is constructed from a collection of datasets, including but not limited to DeepPrinciple \cite{song2025evaluating}, MMLU-Pro \cite{wang2024mmlu}, fev-bench \cite{DBLP:journals/corr/abs-2509-26468}, and TabArena \cite{DBLP:journals/corr/abs-2506-16791}.

\textbf{Multi-task and multi-domain coverage.} \textit{Eywabench} includes tasks spanning natural language, time series, and tabular data. The samples are organized into three domains: physical science, life science, and social science.
Each domain further contains three sub-domains: physical science includes material, energy, and space; life science includes biology, clinic, and drug; social science includes economy, business, and infrastructure.
Details of domains and sub-domains are provided in Appendix \ref{sec:dataset_details}.

\textbf{Scalability.} \textit{Eywabench} scales along both task volume and domain coverage.
Task volume can be increased by sampling new temporal windows, variables, and contextual combinations from source datasets.
Domain coverage can be expanded by applying the same construction pipeline to new time-series and tabular resources beyond fev-bench and TabArena. Moreover, the same design principle naturally extends to other scientific modalities (e.g., vision and geospatial earth observation) by incorporating corresponding datasets and specialist foundation models.

\begin{table}[!t]
    \centering
    \renewcommand{\arraystretch}{1.25}
    \caption{Overall performance comparison across scientific domains on \textit{EywaBench}. We compare all methods on three dimensions, including utility $(\uparrow)$, inference time $(\downarrow)$, and token consumption $(\downarrow)$. Best results are highlighted in bold and second-best results are underlined. Our proposed methods, EywaAgent, EywaMAS, and EywaOrchestra, achieve strong overall performance while maintaining competitive efficiency.}
    % \vspace{-8pt}
    \label{tab:main_comparison_eywabench}
    \small
    \resizebox{\textwidth}{!}{
    \begin{tabular}{l|c|ccc|ccc|ccc|c}
        \toprule
        \multirow{2}{*}{\textbf{Method}} &
        \multirow{2}{*}{\textbf{Metrics}} &
        \multicolumn{3}{c}{\textbf{Physical Science}} &
        \multicolumn{3}{c}{\textbf{Life Science}} &
        \multicolumn{3}{c}{\textbf{Social Science}}
        \\
        \cmidrule(lr){3-5}\cmidrule(lr){6-8}\cmidrule(lr){9-11}\cmidrule(lr){12-12}
        & & Material & Energy & Space & Biology & Clinic & Drug & Economy & Business & Infrastructure & Overall\\
        \midrule[-0.4ex]\midrule\addlinespace[-0.000ex]
        \rowcolor{gray!16}
        \multicolumn{12}{c}{
            \rule{0pt}{1.1em}
            \textbf{\textit{Single-Agent Setting}}
            \rule[-0.3em]{0pt}{1.1em}
        } \\
        [-0.4ex]\midrule\addlinespace[-0.000ex]

        \cellcolor{red!7}
        & Utility ($\uparrow$) & 0.5616 & 0.8202 & 0.5235 & 0.3402 & 0.4582 & 0.6004 & 0.7689 & 0.6528 & 0.6758 & 0.6154\\
        \cellcolor{red!7}\textbf{Single-LLM-Agent}
        & Time ($\downarrow$)  & 34.48 & 27.01 & 26.00 & 34.68 & 22.37 & 21.13 & 22.67 & 22.28 & 18.42 & 25.22\\
        \cellcolor{red!7}
        & Tokens ($\downarrow$) & 6367 & 4854 & 4512 & 6164 & 3618 & 3571 & 4097 & 3915 & 3327 & 4469\\
        [-0.4ex]\midrule\addlinespace[-0.000ex]
        
        \cellcolor{orange!7}
        & Utility ($\uparrow$)   & 0.5871 & 0.8390 & 0.6123 & 0.3718 & 0.5085 & 0.6199 & \textbf{0.8048} & \underline{0.7371} & 0.7060 & 0.6558\\
        \cellcolor{orange!7}\textbf{EywaAgent(Ours)}
        & Time ($\downarrow$)  & 34.88 & 24.42 & 23.12 & 30.84 & 20.32 & 15.84 & 19.71 & 20.98 & 15.99 & 22.78\\
                 
        \cellcolor{orange!7}
        & Tokens ($\downarrow$) & 5040 & 3167 & 3329 & 4858 & 2333 & 2210 & 2791 & 2444 & 2248 & 3137\\

        \addlinespace[-0.4ex] \midrule \addlinespace[-0.000ex]
        \rowcolor{gray!16}
        \multicolumn{12}{c}{
            \rule{0pt}{1.1em}
            \textbf{\textit{Multi-Agent Setting}}
            \rule[-0.3em]{0pt}{1.1em}
        } \\
        [-0.4ex]\midrule\addlinespace[-0.000ex]

        \cellcolor{yellow!7}
        & Utility ($\uparrow$)   & 0.5687 & 0.8667 & 0.6244 & 0.3623 & 0.4504 & 0.6215 & 0.7523 & 0.6880 & 0.6362 & 0.6294\\
        \cellcolor{yellow!7}\textbf{Refine MAS \citeyearpar{DBLP:conf/nips/MadaanTGHGW0DPY23}}
        & Time ($\downarrow$)  & 72.76 & 64.22 & 79.65 & 75.21 & 51.89 & 50.63 & 62.33 & 48.54 & 47.49 & 60.59\\
        \cellcolor{yellow!7}
        & Tokens ($\downarrow$) & 11013 & 9009 & 10043 & 10497 & 7029 & 7498 & 8924 & 6997 & 7438 & 8673\\
        [-0.4ex]\midrule\addlinespace[-0.000ex]

        \cellcolor{green!7}
        & Utility ($\uparrow$)   & 0.5602 & 0.8656 & 0.6543 & 0.3438 & 0.4738 & 0.6198 & 0.7729 & 0.6907 & 0.7237 & 0.6460\\
        \cellcolor{green!7}\textbf{Debate MAS \citeyearpar{DBLP:conf/icml/Du00TM24}}
        & Time ($\downarrow$)  & 82.06 & 79.46 & 74.75 & 101.64 & 78.19 & 63.98 & 92.72 & 72.46 & 60.73 & 78.22\\
        \cellcolor{green!7}
        & Tokens ($\downarrow$) & 16652 & 14278 & 13614 & 17007 & 11159 & 10447 & 14694 & 10953 & 10311 & 13216\\
        [-0.4ex]\midrule\addlinespace[-0.000ex]

        \cellcolor{teal!7}
        & Utility ($\uparrow$)   & 0.5909 & 0.8069 & 0.5863 & 0.3580 & 0.4722 & 0.5686 & 0.7499 & 0.7004 & 0.6938 & 0.6273\\
        \cellcolor{teal!7}\textbf{MoA \citeyearpar{DBLP:conf/iclr/WangWAZZ25}}
        & Time ($\downarrow$)  & 90.15 & 56.95 & 69.32 & 59.10 & 46.53 & 44.31 & 57.35 & 48.29 & 47.34 & 57.75\\
        \cellcolor{teal!7}
        & Tokens ($\downarrow$) & 25327 & 16453 & 17332 & 15980 & 11014 & 10344 & 16114 & 11690 & 12365 & 15317\\
        [-0.4ex]\midrule\addlinespace[-0.000ex]

        \cellcolor{blue!7}
        & Utility ($\uparrow$)   & 0.5831 & 0.8057 & 0.5723 & \underline{0.3737} & 0.4490 & 0.6211 & 0.6923 & 0.6390 & 0.7180 & 0.6188\\
        \cellcolor{blue!7}\textbf{X-MAS \citeyearpar{DBLP:journals/corr/abs-2505-16997}}
        & Time ($\downarrow$)  & 104.48 & 86.63 & 79.06 & 88.20 & 67.94 & 59.76 & 75.50 & 72.82 & 62.95 & 77.42\\
        \cellcolor{blue!7}
        & Tokens ($\downarrow$) & 24149 & 19808 & 16584 & 18451 & 12549 & 11907 & 16499 & 14007 & 14056 & 16537\\
        [-0.4ex]\bottomrule

        \cellcolor{violet!7}
        & Utility ($\uparrow$)   & \textbf{0.6381} & \textbf{0.8742} & \underline{0.6899} & \textbf{0.3798} & \underline{0.5086} & \underline{0.6248} & \underline{0.7959} & 0.7284 & \textbf{0.7406} & \textbf{0.6761}\\
        \cellcolor{violet!7}\textbf{EywaMAS (Ours)}
        & Time ($\downarrow$)  & 77.25 & 75.96 & 72.51 & 111.92 & 59.97 & 59.23 & 68.40 & 58.11 & 46.49 & 72.11\\
        \cellcolor{violet!7}
        & Tokens ($\downarrow$) & 14529 & 11709 & 11787 & 16502 & 9407 & 8078 & 11044 & 9470 & 8912 & 11214\\
        
        \addlinespace[-0.4ex] \midrule \addlinespace[-0.000ex]
        \rowcolor{gray!16}
        \multicolumn{12}{c}{
            \rule{0pt}{1.1em}
            \textbf{\textit{Dynamic Orchestration}}
            \rule[-0.3em]{0pt}{1.1em}
        } \\
        [-0.4ex]\midrule\addlinespace[-0.000ex]

        \cellcolor{pink!7}
            & Utility ($\uparrow$)   & \underline{0.6249} & \underline{0.8711} & \textbf{0.7187} & 0.3682 & \textbf{0.5159} & \textbf{0.6319} & 0.7830 & \textbf{0.7388} & \underline{0.7298} & \underline{0.6746}\\
        \cellcolor{pink!7}\textbf{EywaOrchestra (Ours)}
            & Time ($\downarrow$)  & 61.78 & 39.92 & 75.47 & 67.88 & 45.38 & 45.94 & 49.13 & 34.18 & 28.80 & 48.16\\
        \cellcolor{pink!7}
            & Tokens ($\downarrow$) & 11535 & 7723 & 10810 & 11315 & 7050 & 6495 & 7117 & 7264 & 6892 & 8335\\
        [-0.4ex]\bottomrule\addlinespace[-0.000ex]

    \end{tabular}
    }
    % \vspace{-4pt}
\end{table}

\textbf{Evaluation metric.} \textit{Eywabench} uses a unified utility score $u \in [0,1]$ across all tasks, so results from different modalities are directly comparable.
For natural-language-centered tasks, utility is computed with a soft-match score between predictions and references; for time-series and tabular tasks, utility is derived from normalized prediction errors.
The benchmark reports per-domain mean utility and also mean utility over all tasks. Detailed metric definitions are provided in the Appendix \ref{sec:dataset_details}.

\subsection{Experimental Setup}
\textbf{Language models and foundation models.} We use gpt-5-nano as the default language model and we also evaluate other models from and beyond the GPT family in our later experiments.
We have two foundation models to build EywaAgents. Chronos \cite{DBLP:journals/tmlr/AnsariSTZMSSRPK24, DBLP:journals/corr/abs-2510-15821} is a general-purpose foundation model for time series. 
TabPFN \cite{DBLP:conf/iclr/Hollmann0EH23} is a transformer-based foundation model that uses in-context-learning to solve tabular prediction problems in a forward pass.
Neither foundation models provide a native language interface.

\textbf{Baseline methods.} We compare \textit{Eywa} against three baseline groups: (a) \emph{single-agent LLM baselines} built on the same backbone, including GPT \cite{DBLP:journals/corr/abs-2601-03267}, Gemini \cite{team2023gemini, comanici2025gemini}, and Claude \cite{claude} families of models; (b) \emph{homogeneous LLM-based multi-agent baselines}, including Refine \cite{DBLP:conf/nips/MadaanTGHGW0DPY23} and Debate \cite{DBLP:conf/icml/Du00TM24}, instantiated with the same backbone model; (c) \emph{heterogeneous LLM-based multi-agent baselines}, including Mixture-of-Agents (MoA) \cite{DBLP:conf/iclr/WangWAZZ25} and X-MAS \cite{DBLP:journals/corr/abs-2505-16997}, which combine multiple heterogeneous language models.

\textbf{Implementation details.} We implement \textit{Tsaheylu} (FM-LLM interface) using LangChain agents and FastMCP servers.
Each foundation model is deployed as an independent MCP backend, served over streamable HTTP on a local port.
Each \textit{EywaAgent} connects a language model to its designated MCP endpoint, loads task data into server-side storage, and invokes the foundation model on demand.
All methods are allowed up to two retries when their outputs cannot be parsed. In practice, Eywa rarely triggers this fallback because specialized foundation models already produce structured outputs in one shot, whereas LLM-only baselines benefit more from retries.
For all runs, we record wall-clock latency and token usage. The results are averaged over multiple runs on a 13th Gen Intel(R) Core(TM) i9-13900H CPU with 64GB RAM. 

\subsection{Main Results}

\begin{wrapfigure}{r}{0.45\textwidth}
  \begin{center}
  \vspace{-10pt}
    \includegraphics[width=0.45\textwidth]{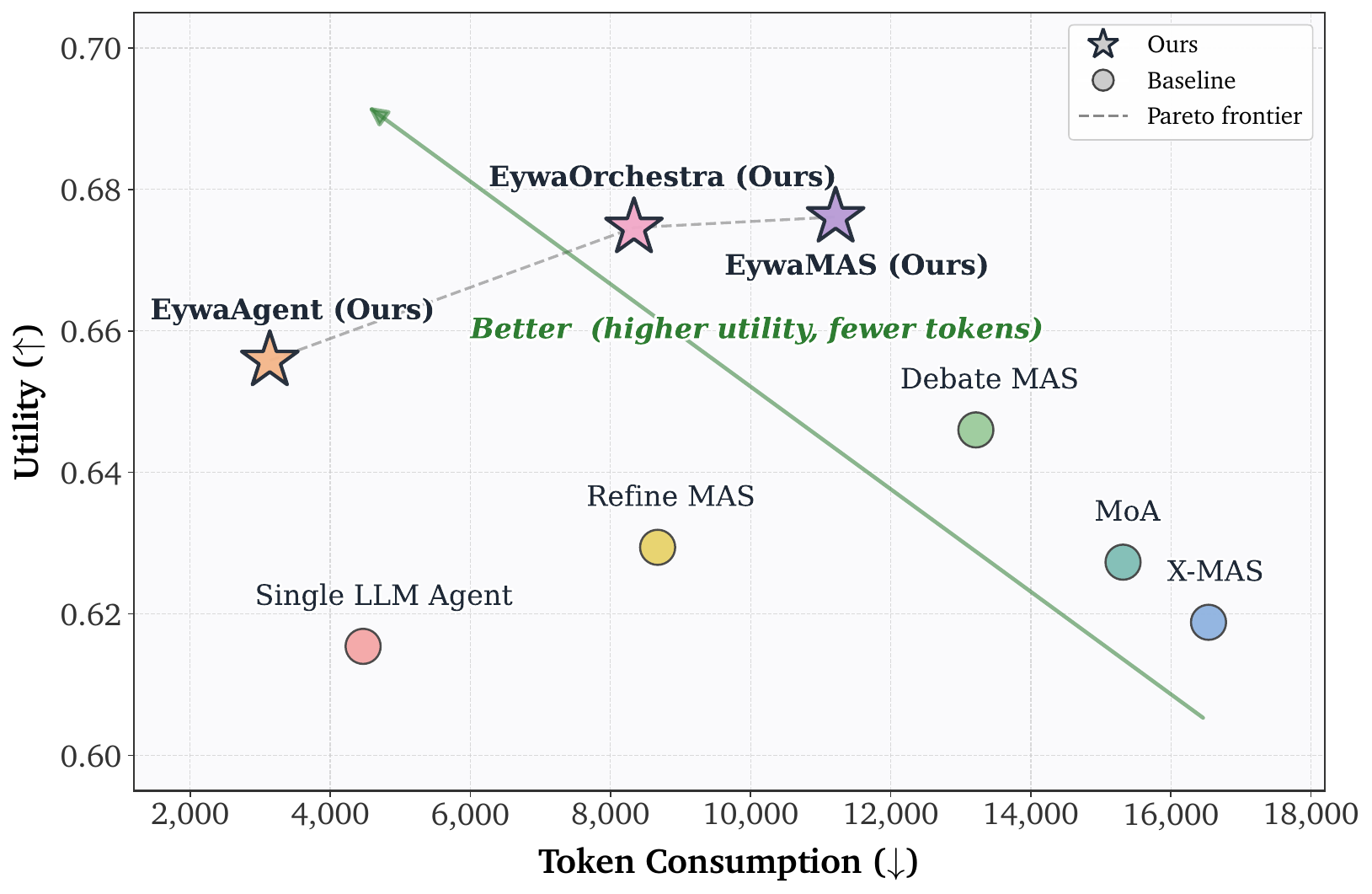}
  \end{center}
  \vspace{-16pt}
  \caption{Overall utility and token consumption of different methods. Full results in Figure \ref{fig:tradeoff_per_domain}. 
  }
  \vspace{-14pt}
\end{wrapfigure}

We evaluate all methods on \textit{Eywabench} under a unified protocol.
For \textit{EywaAgent} and \textit{EywaMAS}, which do not perform dynamic orchestration, we assign the foundation model for each sample using expert-defined configurations. \textit{EywaMAS} uses debate topology by default.
For \textit{EywaOrchestra}, the conductor automatically selects the foundation model and topology conditioned on the sample.

Table \ref{tab:main_comparison_eywabench} reports the performance of all methods on \textit{Eywabench} in scientific settings. We highlight the following observations: (a) \textit{EywaAgent} improves both quality and efficiency under the same backbone.
Compared with the corresponding single-agent baseline, \textit{EywaAgent} increases average utility by $6.6\%$, while reducing latency and cutting token usage by nearly $30\%$ through delegation to domain-specific foundation models.
(b) \textit{EywaMAS} outperforms homogeneous MAS baselines in scientific settings.
\textit{EywaMAS} achieves the best overall utility and outperforms homogeneous multi-agent baselines. Compared with \textit{Refine}, \textit{EywaMAS} delivers significantly stronger utility. Compared with \textit{Debate}, \textit{EywaMAS} not only achieves better utility but also requires fewer tokens under the same debate topology.
(c) {LLM-only heterogeneity is insufficient for scientific tasks.} Heterogeneous LLM-only MAS methods do not consistently outperform strong homogeneous MAS baselines on \textit{Eywabench}. This suggests that, for scientific workloads, cross-modality heterogeneity is more critical than only combining heterogeneous language models.
(d) {Not every domain benefits equally from heavier multi-agent computation.} In domains such as economy and business, single-agent \textit{EywaAgent} is already highly competitive, indicating that always using complex multi-agent topologies is not necessarily optimal. This observation motivates adaptive orchestration conditioned on task and domain characteristics.
(e) \textit{EywaOrchestra} approaches EywaMAS with lower cost and automation. \textit{EywaOrchestra} uses no expert configuration and instead lets the conductor automatically construct the system for each sample. Despite this, \textit{EywaOrchestra} reaches utility close to expert-designed \textit{EywaMAS}, and even surpasses it on several sub-domains. At the same time, dynamic orchestration substantially reduces inference cost in both latency and token usage compared to fixed multi-agent systems.

\subsection{Further Analysis}

\begin{figure*}[t]
  \begin{subfigure}[t]{0.32\linewidth}
    \centering
    \includegraphics[width=\linewidth]{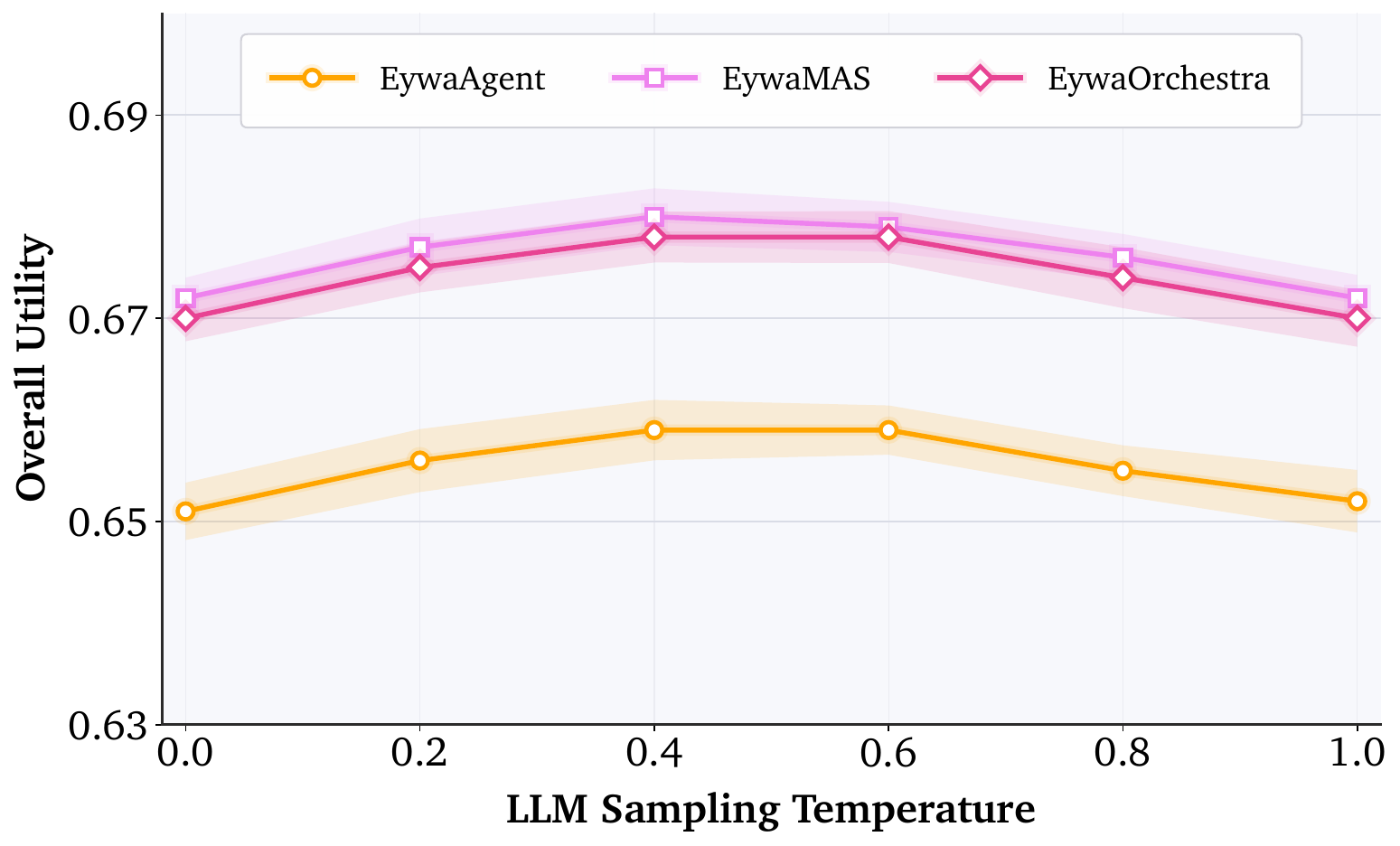}
    \caption{LLM Temperature Ablation.}
  \end{subfigure}\hfill
  \begin{subfigure}[t]{0.32\linewidth}
    \centering
    \includegraphics[width=\linewidth]{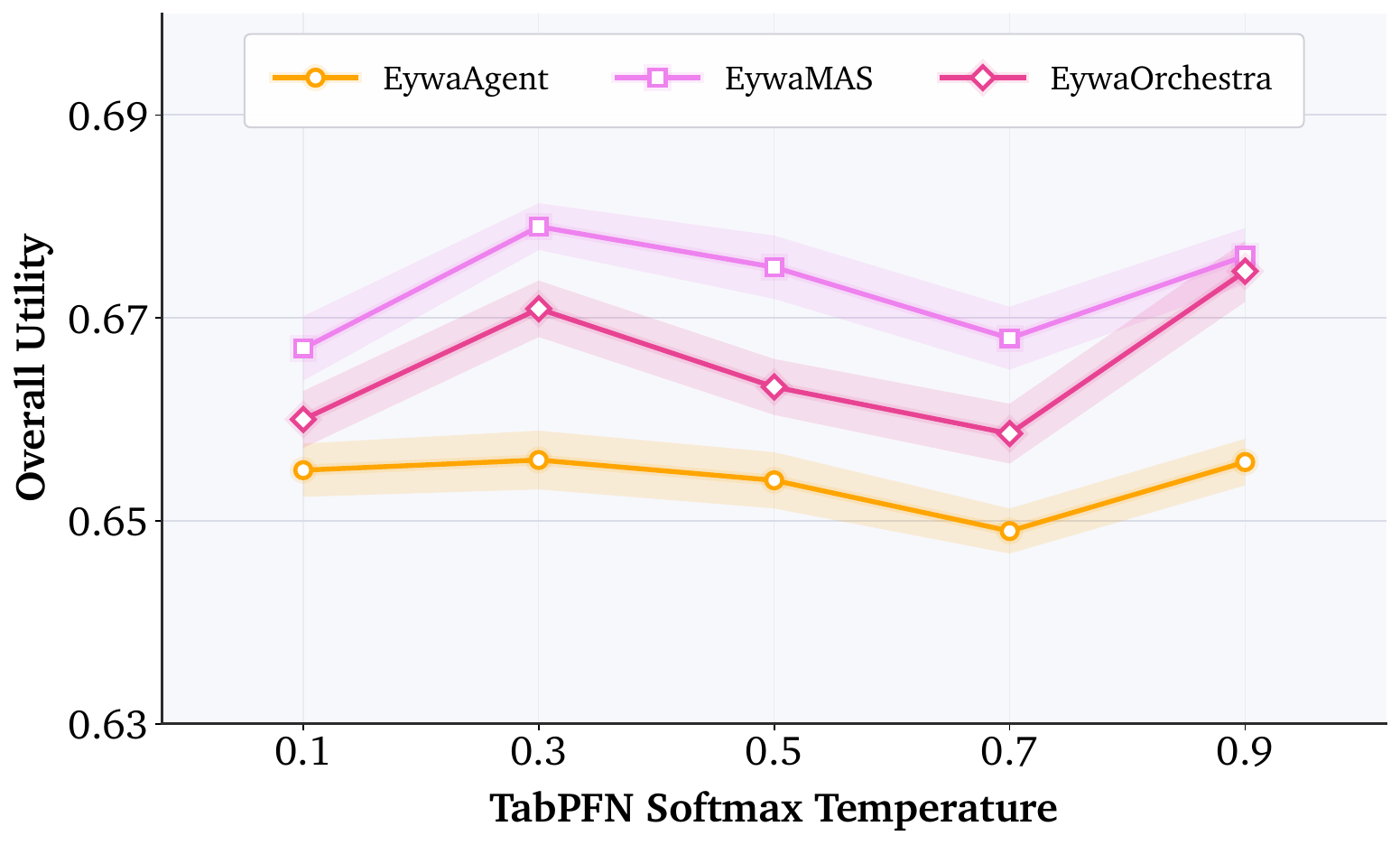}
    \caption{FM Temperature Ablation.}
  \end{subfigure}\hfill
  \begin{subfigure}[t]{0.32\linewidth}
    \centering
    \includegraphics[width=\linewidth]{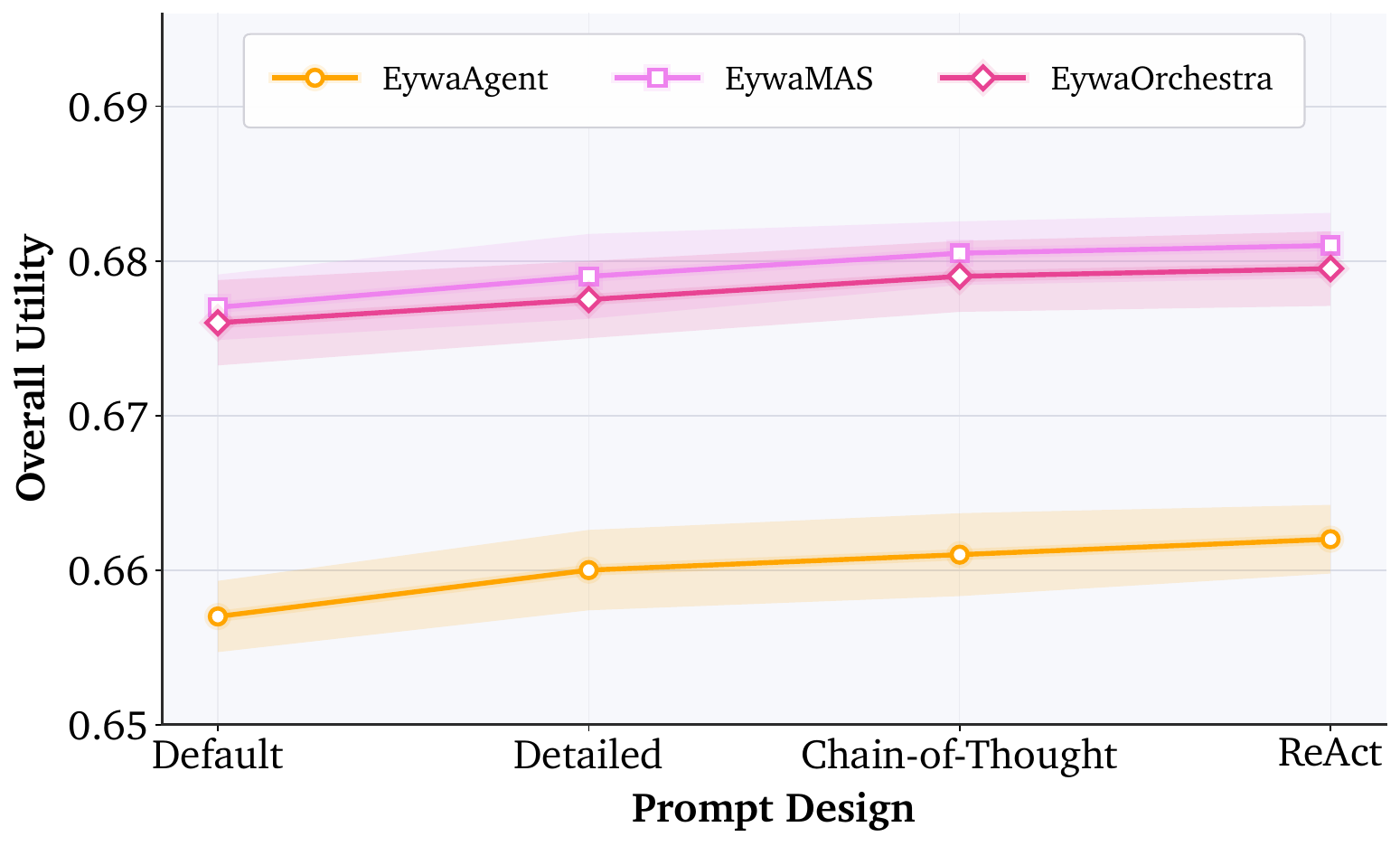}
    \caption{Prompt Design Ablation.}
  \end{subfigure}\\[4pt]

\caption{Hyperparameter sensitivity analysis of \textit{Eywa}. We evaluate \textit{EywaAgent}, \textit{EywaMAS}, and \textit{EywaOrchestra} under different ablation configurations. \textbf{(a)} LLM temperature ablation: \textit{Eywa} maintains stable performance across varying LLM sampling temperatures. \textbf{(b)} FM temperature ablation: \textit{Eywa} remains robust under different TabPFN softmax temperatures. \textbf{(c)} Prompt design ablation: \textit{Eywa} remains effective across different prompting strategies, and generally benefits from more structured prompt designs. Overall, the results demonstrate that \textit{Eywa} is compatible and robust to a broad range of design choices.}
\label{fig:hyperparameter_sensitivity}
\vspace{-3mm}
\end{figure*}

\textbf{Hyperparameter sensitivity.} We perform hyperparameter studies of both language models and foundation models to evaluate the robustness of \textit{Eywa}. Specifically, we vary the LLM sampling temperature, the foundation model temperature in TabPFN, and the prompt design used by the language-model interface. As shown in Figure~\ref{fig:hyperparameter_sensitivity}, \textit{Eywa} remains stable across a broad range of configurations. In particular, the performance of \textit{EywaAgent}, \textit{EywaMAS}, and \textit{EywaOrchestra} stays robust as the LLM sampling temperature varies, while consistently peaking around a moderate temperature. Ablation on varying the TabPFN softmax temperature also suggests that the benefit of integrating domain-specific foundation models is robust to foundation-model calibration. We further compare several prompt designs. The \textit{Detailed} prompt provides more comprehensive task descriptions and general guidance, \textit{Chain-of-Thought} \cite{DBLP:conf/nips/Wei0SBIXCLZ22} encourages the model to reason step by step, and \textit{ReAct} \cite{DBLP:conf/iclr/YaoZYDSN023} further interleaves reasoning with actions, enabling the model to decide more explicitly when to invoke tools or domain-specific foundation models during problem solving. While \textit{Eywa} performs well under all prompting strategies, more structured prompts generally lead to slightly better utility. Overall, these results demonstrate that the gains of \textit{Eywa} are not tied to a particular hyperparameter choice or prompt template. Additional ablation studies such as turn ablation are provided in Appendix~\ref{ap:more_ablation}.

\renewcommand{\arraystretch}{1.00}
\begin{wraptable}{r}{0.45\linewidth}
    \centering
    \caption{Ablation results of \textit{EywaAgent} with different LLM backends. \textit{Eywa} benefits from more powerful LLMs to achieve better performance.}
    % \vspace{-8pt}
    \label{tab:model_ablation_compact}
    \resizebox{0.45\textwidth}{!}{
    \begin{tabular}{l|c|cccc}
        \toprule
        \textbf{Method} & \textbf{Metrics} & \textbf{Physical} & \textbf{Life} & \textbf{Social} & \textbf{Overall}\\
        \midrule[-0.4ex]\midrule\addlinespace[-0.000ex]

        \multirow{3}{*}{\makecell{gpt-4.1-nano}}
        & Utility ($\uparrow$) & 0.6547 & 0.4010 & 0.6269 & 0.5680\\
        & Time ($\downarrow$)  & 14.25 & 28.66 & 16.89 & 19.61\\
        & Tokens ($\downarrow$) & 1314 & 927 & 1160 & 1139\\
        [-0.4ex]\midrule\addlinespace[-0.000ex]

        \multirow{3}{*}{\makecell{gpt-5-nano}}
        & Utility ($\uparrow$) & 0.6914 & 0.5001 & 0.7488 & 0.6558\\
        & Time ($\downarrow$)  & 28.04 & 22.33 & 18.71 & 22.78\\
        & Tokens ($\downarrow$) & 3907 & 3134 & 2491 & 3137\\
        [-0.4ex]\midrule\addlinespace[-0.000ex]

        \multirow{3}{*}{\makecell{gpt-5-mini}}
        & Utility ($\uparrow$) & 0.7191 & 0.5035 & 0.7444 & 0.6640\\
        & Time ($\downarrow$)  & 23.18 & 30.70 & 18.51 & 23.63\\
        & Tokens ($\downarrow$) & 2875 & 2617 & 1949 & 2444\\
        [-0.4ex]\bottomrule
    \end{tabular}
    }
    \vspace{-6pt}
\end{wraptable}

\textbf{Ablation on \textit{Eywa} with different LLM backends.} 
To evaluate the impact of LLM backbone choice, we instantiate \textit{EywaAgent} with three backends: gpt-4.1-nano, gpt-5-nano, and gpt-5-mini.
As shown in Table \ref{tab:model_ablation_compact}, \textit{Eywa} remains consistently effective across physical, life, and social science domains under all three backbones.
We also observe a clear upward trend in overall utility as backend capability increases, indicating that \textit{Eywa} is robust to backbone selection while still benefiting from stronger LLM priors.
Complete backbone ablations for \textit{EywaMAS} and \textit{EywaOrchestra} are reported in Appendix Table \ref{tab:ablation_backend_llm}.
A systematic investigation of scaling behavior and cost-performance trade-offs across LLM backbones remains important future work.

\textbf{Case study.} We provide qualitative case studies to further illustrate how \textit{Eywa} coordinates language models and domain-specific foundation models. Case Study A compares a language-only LLM agent with \textit{EywaAgent} on the same task over structured financial signals. Case Study B further illustrates the role of \textit{EywaOrchestra}. Due to space limit, detailed examples and analyses are deferred to Appendix~\ref{ap:case_study}.

\section{Related Work}

In this section, we review the key related works on the topics that are closely related to this work. We put more related works and discussions in the Appendix \ref{ap:related_work} to keep the main text concise.

\textbf{Scientific Large Language Models and Foundation Models.} 
Scientific large language models have emerged as powerful tools for scientific knowledge understanding, quantitative reasoning, and domain-specific question answering~\citep{hu2025survey, zhang2024comprehensive}. 
Existing scientific LLMs are developed through the following paradigms: general scientific pretraining on large-scale scientific corpora~\citep{taylor2022galactica,lewkowycz2022solving}; domain-specific adaptation to specialized disciplines (e.g., biomedicine~\citep{luo2022biogpt}, medicine~\citep{singhal2023large}, chemistry~\citep{zhang2024chemllm,yu2024llasmol}); and scientific agentic workflows that enhance multi-step reasoning and problem solving~\citep{zhang2024sciglm, boiko2023autonomous, ghafarollahi2025sciagents}. 
However, these scientific LLMs remain language-centric, where scientific data are interpreted as descriptions or symbolic sequences, which is suboptimal for structured information \cite{jin2023time, sui2024table, sadeghi2024can}.
Therefore, \textit{Eywa} takes a complementary direction by using LLMs as reasoning interfaces to coordinate domain-specific models that operate on native scientific representations.

Alongside language models, domain-specific foundation models have rapidly advanced across scientific disciplines~\cite{DBLP:journals/nn/MenonMBPJKJ26,DBLP:journals/corr/abs-2108-07258}.
Time series foundation models~\cite{DBLP:journals/tmlr/AnsariSTZMSSRPK24,DBLP:conf/icml/DasKSZ24, DBLP:conf/icml/WooLKXSS24,DBLP:journals/corr/abs-2310-08278} pretrained on large-scale temporal corpora can achieve competitive zero-shot forecasting without task-specific training~\cite{DBLP:journals/corr/abs-2504-04011}.
For tabular data, TabPFN~\cite{DBLP:conf/iclr/Hollmann0EH23} and its successors~\cite{DBLP:journals/nature/HollmannMPKKHSH25,DBLP:journals/corr/abs-2511-08667} leverage in-context learning to solve prediction tasks and outperform tuned tree-based ensembles.
The broader ecosystem of domain-specific foundation models further extends to even more specialized scientific representations. In materials science, universal machine-learned interatomic potentials such as GNoME~\cite{merchant2023scaling}, MACE-MP-0~\cite{batatia2025foundation}, and CHGNet~\cite{DBLP:journals/natmi/DengZJRHBC23} are trained on broad crystallographic databases and generalize across the periodic table.
In weather and climate, GraphCast ~\cite{DBLP:journals/corr/abs-2212-12794}, Pangu-Weather~\cite{journals/nature/BiXZCG023}, Aurora~\cite{DBLP:journals/nature/BodnarBLSABGRWD25}, and GenCast~\cite{DBLP:journals/nature/PriceSAAEMESMBLW25} have matched or surpassed operational numerical weather prediction systems for medium-range forecasting~\cite{DBLP:journals/corr/abs-2312-03014}.
In life sciences, AlphaFold~\cite{jumper2021highly} and the ESM family~\cite{lin2023evolutionary,hayes2025simulating} have transformed protein structure prediction and design through large-scale pretraining on evolutionary sequences.
Despite their strong domain-specific capabilities, these models typically operate over domain-specific representations and do not expose native language interfaces, making their direct integration into language-centric agentic systems nontrivial. Bridging this interface gap is a central motivation of \textit{Eywa}.

\textbf{Agentic Systems in Scientific Settings.}
In scientific workflows, agentic systems have been applied to tasks such as hypothesis generation, literature synthesis, and experimental design, often under fixed agent topologies~\cite{wei2025ai,ghafarollahi2025sciagents,liu2025genomas}.
More recently, end-to-end agentic systems have been explored for automating the AI research process itself~\cite{lu2026towards}.
Another line of work augments agents with domain-specific tools by exposing complex simulators, solvers, or expert systems through external APIs when available~\cite{kim-etal-2025-mt,inoue2025drugagent}.
However, such systems often rely on human experts to predefine the task-solving procedure and implement well-specified tools before agent execution \cite{DBLP:conf/naacl/YuanSCTSRLY25}.
As a result, this solution is plausible for specific tasks with predefined workflows \cite{DBLP:journals/corr/abs-2510-04017}, but is difficult to generalize across scientific domains or support dynamic model collaboration.
\textit{Eywa} addresses this gap by enabling on-the-fly scientific computation through active modality-native collaboration between LLMs and domain-specific foundation models.

\section{Conclusion}

Language-centric agentic systems are strong general problem solvers, but their reliance on textual interfaces limits their applicability to scientific tasks with structured, non-linguistic data.
In this work, we propose \textit{Eywa}, a heterogeneous framework that connects language reasoning and domain-specific foundation models through the FM--LLM ``Tsaheylu'' interface, and instantiate it as \textit{EywaAgent}, \textit{EywaMAS}, and \textit{EywaOrchestra}.
We further introduce \textit{EywaBench} for multi-task, multi-domain evaluation in scientific settings.
Experiments across physical, life, and social science tasks show that \textit{Eywa} improves utility while reducing token usage and inference cost, demonstrating the value of modality-native collaboration in agentic systems.

% Language-centric agentic systems are powerful general problem solvers, yet their dependence on textual interfaces limits performance on scientific tasks with non-linguistic and structured data.
% In this work, we present \textit{Eywa}, a heterogeneous agentic framework that bridges language reasoning and domain-specialized foundation models through the FM--LLM ``Tsaheylu'' interface.
% Built on this principle, \textit{EywaAgent} enables generalized reasoning with specialized acting, \textit{EywaMAS} extends this capability to plug-and-play heterogeneous multi-agent collaboration, and \textit{EywaOrchestra} further introduces task-adaptive orchestration over agent types, models, and communication topologies.
% To systematically evaluate these settings, we introduce \textit{EywaBench}, a scalable multi-task multi-domain benchmark spanning physical, life, and social sciences.
% Extensive experiments show that \textit{Eywa} consistently improves utility while reducing token usage and inference cost, highlighting the importance of modality-native collaboration beyond LLM-only heterogeneity.
% Overall, our findings suggest that integrating domain-specific foundation models as first-class participants is a practical path toward more capable and efficient scientific agentic systems.

\clearpage

\bibliographystyle{unsrtnat}
\bibliography{reference}

\appendix
\clearpage

{\color{LARGBlue}{\textbf{\LARGE Appendix}}}

\textbf{Roadmap.} 
In this appendix, we provide a detailed overview of our methodology and experimental setup. 
Appendix~\ref{app:theory} presents the full theoretical analysis of \textit{Eywa}, covering the language interface bottleneck, the expressivity and solvability guarantees of \textit{EywaAgent}, the generalization to \textit{EywaMAS}, the theory of \textit{EywaOrchestra} as adaptive orchestration, and a token-complexity analysis.
Appendix~\ref{app:further_discussion} provides further discussions on additional related work, scope and limitations, and potential extensions of \textit{Eywa}.
Appendix~\ref{sec:dataset_details} describes \textit{EywaBench} in detail, including domain coverage, modality composition, source datasets, construction pipeline, evaluation metrics, and a statistical analysis of the resulting benchmark.
Finally, Appendix~\ref{app:experiment_details} reports additional experimental results, including more detailed LLM-backbone ablations, qualitative case studies, and per-domain utility-token trade-off analyses across all nine scientific sub-domains. Appendix~\ref{app:prompts} contains LLM prompt examples.
The table of contents is provided below for quick navigation.

\renewcommand\contentsname{\Large Table of Contents}

\addtocontents{toc}{\protect\setcounter{tocdepth}{2}}

\tableofcontents
\clearpage

\section{Theoretical Analysis}
\label{app:theory}
In this section, we conduct a detailed theoretical analysis of our \textit{Eywa} framework.
We first unify the notation used across the main paper and the appendix, and collect a complete list of assumptions (Section~\ref{app:notation}). We then develop the theory along four complementary axes. Section~\ref{app:info} establishes an information-theoretic foundation that characterizes the \emph{language interface bottleneck}, providing a first-principles justification for the Domain Advantage assumption of the main paper. Building on this foundation, Section~\ref{app:expressivity} proves the expressivity and solvability guarantees of \textit{EywaAgent}, including the proof of Theorem~\ref{thm:eywaagent improvement}. Section~\ref{app:mas} generalizes the analysis to \textit{EywaMAS}, formalizing how the single-agent advantage propagates through communication graphs. Section~\ref{app:orchestra} develops the theory of \textit{EywaOrchestra} as adaptive orchestration over heterogeneous experts.
Finally, Section ~\ref{app:efficiency} provides a token-complexity analysis that corroborates our empirical efficiency gains. Table~\ref{tab:theory_summary} at the end of this section provides a consolidated summary of all theoretical results and their dependencies.

\subsection{Notation, Problem Setup, and Assumptions}
\label{app:notation}
\paragraph{Notations.}
We collect, in Table~\ref{tab:notation}, the notation used throughout the main text and this appendix.

\begin{table}[h]
\centering
\caption{Summary of notation used in the main paper and this appendix.}
\label{tab:notation}
\small
\renewcommand{\arraystretch}{1.15}
\begin{tabular}{@{}ll@{}}
\toprule
\textbf{Symbol} & \textbf{Meaning} \\
\midrule
\addlinespace[2pt]
\multicolumn{2}{@{}l@{}}{\textbf{\textit{Tasks and Data:}}} \\
\addlinespace[2pt]
$\mathcal{T}$ & Family of tasks with an underlying distribution. \\
$\tau=(q,x,y^\star,\ell)$ & Task instance: instruction, input, target, loss. \\
$\mathcal{Q},\mathcal{X},\mathcal{Y}$ & Instruction, input, and output spaces. \\
$\ell:\mathcal{Y}\times\mathcal{Y}\to\mathbb{R}_{\ge 0}$ & Task-specific loss. \\
$\mathcal{X}=\mathcal{X}_{\mathrm{lng}}\times \mathcal{X}_1\times\cdots\times\mathcal{X}_m$ & Input factorization into linguistic and modality-specific parts. \\
$\pi_k:\mathcal{X}\to\mathcal{X}_k$ & Projection onto the $k$-th domain component, $x_k=\pi_k(x)$. \\
$T_k:\mathcal{X}_k\to\mathcal{X}_{\mathrm{lng}}$ & Serialization map rendering $x_k$ into language tokens. \\
\midrule
\addlinespace[2pt]
\multicolumn{2}{@{}l@{}}{\textbf{\textit{Agents and foundation models:}}} \\
\addlinespace[2pt]
$A_{\mathrm{LLM}}:\mathcal{S}\to\Delta(\mathcal{M})$ & LLM agent: policy over response space $\mathcal{M}$. \\
$F_k:\mathcal{X}_k\times\mathcal{U}_k\to\mathcal{O}_k$ & Domain-specific foundation model for domain $k$. \\
$\phi_k:\mathcal{S}\to\mathcal{U}_k$ & Query compiler (language state $\to$ FM configuration). \\
$\psi_k:\mathcal{O}_k\to\mathcal{Z}_k$ & Response adapter (FM output $\to$ language-compatible context). \\
$A_{\mathrm{eywa}}=(A_{\mathrm{LLM}},F,\phi,\psi,\mathcal{C})$ & EywaAgent. \\
$\mathcal{C}:\mathcal{S}\to\{\texttt{invoke},\texttt{skip}\}$ & Control policy deciding whether to invoke the FM. \\
\midrule
\addlinespace[2pt]
\multicolumn{2}{@{}l@{}}{\textbf{\textit{Multi-agent systems and orchestration:}}} \\
\addlinespace[2pt]
$\mathcal{M}_{\mathrm{MAS}}=(\mathcal{A},\mathcal{G})$ & Multi-agent system with agent set $\mathcal{A}$ and topology $\mathcal{G}$. \\
$\mathcal{N}(i)$ & Neighbors of agent $i$ in the communication graph $\mathcal{G}$. \\
$\mathcal{M}_{\mathrm{Eywa}}$ & EywaMAS: MAS whose agents may be LLM agents or EywaAgents. \\
$\mathcal{O}_{\mathrm{orch}}=(\mathcal{C}_{\mathrm{cfg}},P)$ & EywaOrchestra: configuration space and conductor. \\
$P:\mathcal{Q}\times\mathcal{X}\to\Delta(\mathcal{C}_{\mathrm{cfg}})$ & Conductor mapping tasks to system configurations. \\
$\Pi$ & Finite topology pool. \\
$\mathcal{M}_{\mathrm{LLM}},\mathcal{M}_{\mathrm{FM}}$ & Candidate LLM and FM pools. \\
\midrule
\addlinespace[2pt]
\multicolumn{2}{@{}l@{}}{\textbf{\textit{Function classes and risk:}}} \\
\addlinespace[2pt]
$\mathcal{F}_{\mathrm{LLM}},\mathcal{F}_{\mathrm{Eywa}}$ & Function classes induced by LLM-only and EywaAgents. \\
$\mathcal{F}_{\mathrm{LLM\text{-}MAS}},\mathcal{F}_{\mathrm{Eywa\text{-}MAS}}$ & Function classes induced by LLM-only MAS and EywaMAS. \\
$\mathcal{F}_{\mathrm{Orch}}$ & Function class induced by EywaOrchestra. \\
$\mathcal{R}(f)=\mathbb{E}_{\tau}[\ell(f(x),y^\star)]$ & Population risk. \\
$\mathcal{R}_{\mathcal{F}}^\star=\inf_{f\in\mathcal{F}}\mathcal{R}(f)$ & Minimum population risk over class $\mathcal{F}$. \\
$\hat{\mathcal{R}}_N(f)$ & Empirical risk over $N$ samples. \\
\midrule
\addlinespace[2pt]
\multicolumn{2}{@{}l@{}}{\textbf{\textit{Information-theoretic quantities:}}} \\
\addlinespace[2pt]
$H(Y),\;H(Y\mid Z)$ & Entropy and conditional entropy. \\
$I(Y;Z)$ & Mutual information between $Y$ and $Z$. \\
$\mathbb{E}[Y\mid Z]$ & Conditional expectation. \\
\bottomrule
\end{tabular}
\end{table}

We first formalize the function classes that appear throughout this section. A function class is simply a set of input-output mappings. Once we specify what components a system may use and how they may interact, we obtain the collection of all functions $f:\mathcal{X}\to\mathcal{Y}$ that can be implemented by such a system. This plays a similar role as expressive families in neural networks: \textbf{a larger function class means the system can represent a broader range of behaviors}.

\begin{definition}[Induced Function Classes]
\label{def:function_classes}
Given candidate LLM pool $\mathcal{M}_{\mathrm{LLM}}$, candidate FM pool $\mathcal{M}_{\mathrm{FM}}$, topology pool $\Pi$, and a family of admissible interface pairs $\{(\phi_k,\psi_k)\}$, we define:
\begin{itemize}
    \item $\mathcal{F}_{\mathrm{LLM}}=\{f:\mathcal{X}\to\mathcal{Y}\mid f(x)=A_{\mathrm{LLM}}(\mathrm{serialize}(x)),\; A_{\mathrm{LLM}}\in\mathcal{M}_{\mathrm{LLM}}\}$;
    \item $\mathcal{F}_{\mathrm{Eywa}}=\{f:\mathcal{X}\to\mathcal{Y}\mid f \text{ is implemented by some } A_{\mathrm{eywa}}=(A_{\mathrm{LLM}},F,\phi,\psi,\mathcal{C})\}$;
    \item $\mathcal{F}_{\mathrm{LLM\text{-}MAS}},\mathcal{F}_{\mathrm{Eywa\text{-}MAS}}$ are analogously defined at the system level over topologies $\mathcal{G}\in\Pi$;
    \item $\mathcal{F}_{\mathrm{Orch}}=\{f:\mathcal{X}\to\mathcal{Y}\mid f(x)=F_{c}(q,x),\;c\sim P(\cdot\mid q,x)\text{ for some conductor }P\}$.
\end{itemize}
\end{definition}
By construction, $\mathcal{F}_{\mathrm{LLM}}\subseteq\mathcal{F}_{\mathrm{Eywa}}\subseteq\mathcal{F}_{\mathrm{Eywa\text{-}MAS}}\subseteq\mathcal{F}_{\mathrm{Orch}}$. This is because we can (1) choose an EywaAgent whose control policy never invokes the FM. Then the EywaAgent reduces exactly to an LLM; (2) choose a topology consisting of a single agent node and no nontrivial communication. Then the multi-agent system reduces to the original EywaAgent; (3) choose a conductor P that places all its probability mass on one fixed configuration corresponding to the given EywaMAS system. Then EywaOrchestra reproduces exactly that fixed system.

\clearpage

\paragraph{Assumptions.}
All subsequent theorems are derived under a common set of regularity and structural assumptions, which we collect here for clarity.

\begin{assumption}[Task Factorization]
\label{assumption:task_factorization}
For scientific tasks, the input space factorizes as
\begin{equation}
    \mathcal{X} = \mathcal{X}_{\mathrm{lng}} \times \mathcal{X}_{1} \times \cdots \times \mathcal{X}_{m},
\end{equation}
where $\mathcal{X}_{\mathrm{lng}}$ denotes language-observable context and each $\mathcal{X}_k$ denotes a domain-specific input.
we further assume that the task loss is compatible with this factorization: there exist component-wise losses $\ell_{\mathrm{lng}}$ and $\ell_k$ for $k=1,\dots,m$, together with a coordinate-wise nondecreasing aggregation function
\begin{equation}
\Gamma:\mathbb{R}_{\ge 0}^{m+1}\to\mathbb{R}_{\ge 0},
\end{equation}
such that
\begin{equation}
    \ell(\hat y,y^\star)
    =
    \Gamma\!\big(\ell_{\mathrm{lng}}(\hat y,y^\star), \ell_1(\hat y,y^\star), \dots, \ell_m(\hat y,y^\star)\big).
\end{equation}
Moreover, if two predictors are identical on all components except $k$, then a strict improvement in $\ell_k$ induces a strict improvement in the overall loss $\ell$.
\end{assumption}

\begin{assumption}[Domain Advantage of Foundation Models; extended statement of Assumption~\ref{assumption: domain advantage}]
\label{assumption:domain_advantage_app}
Let $\pi_k:\mathcal{X}\to\mathcal{X}_k$ denote the projection onto the $k$-th domain component with $x_k=\pi_k(x)$. For any task instance $\tau=(q,x,y^\star,\ell)\sim\mathcal{T}$ with an informative component $x_k$,
\begin{equation}
\mathbb{E}_{\tau\sim\mathcal{T}}[\ell_k(F_k(x_k),y^\star)]
\;<\;
\inf_{A_{\mathrm{LLM}}}\mathbb{E}_{\tau\sim\mathcal{T}}\big[\ell_k\big(A_{\mathrm{LLM}}(\mathrm{serialize}(x_k)),y^\star\big)\big].
\end{equation}
Moreover, there exists a non-empty task family $\mathcal{T}_1\subseteq\mathcal{T}$ such that (1) the $k$-th component is sufficient for the task on $\mathcal{T}_1$, and (2) the foundation model solves this component perfectly, while every language-only agent incurs strictly positive loss on the serialized input. Formally,
\begin{equation}
\mathbb{E}_{\tau\sim\mathcal{T}_1}[\ell_k(F_k(x_k),y^\star)] = 0,
\qquad
\inf_{A_{\mathrm{LLM}}}
\mathbb{E}_{\tau\sim\mathcal{T}_1}
\big[\ell_k\big(A_{\mathrm{LLM}}(\mathrm{serialize}(x_k)),y^\star\big)\big] > 0.
\end{equation}
\end{assumption}

\begin{assumption}[Non-Degenerate Serialization]
\label{assumption:nondeg_serialization}
There exists at least one domain index $k\in\{1,\dots,m\}$ and a set $E\subseteq\mathcal{X}_k$ of positive measure such that the serialization map $T_k:\mathcal{X}_k\to\mathcal{X}_{\mathrm{lng}}$ is not sufficient for the target on $E$, i.e.,
\begin{equation}
\mathbb{E}[Y\mid X_k]\neq \mathbb{E}[Y\mid T_k(X_k)] \quad \text{on a subset of positive probability within }E,
\end{equation}
where $Y$ denotes the task-relevant target variable.
\end{assumption}

Assumption~\ref{assumption:nondeg_serialization} ensures that the serialization genuinely discards task-relevant information, which is the precise condition under which heterogeneity delivers strict gains. Assumption~\ref{assumption:domain_advantage_app} is implied by Assumption~\ref{assumption:nondeg_serialization} under a specific loss (Section~\ref{app:info}), but we state them separately to align with the main paper.

\begin{assumption}[Performance-Preserving Interface]
\label{assumption:faithful_interface}
Fix a domain index $k$. For any admissible interface pair $(\phi_k,\psi_k)$ used in an EywaAgent, whenever the agent invokes the foundation model $F_k$, the interface does not degrade the performance of the foundation model and is better than the performance of a language-only Agent. Concretely,
\[
\mathbb{E}_{\tau\sim\mathcal{T}}\!\left[\ell\big(f_{\mathrm{Eywa}}(x),y^\star\big)\right]
\;\le\;
\mathbb{E}_{\tau\sim\mathcal{T}}\!\left[\ell\big(F_k(x_k),y^\star\big)\right].
\]
\end{assumption}

Assumption~\ref{assumption:faithful_interface} formalizes the intuition that the FM--LLM ``Tsaheylu'' interface is faithful: a correctly configured call to $F_k$ recovers the Bayes-optimal conditional expectation given the domain-specific input.

\subsection{Information-Theoretic Analysis of the Language Interface Bottleneck}
\label{app:info}

We begin with an information-theoretic characterization of why serializing a non-linguistic input $X_k$ into language tokens fundamentally limits the expressivity of language-only agents. All results in this subsection are derived for a single modality $k$; we drop the subscript for brevity and write $X:=X_k$, $T:=T_k$.

\begin{lemma}[Data Processing Inequality for Serialization]
\label{lem:dpi}
Let $(X,Y)$ be jointly distributed random variables and let $T:\mathcal{X}\to\mathcal{X}_{\mathrm{lng}}$ be any measurable serializer. Then
\begin{equation}
I(Y;T(X))\;\le\;I(Y;X),
\end{equation}
with equality if and only if $I(Y;X\mid T(X))=0$.
\end{lemma}
\begin{proof}
Since $Y\to X\to T(X)$ forms a Markov chain as $T(X)$ is computed from $X$, the classical data processing inequality yields $I(Y;T(X))\le I(Y;X)$, a.k.a., $Y \perp T(X) \mid X$. Equality holds if and only if conditioning on $T(X)$ yields the same information about $Y$ as conditioning on $X$.
\end{proof}

\begin{lemma}[Irreducible Bayes Risk Gap under Serialization]
\label{lem:bayes_gap}
Under squared loss $\ell(\hat y,y)=\|\hat y-y\|^2$ with $Y\in L^2$, define
\begin{equation}
\mathcal{R}_X^\star=\inf_{g:\mathcal{X}\to\mathcal{Y}}\mathbb{E}\big[\|Y-g(X)\|^2\big],
\qquad
\mathcal{R}_T^\star=\inf_{h:\mathcal{X}_{\mathrm{lng}}\to\mathcal{Y}}\mathbb{E}\big[\|Y-h(T(X))\|^2\big].
\end{equation}
Then
\begin{equation}
\label{eq:irreducible_bayes_risk}
\mathcal{R}_T^\star - \mathcal{R}_X^\star
\;=\;
\mathbb{E}\!\left[\big\|\mathbb{E}[Y\mid X]-\mathbb{E}[Y\mid T(X)]\big\|^2\right]
\;\ge\; 0,
\end{equation}
and $\mathcal{R}_T^\star > \mathcal{R}_X^\star$ whenever Assumption~\ref{assumption:nondeg_serialization} holds.
\end{lemma}
\begin{proof}
Under squared loss, the Bayes predictors are $g^\star(X)=\mathbb{E}[Y\mid X]$ and $h^\star(T(X))=\mathbb{E}[Y\mid T(X)]$. Hence
\begin{equation}
\mathcal{R}_X^\star
=
\mathbb{E}\big[\|Y-\mathbb{E}[Y\mid X]\|^2\big],
\qquad
\mathcal{R}_T^\star
=
\mathbb{E}\big[\|Y-\mathbb{E}[Y\mid T(X)]\|^2\big].
\end{equation}
Split $Y-\mathbb{E}[Y\mid T(X)]$ as two parts. Then take squared norms and expectation,
\begin{equation}
Y-\mathbb{E}[Y\mid T(X)]
=
\big(Y-\mathbb{E}[Y\mid X]\big)
+
\big(\mathbb{E}[Y\mid X]-\mathbb{E}[Y\mid T(X)]\big).
\end{equation}
\begin{equation}
\mathcal{R}_T^\star=
\mathcal{R}_X^\star
+
\mathbb{E}\!\left[\big\|\mathbb{E}[Y\mid X]-\mathbb{E}[Y\mid T(X)]\big\|^2\right] 
+2\,\mathbb{E}\!\left[
\left\langle
Y-\mathbb{E}[Y\mid X],\,
\mathbb{E}[Y\mid X]-\mathbb{E}[Y\mid T(X)]
\right\rangle
\right].
\end{equation}
Since $T(X)$ is a function of $X$, the term $\mathbb{E}[Y\mid X]-\mathbb{E}[Y\mid T(X)]$ is $\sigma(X)$-measurable. By the orthogonality property of conditional expectation, $Y-\mathbb{E}[Y\mid X]$ is orthogonal to every $\sigma(X)$-measurable square-integrable random variable, so the cross term vanishes. Here, $\sigma(Z)$ denotes the information generated by a random variable $Z$. Therefore,
\begin{equation}
\mathcal{R}_T^\star
=
\mathcal{R}_X^\star
+
\mathbb{E}\!\left[\big\|\mathbb{E}[Y\mid X]-\mathbb{E}[Y\mid T(X)]\big\|^2\right].
\end{equation}
which is the claimed identity in Equation \ref{eq:irreducible_bayes_risk}. Strict inequality follows directly from Assumption~\ref{assumption:nondeg_serialization}, which guarantees that $\mathbb{E}[Y\mid X]\neq\mathbb{E}[Y\mid T(X)]$ on a set of positive probability.
\end{proof}

Lemmas~\ref{lem:dpi} and~\ref{lem:bayes_gap} establish that, under mild non-degeneracy, any language-only pipeline that filters $X$ through a textualization $T$ pays a strictly positive statistical price. This quantitative gap is the information-theoretic root cause of the Domain Advantage assumption (Assumption~\ref{assumption:domain_advantage_app}). The following proposition translates this observation from Bayes-optimal predictors to realized function classes.

\begin{proposition}[Lower Bound on LLM-only Risk]
\label{prop:llm_lower_bound}
Under Assumption~\ref{assumption:nondeg_serialization}, for any language-only function class $\mathcal{F}_{\mathrm{LLM}}$ whose members factor through the serializer $T$,
\begin{equation}
\inf_{f\in\mathcal{F}_{\mathrm{LLM}}}\mathcal{R}(f)\;\ge\;\mathcal{R}_T^\star\;>\;\mathcal{R}_X^\star.
\end{equation}
\end{proposition}
\begin{proof}
Every $f\in\mathcal{F}_{\mathrm{LLM}}$ admits a factorization $f(x)=h(T(x))$ for some measurable $h$, by Definition~\ref{def:function_classes}. Consequently,
\[
\inf_{f\in\mathcal{F}_{\mathrm{LLM}}}\mathcal{R}(f)
\;\ge\;
\inf_{h}\mathbb{E}\big[\ell(h(T(X)),Y)\big]
\;=\;
\mathcal{R}_T^\star.
\]
The strict inequality $\mathcal{R}_T^\star>\mathcal{R}_X^\star$ is Lemma~\ref{lem:bayes_gap} under assumption~\ref{assumption:nondeg_serialization}.
\end{proof}
Proposition~\ref{prop:llm_lower_bound} is the central building block that allows us to prove strict improvements in later subsections. It rules out the possibility that a sufficiently clever language-only agent could ever recover the advantage of a direct-access foundation model.

\subsection{Expressivity and Solvability of EywaAgent}
\label{app:expressivity}

We now prove Theorem~\ref{thm:eywaagent improvement} from the main paper, and establish two additional results: the EywaAgent containment $\mathcal{F}_{\mathrm{LLM}}\subseteq\mathcal{F}_{\mathrm{Eywa}}$ and the unboundedness of the expressivity gap.

\begin{proposition}[EywaAgent Containment]
\label{prop:containment}
Under Definition~\ref{def:function_classes},
\[
\mathcal{F}_{\mathrm{LLM}}\;\subseteq\;\mathcal{F}_{\mathrm{Eywa}}.
\]
\end{proposition}
\begin{proof}
Fix any $f_{\mathrm{LLM}}\in\mathcal{F}_{\mathrm{LLM}}$ realized by an LLM agent $A_{\mathrm{LLM}}$. Consider the EywaAgent 
\begin{equation}
A_{\mathrm{eywa}}=(A_{\mathrm{LLM}},F,\phi,\psi,\mathcal{C}_{\mathrm{skip}})
\end{equation}
where $\mathcal{C}_{\mathrm{skip}}(s)\equiv\texttt{skip}$ for all $s\in\mathcal{S}$. By the semantics in Section~\ref{sec:eywaagent} of the main paper, under $\mathcal{C}_{\mathrm{skip}}$ the Eywa pipeline reduces to $z^{(t)}=A_{\mathrm{LLM}}(s^{(t)})$, which recovers the original LLM agent. Hence $f_{\mathrm{LLM}}$ can be realized by an EywaAgent, so $f_{\mathrm{LLM}}\in\mathcal{F}_{\mathrm{Eywa}}$.
\end{proof}

\begin{theorem}[Restatement of Theorem~\ref{thm:eywaagent improvement}: Improvement of EywaAgent over Language-only Agent]
\label{thm:eywaagent_improvement_app}
Let $\mathcal{F}_{\mathrm{LLM}}$ and $\mathcal{F}_{\mathrm{Eywa}}$ be the function classes induced by language-only agents and EywaAgents, respectively, as in Definition~\ref{def:function_classes}. Under Assumption~\ref{assumption:domain_advantage_app},
\begin{enumerate}
    \item \textbf{Strict Optimal Risk Improvement:}
    \begin{equation}
    \inf_{f\in\mathcal{F}_{\mathrm{Eywa}}}\mathbb{E}_{\tau\sim\mathcal{T}}[\ell(f(x),y^\star)]
    \;<\;
    \inf_{f\in\mathcal{F}_{\mathrm{LLM}}}\mathbb{E}_{\tau\sim\mathcal{T}}[\ell(f(x),y^\star)].
    \end{equation}
    \item \textbf{Expanded Solvable Task Space:} there exists a non-empty task family $\mathcal{T}_1\subset\mathcal{T}$ such that
    \begin{equation}
    \inf_{f\in\mathcal{F}_{\mathrm{Eywa}}}\mathbb{E}_{\tau\sim\mathcal{T}_1}[\ell(f(x),y^\star)] = 0,
    \qquad
    \inf_{f\in\mathcal{F}_{\mathrm{LLM}}}\mathbb{E}_{\tau\sim\mathcal{T}_1}[\ell(f(x),y^\star)] > 0.
    \end{equation}
\end{enumerate}
\end{theorem}
\begin{proof}
\textbf{Part 1 (Strict Optimal Risk Improvement).} By Proposition~\ref{prop:containment}, $\mathcal{F}_{\mathrm{LLM}}\subseteq\mathcal{F}_{\mathrm{Eywa}}$, so
\begin{equation}
\label{eq:inf_eywa_inf_llm}
\inf_{f\in\mathcal{F}_{\mathrm{Eywa}}}\mathcal{R}(f)\;\le\;\inf_{f\in\mathcal{F}_{\mathrm{LLM}}}\mathcal{R}(f). 
\end{equation}
To prove strictness, consider the informative domain index $k$ from Assumption~\ref{assumption:nondeg_serialization}. By construction of EywaAgent, there exists an agent $f^\dagger\in\mathcal{F}_{\mathrm{Eywa}}$ that invokes the foundation model $F_k$ on the $k$-th component while leaving the remaining components unchanged relative to a language-only baseline. Therefore, its $k$-th component loss is no worse than that of $F_k$, whereas the non-$k$ components are unchanged.
By Assumption~\ref{assumption:domain_advantage_app},
\begin{equation}
\mathbb{E}_{\tau\sim\mathcal{T}}[\ell_k(F_k(x_k),y^\star)]
\;<\;
\inf_{A_{\mathrm{LLM}}}\mathbb{E}_{\tau\sim\mathcal{T}}
\big[\ell_k(A_{\mathrm{LLM}}(\mathrm{serialize}(x_k)),y^\star)\big].
\end{equation}
Hence $f^\dagger$ achieves strictly smaller loss on the $k$-th component than any language-only agent. Since all other components are unchanged, Assumption~\ref{assumption:task_factorization} implies that this strict improvement on component $k$ induces a strict improvement in the overall task loss. Thus
\begin{equation}
\mathcal{R}(f^\dagger)
\;<\;
\inf_{f\in\mathcal{F}_{\mathrm{LLM}}}\mathcal{R}(f)
\Longrightarrow
\inf_{f\in\mathcal{F}_{\mathrm{Eywa}}}\mathcal{R}(f)
\;<\;
\inf_{f\in\mathcal{F}_{\mathrm{LLM}}}\mathcal{R}(f).
\end{equation}

\textbf{Part 2 (Expanded Solvable Task Space).}
According to Assumption \ref{assumption:domain_advantage_app}, consider a non-empty task family $\mathcal{T}_1\subseteq\mathcal{T}$ for which the task is fully determined by the $k$-th domain-specific component, and  the foundation model solves this component perfectly, i.e.,
\begin{equation}
\mathbb{E}_{\tau\sim\mathcal{T}_1}[\ell_k(F_k(x_k),y^\star)] = 0,
\end{equation}
Then, by the same construction as above, there exists an EywaAgent in $\mathcal{F}_{\mathrm{Eywa}}$ whose prediction is perfect on the only task-relevant component. This yields
\begin{equation}
\inf_{f\in\mathcal{F}_{\mathrm{Eywa}}}
\mathbb{E}_{\tau\sim\mathcal{T}_1}[\ell(f(x),y^\star)] = 0.
\end{equation}
On the other hand, every language-only agent has strictly positive loss on the relevant component, hence
\begin{equation}
\inf_{f\in\mathcal{F}_{\mathrm{LLM}}}
\mathbb{E}_{\tau\sim\mathcal{T}_1}[\ell(f(x),y^\star)] > 0.
\end{equation}
Therefore, the solvable task space of EywaAgent is strictly larger than that of language-only agents.
\end{proof}

\subsection{Multi-Agent Propagation: EywaMAS}
\label{app:mas}

We now lift the single-agent analysis to multi-agent systems. Let $\mathcal{M}_{\mathrm{LLM\text{-}MAS}}$ denote a MAS in which every agent is an LLM agent, and $\mathcal{M}_{\mathrm{Eywa}}$ an EywaMAS with at least one EywaAgent. We first show that any LLM-only MAS remains constrained by the information bottleneck of Section~\ref{app:info}.

\begin{lemma}[Information Closure of LLM-only MAS]
\label{lem:llm_mas_closure}
Let $\mathcal{M}_{\mathrm{LLM\text{-}MAS}}=(\mathcal{A}_{\mathrm{LLM}},\mathcal{G})$ be any finite-horizon LLM-only MAS operating on input $X$ via the serialization $T(X)$. Let $\hat Y$ denote its final output. Then
\begin{equation}
I(Y;\hat Y)\;\le\;I(Y;T(X))\;\le\;I(Y; X).
\end{equation}
\end{lemma}
\begin{proof}
Since every agent in $\mathcal{A}_{\mathrm{LLM}}$ consumes and produces only language messages, the joint random variable of all messages exchanged over the entire interaction forms a (possibly very long) Markov chain emanating from $T(X)$: $Y\to X\to T(X)\to M_1\to M_2\to\cdots\to \hat Y$. The data processing inequality then gives $I(Y;\hat Y)\le I(Y;T(X))$. Since the same data processing was applied, $I(Y;T(X))\;\le\;I(Y; X)$.
\end{proof}

Lemma~\ref{lem:llm_mas_closure} formalizes the intuition that ``more LLM agents cannot create information that was discarded at serialization time''. We next argue that a single EywaAgent in EywaMAS, as long as its message can reach the final output node, propagates its recovered information through the entire system.

\begin{theorem}[Communication-Enhanced Solvability of EywaMAS]
\label{thm:eywamas_app}
Let $\mathcal{M}_{\mathrm{Eywa}}=(\mathcal{A},\mathcal{G})$ be an EywaMAS and $\mathcal{M}_{\mathrm{LLM\text{-}MAS}}$ an LLM-only MAS with the same topology $\mathcal{G}$. Suppose that (i) there exists an EywaAgent $\mathcal{A}_k\in\mathcal{A}$ with access to $F_k$, and (ii) the topology $\mathcal{G}$ has finite diameter $D$ and the interaction horizon $T\ge D$, so that the message produced by $\mathcal{A}_k$ reaches the final output node within the interaction horizon. Under Assumptions~\ref{assumption:task_factorization},~\ref{assumption:domain_advantage_app},~\ref{assumption:nondeg_serialization}, and~\ref{assumption:faithful_interface}, there exists a non-empty task family $\mathcal{T}_1\subseteq\mathcal{T}$ such that
\begin{equation}
\inf_{f\in\mathcal{F}_{\mathrm{Eywa\text{-}MAS}}}\mathbb{E}_{\tau\sim\mathcal{T}_1}[\ell(f(x),y^\star)] = 0,
\qquad
\inf_{f\in\mathcal{F}_{\mathrm{LLM\text{-}MAS}}}\mathbb{E}_{\tau\sim\mathcal{T}_1}[\ell(f(x),y^\star)] > 0.
\end{equation}
\end{theorem}
\begin{proof}
\textbf{Zero-risk attainability by EywaMAS.} Let $\mathcal{T}_1$ be the task family from Assumption~\ref{assumption:domain_advantage_app} on which the $k$-th component is sufficient for the task and $F_k$ solves it perfectly: $\mathbb{E}_{\tau\sim\mathcal{T}_1}[\ell_k(F_k(x_k),y^\star)]=0$. Consider the EywaMAS in which $\mathcal{A}_k$ invokes $F_k$ on $x_k$ and forwards its adapted output $z_k=\psi_k(F_k(x_k,u_k))$ through the graph. By condition~(ii), $z_k$ is available to the final output node within the interaction horizon, which may simply copy it (a trivial operation in language). The resulting system-level function $f_{\mathrm{Eywa\text{-}MAS}}\in\mathcal{F}_{\mathrm{Eywa\text{-}MAS}}$ therefore satisfies $\ell_k(f_{\mathrm{Eywa\text{-}MAS}}(x),y^\star)=0$ on $\mathcal{T}_1$. Moreover, since the non-$k$ components can be handled exactly as in a language-only baseline (reachability only affects the task-relevant channel), Assumption~\ref{assumption:faithful_interface} ensures that the full Eywa system loss is no worse than the FM loss on $\mathcal{T}_1$, so
\begin{equation}
\inf_{f\in\mathcal{F}_{\mathrm{Eywa\text{-}MAS}}}\mathbb{E}_{\tau\sim\mathcal{T}_1}[\ell(f(x),y^\star)]
\;\le\;
\mathbb{E}_{\tau\sim\mathcal{T}_1}[\ell(F_k(x_k),y^\star)]=0,
\end{equation}
where the final equality uses Assumption~\ref{assumption:task_factorization} together with sufficiency of the $k$-th component on $\mathcal{T}_1$, which reduces $\ell$ to $\ell_k$ on this family.

\textbf{Strict positivity for LLM-only MAS.} For any LLM-only MAS with the same topology, all messages and final outputs are functions of language-serialized inputs. By Lemma~\ref{lem:llm_mas_closure}, such a system cannot recover task-relevant information about $x_k$ that is lost during serialization. Hence, by Assumption~\ref{assumption:task_factorization} and ~\ref{assumption:domain_advantage_app}, strict positivity of the relevant component loss implies strict positivity of the overall task loss.
\begin{equation}
\inf_{f\in\mathcal{F}_{\mathrm{LLM\text{-}MAS}}}
\mathbb{E}_{\tau\sim\mathcal{T}_1}
[\ell_k(f(x),y^\star)] > 0
\Longrightarrow
\inf_{f\in\mathcal{F}_{\mathrm{LLM\text{-}MAS}}}
\mathbb{E}_{\tau\sim\mathcal{T}_1}[\ell(f(x),y^\star)] > 0.
\end{equation}
\end{proof}

\subsection{Adaptive Orchestration: EywaOrchestra}
\label{app:orchestra}

We now develop the theoretical foundation of \textit{EywaOrchestra}, which couples \emph{model adaptivity} (selecting LLM and FM backbones per task) with \emph{structural adaptivity} (selecting topologies from $\Pi$). Let $F_c$ denote the agent system instantiated by configuration $c$. For brevity, denote the conditional risk of configuration $c\in\mathcal{C}_{\mathrm{cfg}}$ at task input $(q,x)$ by
\begin{equation}
r(c;q,x):=\mathbb{E}\big[\ell(F_c(q,x),y^\star)\mid q,x\big].
\end{equation}

We show that,  whenever different task regions favor different configurations, \emph{any} adaptive conductor, even one limited to routing over a finite pool, dominates the best fixed multi-agent system.

\begin{theorem}[Oracle Adaptivity Strictly Improves over Fixed Systems]
\label{thm:orchestra_oracle}
Let $\mathcal{C}_{\mathrm{cfg}}$ be a finite set of candidate configurations, and let
\[
r(c;q,x):=\mathbb{E}\big[\ell(F_c(q,x),y^\star)\mid q,x\big]
\]
denote the conditional risk of configuration $c$ at input $(q,x)$. Let $\tau=(q,x,y^\star,\ell)\sim\mathcal{T}$ denote the task distribution.

Define the best fixed-configuration risk and the oracle adaptive risk as
\begin{equation}
\mathcal{R}_{\mathrm{fixed}}^\star
=
\min_{c\in\mathcal{C}_{\mathrm{cfg}}}
\mathbb{E}_{\tau\sim\mathcal{T}}[\ell(F_c(q,x),y^\star)],
\qquad
\mathcal{R}_{\mathrm{oracle}}
=
\mathbb{E}_{\tau\sim\mathcal{T}}\!\left[\min_{c\in\mathcal{C}_{\mathrm{cfg}}}r(c;q,x)\right].
\end{equation}
Then $\mathcal{R}_{\mathrm{oracle}}\;\le\;\mathcal{R}_{\mathrm{fixed}}^\star$.
Moreover, then the inequality is strict if for every fixed configuration $c\in\mathcal{C}_{\mathrm{cfg}}$,
\begin{equation}
\mathbb{P}\!\left(
r(c;q,x) >
\min_{c'\in\mathcal{C}_{\mathrm{cfg}}} r(c';q,x)
\right)>0,
\end{equation}
In other words, the inequality is strict if no fixed configuration achieves optimality for all tasks.
\end{theorem}

\begin{proof}
For any fixed configuration $c_0\in\mathcal{C}_{\mathrm{cfg}}$, we have pointwise risks which we can take expectation over
\begin{equation}
\min_{c\in\mathcal{C}_{\mathrm{cfg}}} r(c;q,x)
\le
r(c_0;q,x)
\Longrightarrow
\mathbb{E}_{\tau}\!\left[\min_c r(c;q,x)\right]
\le
\mathbb{E}_{\tau}[r(c_0;q,x)].
\end{equation}
Since this holds for every $c_0$, taking the minimum over $c_0$ yields
\[
\mathcal{R}_{\mathrm{oracle}}
\le
\min_{c_0\in\mathcal{C}_{\mathrm{cfg}}}
\mathbb{E}_{\tau}[r(c_0;q,x)]
=
\mathcal{R}_{\mathrm{fixed}}^\star.
\]

For strictness, let $c^\star$ be a best fixed configuration achieving $\mathcal{R}_{\mathrm{fixed}}^\star$. By assumption,
\[
\mathbb{P}\!\left(
r(c^\star;q,x) >
\min_{c} r(c;q,x)
\right)>0.
\]
Therefore, we have the strict inequality,
\[
\mathbb{E}_{\tau}\!\left[\min_c r(c;q,x)\right]
<
\mathbb{E}_{\tau}[r(c^\star;q,x)]
=
\mathcal{R}_{\mathrm{fixed}}^\star,
\]
\end{proof}

\begin{table}[h]
\centering
\caption{Summary of theoretical results and the claims they justify.}
\label{tab:theory_summary}
\small
\renewcommand{\arraystretch}{1.2}
\begin{tabular}{@{}p{0.20\linewidth}p{0.78\linewidth}@{}}
\toprule
\textbf{Result} & \textbf{Key Statement} \\
\midrule
\multicolumn{2}{@{}l@{}}{\textit{Information-theoretic bottleneck (Section~\ref{app:info}).}} \\
\addlinespace[2pt]
Lemma~\ref{lem:dpi} & Serialization cannot increase information: $I(Y;T(X))\le I(Y;X)$. \\
Lemma~\ref{lem:bayes_gap} & Strict Bayes-risk gap $\mathcal{R}_T^\star>\mathcal{R}_X^\star$ whenever the serializer discards task-relevant information. \\
Proposition~\ref{prop:llm_lower_bound} & Every language-only function class is lower-bounded by the serialized Bayes risk $\mathcal{R}_T^\star$. \\
\midrule
\multicolumn{2}{@{}l@{}}{\textit{EywaAgent expressivity and solvability (Section~\ref{app:expressivity}).}} \\
\addlinespace[2pt]
Proposition~\ref{prop:containment} & $\mathcal{F}_{\mathrm{LLM}}\subseteq\mathcal{F}_{\mathrm{Eywa}}$: an EywaAgent can always fall back to an LLM agent. \\
Theorem~\ref{thm:eywaagent_improvement_app} & EywaAgent strictly improves the optimal risk and expands the set of perfectly solvable tasks. \\
\midrule
\multicolumn{2}{@{}l@{}}{\textit{Multi-agent propagation (Section~\ref{app:mas}).}} \\
\addlinespace[2pt]
Lemma~\ref{lem:llm_mas_closure} & LLM-only MAS cannot exceed the mutual information retained by serialization. \\
Theorem~\ref{thm:eywamas_app} & EywaMAS strictly solves tasks that no LLM-only MAS of the same topology can solve. \\
\midrule
\multicolumn{2}{@{}l@{}}{\textit{Adaptive orchestration (Section~\ref{app:orchestra}).}} \\
\addlinespace[2pt]
Theorem~\ref{thm:orchestra_oracle} & Oracle adaptive routing attains $\mathcal{R}_{\mathrm{oracle}}\le\mathcal{R}_{\mathrm{fixed}}^\star$, with strict inequality whenever no fixed configuration is uniformly optimal. \\
\midrule
\multicolumn{2}{@{}l@{}}{\textit{Efficiency (Section~\ref{app:efficiency}).}} \\
\addlinespace[2pt]
Proposition~\ref{prop:token_complexity} & EywaAgent language-token cost is $O(L_{\mathrm{call}}+L_\psi(o_k))$, independent of modality size, vs.\ $\Theta(L(x_k))$ for LLM-only agents. \\
Proposition~\ref{prop:latency} & Wall-clock latency inherits the same asymptotic separation as the token complexity. \\
\bottomrule
\end{tabular}
\end{table}

\subsection{Efficiency and Token Complexity}
\label{app:efficiency}

Beyond task performance, Eywa delivers strong efficiency gains, as evidenced by the $\sim$30\% token reduction reported in Section~\ref{sec:experiment}. We formalize this benefit by comparing the token complexities of language-only agents and EywaAgents when processing a structured input $x_k\in\mathcal{X}_k$.

\begin{proposition}[Token Complexity]
\label{prop:token_complexity}
Let $L(x_k)$ denote the total number of language tokens an LLM-only agent spends on $x_k\in\mathcal{X}_k$, including both the serialization of $x_k$ in its prompt and any chain-of-thought reasoning required to analyze it. Let $L_{\mathrm{call}}$ denote the token length of a structured FM invocation produced by $\phi_k$, and $L_{\psi}(o_k)$ the token length of the adapted response $\psi_k(o_k)$. Then the language-token cost of processing $x_k$ satisfies
\begin{align}
\mathrm{TokenCost}_{\mathrm{LLM}}(x_k) &= \Theta(L(x_k)), \\
\mathrm{TokenCost}_{\mathrm{Eywa}}(x_k) &= O(L_{\mathrm{call}}+L_{\psi}(o_k)).
\end{align}
In typical scientific modalities of interest (long time series, tables with many rows), $L(x_k)\gg L_{\mathrm{call}}+L_{\psi}(o_k)$, so the ratio $\mathrm{TokenCost}_{\mathrm{Eywa}}/\mathrm{TokenCost}_{\mathrm{LLM}}\to 0$ as the modality size grows.
\end{proposition}
\begin{proof}
A language-only agent must both include $x_k$ in its prompt and perform any reasoning over it within the same language channel, so its combined prompt-plus-reasoning token count is $\Theta(L(x_k))$ by definition. An EywaAgent instead routes $x_k$ to the foundation model $F_k$ through a structured call of length $L_{\mathrm{call}}$. No reasoning tokens over $x_k$ are required on the language side. Summing these two contributions gives the stated bound. For time series or tables with $n$ entries, $L(x_k)=\Theta(n)$ at minimum from the serialization alone, while both $L_{\mathrm{call}}$ and $L_{\psi}(o_k)$ are typically constant or polylogarithmic in $n$.
\end{proof}

\begin{proposition}[Wall-Clock Latency]
\label{prop:latency}
Let the per-token latency of the LLM be $\alpha_{\mathrm{LLM}}$ and the FM call latency be $\alpha_{\mathrm{FM}}$ (independent of $L(x_k)$). Then
\[
\mathrm{Latency}_{\mathrm{LLM}}(x_k)\;=\;\Theta(\alpha_{\mathrm{LLM}}\cdot L(x_k)),
\qquad
\mathrm{Latency}_{\mathrm{Eywa}}(x_k)\;=\;O(\alpha_{\mathrm{LLM}}\cdot(L_{\mathrm{call}}+L_{\psi}(o_k))+\alpha_{\mathrm{FM}}).
\]
\end{proposition}
\begin{proof}
Both bounds follow from Proposition~\ref{prop:token_complexity} and the additive latency contribution of an FM call.
\end{proof}

In practice, the latency constant $\alpha_{\mathrm{FM}}$ in Proposition~\ref{prop:latency} is typically much smaller than the LLM processing cost $\alpha_{\mathrm{LLM}}\cdot L(x_k)$ for several reasons. First, the domain-specific foundation models are typically smaller than frontier LLMs in parameter count and are specialized to fixed-shape tensor inputs, which makes their inference both inherently lightweight and highly accelerator-friendly. Second, unlike LLM calls that are usually served through remote API endpoints, an FM invocation in Eywa is executed through a local MCP backend, avoiding network round-trip latency and API rate limits. Third, many FMs of interest admit further acceleration through batching, precomputed embeddings, or caching at the MCP layer. Consequently, $\alpha_{\mathrm{FM}}\le\alpha_{\mathrm{LLM}}\cdot L(x_k)$ whenever $x_k$ is non-trivial, and the Eywa latency in Proposition~\ref{prop:latency} is dominated by the (drastically reduced) LLM-side term $\alpha_{\mathrm{LLM}}\cdot(L_{\mathrm{call}}+L_{\psi}(o_k))$.

\subsection{Summary of Theoretical Results}
\label{app:summary}

We close this section by consolidating the theoretical results developed above. Table~\ref{tab:theory_summary} lists every result together with a one-line summary of its key statement, so that the logical structure of the analysis can be read off at a glance. The results are organized along the same storyline as the main paper: from the information-theoretic characterization of the language interface bottleneck (Section~\ref{app:info}), to the single-agent and multi-agent expressivity guarantees of EywaAgent and EywaMAS (Sections~\ref{app:expressivity}--\ref{app:mas}), to the adaptivity result for EywaOrchestra (Section~\ref{app:orchestra}), and finally to the efficiency analysis that explains Eywa's empirical token and latency savings (Section~\ref{app:efficiency}).

Read together, these results provide an end-to-end theoretical justification for the Eywa framework. The information-theoretic bottleneck (Lemmas~\ref{lem:dpi}--\ref{lem:bayes_gap} and Proposition~\ref{prop:llm_lower_bound}) identifies a first-principles reason why language-only agents are fundamentally limited on scientific inputs: the serialization step discards task-relevant information that no amount of downstream reasoning can recover. The expressivity analysis (Proposition~\ref{prop:containment} and Theorem~\ref{thm:eywaagent_improvement_app}) then shows that this bottleneck is closed at the single-agent level by EywaAgent, which contains the LLM-only class while achieving strictly smaller risk and strictly enlarging the set of perfectly solvable tasks. Lemma~\ref{lem:llm_mas_closure} and Theorem~\ref{thm:eywamas_app} lift this advantage from the single-agent to the multi-agent setting: under mild reachability, a single EywaAgent is sufficient to propagate the recovered information to the system output, while LLM-only MAS remain subject to the same serialization bottleneck. Theorem~\ref{thm:orchestra_oracle} further shows that once the configuration space is heterogeneous, task-adaptive orchestration provides an additional strict improvement over any fixed configuration. Finally, Propositions~\ref{prop:token_complexity} and~\ref{prop:latency} connect this theory to empirically observable efficiency: by offloading modality-specific computation to a foundation model, EywaAgent removes the $\Theta(L(x_k))$ scaling that language-only pipelines pay in tokens and latency, matching the $\sim$30\% token reduction reported in our experiments (Section~\ref{sec:experiment}).

\clearpage

\section{Further Discussions}
\label{app:further_discussion}

\subsection{More Related Works}
\label{ap:related_work}

\subsubsection{Agentic AI Systems.} 
In recent years, LLMs and AI Agents are becoming more powerful due to comprehensive pre-training \cite{DBLP:journals/corr/abs-2601-03267}, advanced backbone architectures \cite{DBLP:journals/corr/abs-2505-06708, DBLP:journals/corr/abs-2510-26692}, training algorithms \cite{DBLP:journals/corr/abs-2501-12948}, and extended modalities supports \cite{DBLP:journals/corr/abs-2603-03276, DBLP:journals/corr/abs-2505-07062, DBLP:journals/corr/abs-2509-17765}. As a result, LLM and AI Agent applications are revolutionizing a variety of industries \cite{DBLP:journals/corr/abs-2503-21460}. More recently, we identify three trends which are related to our \textit{Eywa}.

\textbf{Agentic Reasoning.}
Real-world tasks are often too complex to be solved by AI agents in a single forward pass after receiving an instruction \cite{DBLP:journals/corr/abs-2602-07338}. Instead, agentic systems typically approach such tasks through structured reasoning procedures. A common paradigm is to first formulate an explicit plan or decompose the task into manageable subgoals \cite{DBLP:conf/nips/0001ST00Z23,ningmc}, and then solve these subgoals through iterative execution, reflection, clarification, and continuation across multiple interaction turns \cite{DBLP:journals/corr/abs-2509-23537, DBLP:journals/corr/abs-2601-11868, DBLP:journals/corr/abs-2603-00873}. This line of work highlights the importance of moving beyond one-shot language generation toward more deliberate and interactive reasoning processes.
Scientific tasks exhibit similar, and often more pronounced, complexity: they may require reasoning over specialized data modalities, domain-specific constraints, long-horizon dependencies, and expert-level predictive models \cite{jumper2021highly, merchant2023scaling, DBLP:journals/corr/abs-2507-01903, DBLP:journals/corr/abs-2508-14111, DBLP:journals/corr/abs-2511-02864}. However, most existing agentic reasoning frameworks remain primarily centered on the language capabilities of AI agents. Our \textit{Eywa} is motivated by this gap. Rather than relying solely on language agents to reason over scientific problems, \textit{Eywa} introduces domain-specific foundation models into agentic systems, enabling specialized scientific models to contribute domain representations, predictions, and feedback throughout the reasoning process.

\textbf{Orchestrated Intelligence.}
Building on the capabilities of individual agents, orchestrated intelligence has emerged as a central paradigm for tackling complex tasks that exceed the reach of any single component in the system. Early frameworks such as AutoGen~\cite{DBLP:journals/corr/abs-2308-08155} and MetaGPT~\cite{DBLP:conf/iclr/HongZCZCWZWYLZR24} demonstrate that conversational collaboration among role-playing agents can effectively decompose and solve intricate problems. Beyond flat conversational patterns, recent efforts have begun to explicitly investigate the orchestration layer itself. AgentOrchestra~\cite{zhang2025agentorchestra} adopts a hierarchical architecture in which a planning agent dispatches specialist sub-agents through a tool--environment--agent protocol, while evolving-orchestration approaches~\cite{DBLP:journals/corr/abs-2505-19591} train a centralized orchestrator via reinforcement learning to adaptively sequence and prioritize agents according to the evolving task state. Difficulty-aware orchestration~\cite{DBLP:journals/corr/abs-2509-11079} further tailors multi-agent workflows to query-level complexity through learned routing.
Despite this progress, existing orchestration frameworks remain largely language-centric: they primarily route, sequence, or coordinate agents that communicate through natural language, while non-linguistic expert models are often treated as passive callable tools. This design is insufficient for scientific domains, where reasoning may depend on heterogeneous representations such as molecular structures, crystal graphs, time series, spatial fields, and tabular measurements. Our \textit{EywaOrchestra} addresses this limitation by dynamically orchestrating heterogeneous experts, including LLM agents and domain-specific foundation models. By jointly selecting agent configurations, model specializations, and communication topologies, \textit{EywaOrchestra} extends orchestrated intelligence from language-agent workflow optimization toward heterogeneous scientific model collaboration.

\subsubsection{Modeling Structured Data with LLMs.}

Structured data modeling has long been a central problem in machine learning, spanning data types such as graphs, tables, time series, knowledge graphs, and other relational or compositional objects. Since the rise of deep learning, a wide range of neural architectures have been developed to capture the inductive biases of different structured domains, including graph neural networks, temporal models, tabular learning methods, and representation learning techniques for relational data \cite{wu2020comprehensive, lim2021time, borisov2022deep, antelmi2023survey, han2022data, fu2023investigating}. More recently, the rapid development of large language models and AI agents has opened new possibilities for leveraging language-based reasoning, instruction following, and tool use to support structured data modeling.

\textbf{LLMs directly handle structured data.}
One line of work investigates whether LLMs can directly process structured data by converting non-textual structures into textual or sequence-based representations \cite{DBLP:conf/sigir/Tang00SSCY024, DBLP:journals/corr/abs-2506-11040, DBLP:journals/corr/abs-2601-23204, DBLP:journals/corr/abs-2407-12522, DBLP:conf/emnlp/JiangZDYZW23}. These approaches benefit from the general reasoning and in-context learning capabilities of LLMs, enabling them to perform tasks such as question answering, classification, forecasting, and relational reasoning over structured inputs. However, direct textualization may lose important structural information, especially when the data contains long-range dependencies, high-order relations, precise numerical values, or domain-specific semantics that are difficult to faithfully express in natural language \cite{DBLP:conf/wsdm/SuiZZH024, DBLP:journals/corr/abs-2507-13646, DBLP:journals/bmcbi/SadeghiBFLN24, DBLP:journals/corr/abs-2510-01538, DBLP:conf/emnlp/YangTXZLH25}.

\textbf{LLMs to help existing models.}
Another line of work uses LLMs as assistants to improve existing structured data models rather than replacing them \cite{DBLP:journals/tkde/JinLHJJH24, DBLP:journals/sigkdd/ChenMLJWWWYFLT23, DBLP:journals/corr/abs-2502-08942}. For example, textual metadata or descriptions produced by LLMs can be combined with graph, tabular, or temporal encoders to enrich structured representations \cite{DBLP:journals/corr/abs-2510-21131, DBLP:conf/nips/YanLLY0ZYZHSDZ023, DBLP:conf/acl/0006ZJFJBH025, DBLP:journals/corr/abs-2509-00687, DBLP:journals/corr/abs-2506-00009}. LLMs can also act as controllers that select models, design prompts or queries, invoke external tools, and interpret the outputs of specialized predictors \cite{DBLP:conf/icml/TriratJH25, DBLP:journals/tmlr/TornedeDEGMRSTT24, DBLP:conf/mm/LuoFN024}. These methods preserve the strengths of domain-specific architectures while using LLMs to provide semantic knowledge and flexible reasoning.

\textbf{Re-programming LLMs for structured data.}
A third direction adapts LLMs themselves to structured data through re-programming, prompting, or lightweight adaptation. Instead of training a structured-data model from scratch, these methods map structured inputs into the embedding or token space of a pretrained language model, allowing the LLM to be reused for tasks beyond natural language \cite{DBLP:conf/iclr/0005WMCZSCLLPW24, DBLP:journals/corr/abs-2402-05862, DBLP:journals/corr/abs-2502-13449, DBLP:journals/corr/abs-2409-03444}. 
% Representative strategies include prefix or prompt tuning, modality projection, tokenization of numerical or relational objects, and alignment between structured representations and language embeddings. 
This paradigm leverages the broad prior knowledge and scaling properties of LLMs, but it also faces challenges in preserving structural fidelity, handling large-scale inputs, and ensuring that the language model's inductive biases align with the target scientific domain.

Despite these advances, most existing approaches still center on adapting structured data to the language-model interface. In contrast, many scientific problems are naturally handled by specialized foundation models that operate directly over structured representations. Our work therefore takes a complementary perspective: rather than forcing all structured data into language, \textit{Eywa} introduces domain-specific foundation models into agentic systems, allowing LLM agents to coordinate, invoke, and reason with specialized models that retain their native structured-data inductive biases.

\subsection{Future Directions}

\textbf{Scaling heterogeneous scientific model ecosystems.}
A natural future direction is to scale heterogeneous agentic systems beyond a small set of manually integrated domain-specific foundation models. Enabling heterogeneous components to collaborate within a unified agentic system remains challenging from the engineering perspective, because different models may operate over incompatible input formats and provide outputs with varying degrees of interpretability. As the number of available experts grows, the system must also determine which models are relevant to a task, whether their predictions are reliable, and how conflicting evidence from multiple experts should be reconciled.
Building such an ecosystem requires more than simply adding more APIs. It calls for standardized model interfaces, rich metadata descriptions, and capability profiling. Addressing these challenges would move heterogeneous agentic systems from small-scale proof-of-concept integrations toward scalable scientific model ecosystems.

\textbf{Learning better orchestration policies.}
In this work, orchestration is primarily driven by task-level reasoning over available agents, foundation models, and communication topologies. A promising direction is to learn orchestration policies from interaction data, execution traces, and task outcomes. Future systems may train orchestrators to predict which expert models to invoke, when to invoke them, and which communication topology is most suitable for a given task.

\textbf{More advanced integration between LLMs and scientific foundation models.}
The stronger "Tsaheylu" is, the better. Future work could explore tighter coupling mechanisms beyond model context protocol, including shared representation spaces, differentiable interfaces, bidirectional alignment between language and scientific embeddings, and memory/skill mechanisms that preserve domain-specific evidence across reasoning steps.

\textbf{Extending EywaBench for heterogeneous scientific reasoning.}
Another important direction is to further extend \textit{EywaBench} into a broader benchmark suite for heterogeneous scientific reasoning. While \textit{EywaBench} provides an initial testbed, scientific reasoning spans a much wider range of domains, data modalities, task formats, and expert models. Future extensions could incorporate more datasets as well as additional data types.. Beyond expanding coverage, \textit{EywaBench} can also be extended to evaluate richer capabilities such as model selection, cross-domain evidence integration, communication efficiency, and adaptive topology selection. Such extensions would make \textit{EywaBench} a more comprehensive testbed for understanding when heterogeneous orchestration is genuinely beneficial over language-only agents, single-domain models, or static multi-agent workflows.

\subsection{Limitation}

\textbf{Dependence on foundation model capabilities.}
\textit{Eywa} builds on existing language models and domain-specific foundation models, and therefore its performance can be influenced by the reasoning ability of the underlying LLMs, the predictive quality of scientific experts, and the reliability of their interfaces. As these base models continue to improve, we expect \textit{Eywa} to benefit from stronger general-purpose and domain-specific models.

\textbf{Coverage of domains and expert models.}
Although \textit{EywaBench} covers a diverse set of scientific domains and tasks, it cannot exhaustively represent the full space of scientific applications. Some domains may involve specialized data formats, assumptions, or expert models that are not included in the current evaluation. Extending \textit{EywaBench} to broader datasets and scientific workflows remains an important future direction.

\textbf{Computational cost.}
Compared with single-agent baselines, heterogeneous collaboration may introduce additional computation, latency, and communication overhead. This cost can be mitigated through more efficient orchestration, selective expert invocation, and adaptive stopping strategies, which we leave for future work.

\section{EywaBench Details}
\label{sec:dataset_details}

\subsection{Source Datasets}

\begin{figure}[t]
  \centering
  \includegraphics[width=0.8\linewidth]{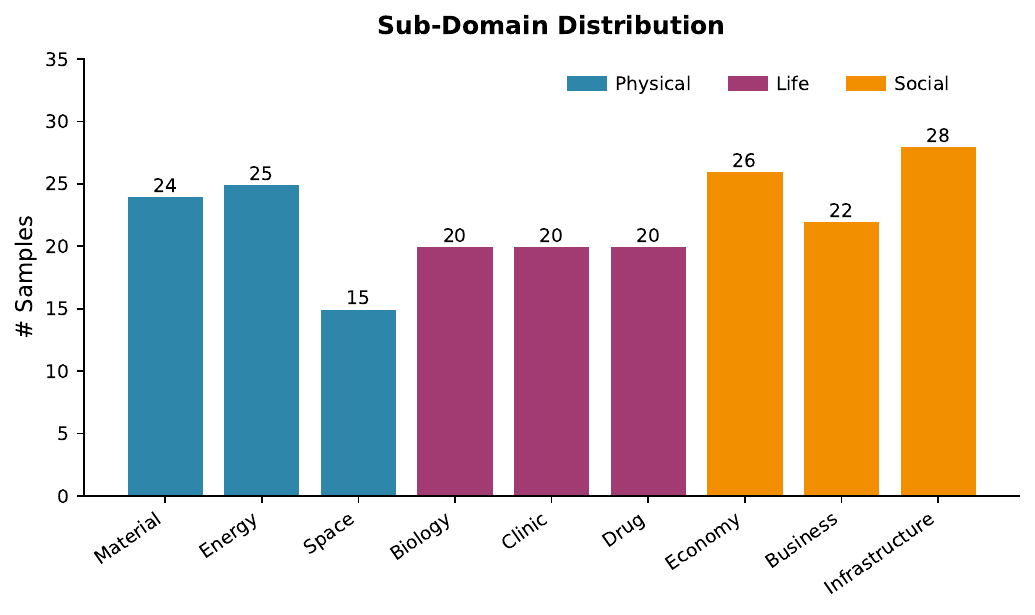}
  \caption{Composition of \textit{Eywabench} along sub-domains. distributions are all near-uniform, avoiding the domain-collapse failure modes.}
  \label{fig:eywa_overview}
\end{figure}

\textit{EywaBench} currently consists of 4 datasets: DeepPrinciple \cite{song2025evaluating}, MMLU-Pro \cite{wang2024mmlu}, fev-bench \cite{DBLP:journals/corr/abs-2509-26468}, and TabArena \cite{DBLP:journals/corr/abs-2506-16791}.

\textbf{DeepPrinciple: scientific QA dataset.} 
A scenario-grounded benchmark for authentic scientific discovery across four core disciplines (biology, chemistry, materials science, and physics). It advances beyond decontextualized QA by introducing a large-scale, two-phase evaluation that requires agents to perform iterative reasoning, hypothesis generation, and experimental design in highly diverse scientific contexts.

\textbf{MMLU-Pro: multi-task language understanding benchmark.}
A substantially enhanced reasoning benchmark comprising over 12,000 carefully curated questions across 14 diverse domains. We select scientific domains and convert the original multiple-choice format into open-ended question answering, requiring models to produce direct answers.

\textbf{fev-bench: realistic large-scale time series dataset.}
A comprehensive time series forecasting benchmark containing ~100 rigorous time series. Many of the time series incorporate complex, realistic covariates, which further highlight their high diversity.

\textbf{TabArena: large-scale tabular dataset.}
A curated collection of 51 tabular classification and regression datasets covering a wide array of real-world use cases. Supported by a massive evaluation scale of approximately 25 million trained model instances, it offers immense diversity in feature spaces and data distributions for tabular machine learning.

\begin{figure*}[t]
  \centering
  \begin{subfigure}[t]{0.49\linewidth}
    \centering
    \includegraphics[width=\linewidth]{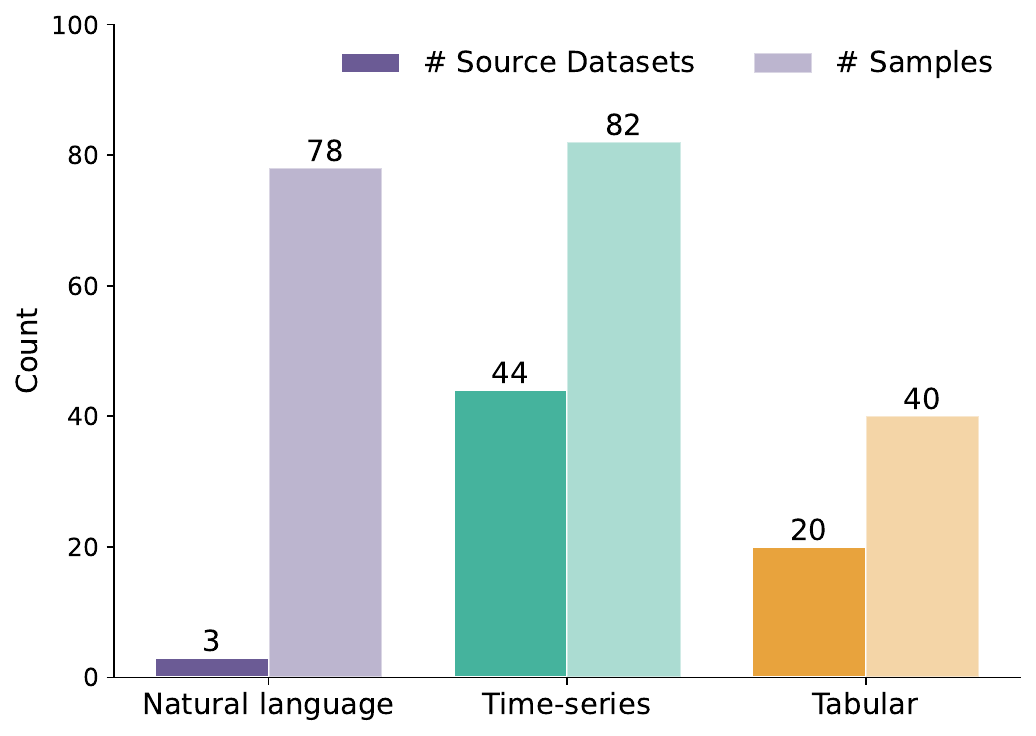}
    \caption{Source-dataset and sample counts by modality.}
    \label{fig:source_modality_counts}
  \end{subfigure}\hfill
  \begin{subfigure}[t]{0.49\linewidth}
    \centering
    \includegraphics[width=\linewidth]{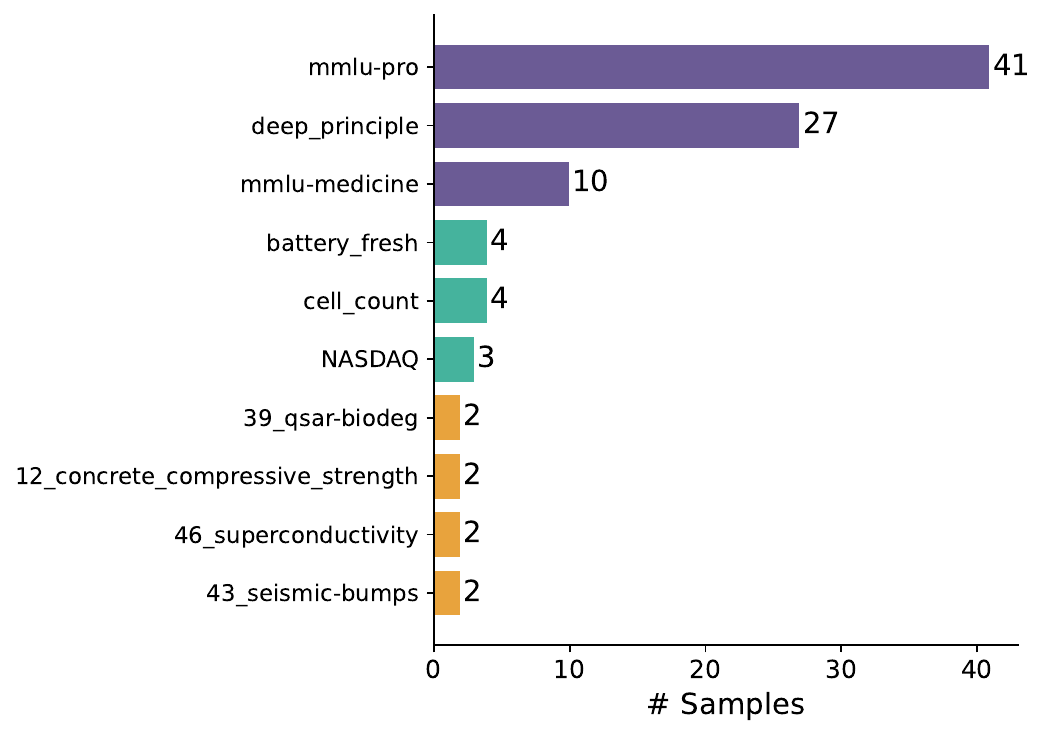}
    \caption{Top-10 source datasets by sample count.}
    \label{fig:top_source_datasets}
  \end{subfigure}
  \caption{Source diversity of \textit{EywaBench-V1}. 
  (a) The benchmark covers distinct source datasets and task samples across natural-language, time-series, and tabular modalities, showing broad modality-level coverage. 
  (b) The top-10 source datasets follow a long-tailed distribution, with no single source dominating the benchmark. 
  This design reduces the risk that method rankings are driven by overfitting to a single source distribution.}
  \label{fig:eywa_sources}
\end{figure*}

\textbf{Scalability and Sampling.}
\textit{EywaBench} is designed to be scalable. Its source tasks are drawn from large and diverse scientific benchmarks: DeepPrinciple contains 1,125 questions spanning chemistry, materials science, biology, and physics; MMLU-Pro contains 6,978 questions across physics, chemistry, engineering, economics, health, business, biology, and computer science; FEV-Bench includes 96 time-series datasets, each with up to 30,000 covariate series; and the tabular benchmark contains 51 tables, each with up to 150,000 rows. Moreover, for time-series and tabular tasks, additional task instances can be constructed by sampling different sub-series, covariate groups, rows, columns, and prediction targets from the original datasets. This makes the potential scale of \textit{EywaBench} substantially larger than the number of original datasets alone.

\textbf{\textit{EywaBench-V1.}}
In this work, we sample 200 task instances from the four benchmark sources to form a controlled evaluation subset. This choice is motivated by two considerations. First, unlike standard language-only benchmarks, evaluating heterogeneous agentic systems requires manually configuring and validating \textit{EywaMAS}. Exhaustively evaluating all possible task instances would therefore be prohibitively expensive for human expert inspection. Second, a moderate-sized subset allows us to maintain balanced coverage across domains and data modalities while keeping the evaluation reproducible and comparable across different agentic frameworks.

\subsection{Data Schema}

\begin{table}[h]
    \centering
    \caption{Data schema of \textit{EywaBench}. The benchmark is stored as a dictionary-encoded Parquet file with 200 task instances and six fields.}
    \label{tab:eywabench_schema}
    \resizebox{0.9\linewidth}{!}{
    \begin{tabular}{lll}
        \toprule
        \textbf{Column} & \textbf{Type} & \textbf{Description} \\
        \midrule
        \texttt{domain} & categorical & Scientific sub-domain, e.g., materials, energy, biology. \\
        \texttt{task} & string & Task type associated with the sample. \\
        \texttt{description} & string & Detailed task description of the sample. \\
        \texttt{output\_size} & int64 & Maximum allowed output length for the solver. \\
        \texttt{input} & string & Prompt, context, or structured input provided to the solver. \\
        \texttt{label} & string & Ground-truth output or reference answer. \\
        \bottomrule
    \end{tabular}
    }
\end{table}

\textit{EywaBench-V1} is released as a single self-contained Parquet file. The file is dictionary-encoded with Snappy compression and generated using PyArrow~20.0.0 and pandas~2.3.3. Each row represents one self-contained scientific problem and follows the six-column schema summarized in Table~\ref{tab:eywabench_schema}. For each source dataset, we also provide parsing scripts that convert its original data format into the unified \textit{EywaBench} schema.

\subsection{Composition and Coverage Analysis}
\label{sec:eywabench_analysis}

All statistics in this section are computed directly from the \textit{EywaBench-V1} parquet split.

\textbf{Overview.}

\begin{figure}[t]
  \centering
  \includegraphics[width=0.6\linewidth]{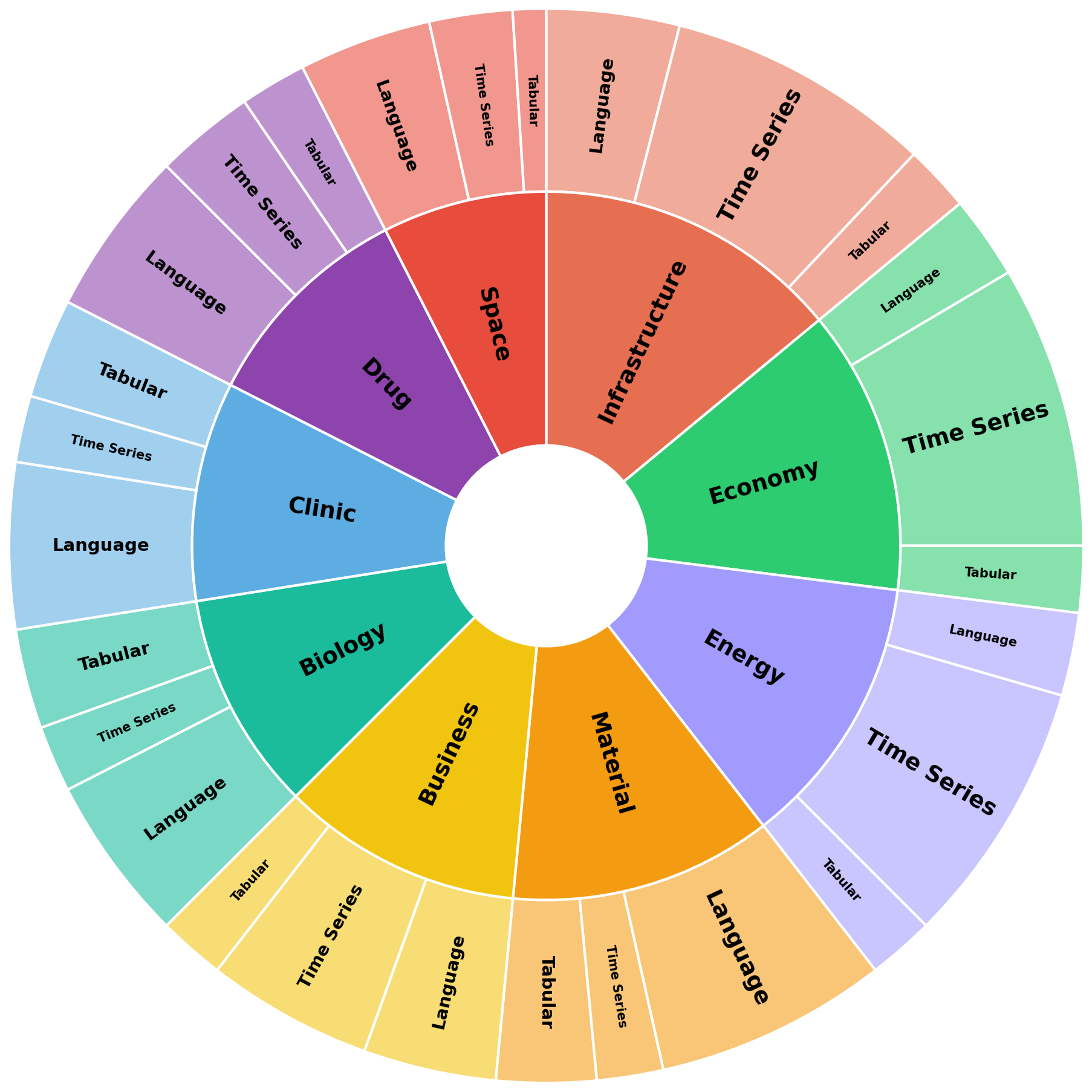}
  \caption{Hierarchical view of \textit{Eywabench}: nine sub-domains
(inner ring) $\times$ three modalities (outer ring). All $27$
cells are populated; the largest sub-domain accounts for $14.0\%$
of the benchmark.}
  \label{fig:eywa_hierarchy}
\end{figure}

The released split contains $N{=}200$ task samples. The underlying construction pipeline is fully parametric and can be scaled along two independent axes: (i)~\emph{sample volume}, where additional instances can be generated from the same source datasets by resampling temporal windows, tabular subsets, covariate groups, or prediction targets without changing the benchmark schema; and (ii)~\emph{domain and modality coverage}, where new scientific domains and data modalities can be incorporated through additional parsers. Therefore, the released split should be viewed as a \emph{representative slice} of a much larger and extensible benchmark space.

\paragraph{Balanced taxonomy along three orthogonal axes.}
\textit{Eywabench} is organised as a three-tier taxonomy (parent domain~$\to$~sub-domain~$\to$~source dataset), and is simultaneously stratified along the modality axis. The benchmark is near-uniform along all three principal axes: the three parent domains are covered at $32.0\%/30.0\%/38.0\%$, the nine sub-domains each contain $15$--$28$ instances, and the modalities are represented at $41.0\%/39.0\%/20.0\%$. We provide a visualization in Figure~\ref{fig:eywa_overview}. Quantitatively, the normalised Shannon entropy (where $H_{n}{=}1$ denotes a perfectly uniform distribution) is $H_{n}{=}0.995$ at the parent level, $H_{n}{=}0.993$ at the sub-domain level, and $H_{n}{=}0.960$ at the modality level, indicating that no single domain or modality dominates the benchmark.
every one of the nine sub-domains carries a non-trivial mix of all three modalities, yielding $100\%$ cross-modal coverage of the taxonomy. This ensures that conclusions about modality-specific agent behaviour drawn from \textit{Eywabench} generalise across scientific fields rather than conflating modality effects with domain effects. We provide a modality visualization in Figure \ref{fig:eywa_hierarchy}.

% ----------------- Detailed sub-domain table ------------------------
\begin{table}[h]
\centering
\caption{Detailed composition of \textit{EywaBench-V1} by parent domain, sub-domain, modality, and source dataset. 
For each sub-domain, we report the total number of samples, the number of natural-language, time-series, and tabular samples, and the number of distinct source datasets. 
The total number of unique source datasets is $67$, which is smaller than the sum of per-sub-domain source counts because each source dataset may contribute multiple samples.}
\label{tab:eywa_subdomain}
\resizebox{0.85\linewidth}{!}{%
\begin{tabular}{llccccc}
\toprule
\textbf{Parent} & \textbf{Sub-domain} & \textbf{\# samples} & \textbf{Natural Language} & \textbf{Time Series} & \textbf{Tabular} & \textbf{\# Source Datasets} \\
\midrule
\multirow{3}{*}{Physical}
 & Material       & $24$ & $14$ & $4$  & $6$ & $5$  \\
 & Energy         & $25$ & $5$  & $16$ & $4$ & $11$ \\
 & Space          & $15$ & $8$  & $5$  & $2$ & $7$  \\
\midrule
\multirow{3}{*}{Life}
 & Biology        & $20$ & $10$ & $4$  & $6$ & $5$  \\
 & Clinic         & $20$ & $10$ & $4$  & $6$ & $6$  \\
 & Drug           & $20$ & $10$ & $6$  & $4$ & $9$  \\
\midrule
\multirow{3}{*}{Social}
 & Economy        & $26$ & $5$  & $17$ & $4$ & $11$ \\
 & Business       & $22$ & $8$  & $10$ & $4$ & $8$  \\
 & Infrastructure & $28$ & $8$  & $16$ & $4$ & $11$ \\
\midrule
\multicolumn{2}{l}{\textbf{Total}} & $\mathbf{200}$ & $\mathbf{78}$ & $\mathbf{82}$ & $\mathbf{40}$ & $\mathbf{67}$ \\
\bottomrule
\end{tabular}}
\end{table}

\paragraph{Diversity of source datasets.}
At the leaf level of the taxonomy, \textit{EywaBench-V1} draws samples from $67$ distinct source datasets. These sources cover a broad range of scientific data, including ETT, ERCOT, NASDAQ, Jena Weather, LOOP-SEATTLE, Concrete Compressive Strength, and Superconductivity, among others. Overall, the benchmark includes $21$ physical-science sources, $19$ life-science sources, and $28$ social-science sources. The source distribution has a normalized Shannon entropy of $H_n{=}0.846$, indicating high diversity with a long-tailed structure, as visualized in Figure~\ref{fig:eywa_sources}. Importantly, even the largest individual source, MMLU-Pro, contributes only $41$ samples, accounting for $20.5\%$ of the benchmark. This reduces the risk that method rankings are dominated by any single source distribution.

\subsection{Metrics}
\label{sec:metrics}

Each task instance $i$ produces a per-instance \emph{utility} $u_i \in [0,1]$ ($u_i{=}1$ for a perfect prediction). Because \textit{EywaBench} mixes three output modalities, $u_i$ is defined modality-specifically while always being bounded in $[0,1]$ so that scores are directly comparable across modalities. Slice-level numbers reported in the paper are the unweighted mean $\bar{u}(\mathcal{D}) = \tfrac{1}{|\mathcal{D}|}\sum_{i\in\mathcal{D}} u_i$ (with sample standard deviation as dispersion); the same aggregation is applied to runtime and the four token-cost components.

\textbf{Natural language.}
For DeepPrinciple and the open-ended MMLU-Pro variant, the raw output is first reduced to a final answer $\hat{y}_i$ and normalised by an operator $\pi$ that trims whitespace/quotes, collapses inner whitespace, and maps a few Unicode variants to ASCII. We then apply a three-stage soft-score cascade:
\emph{(i) Exact match.} If $\pi(\hat{y}_i)=\pi(y_i)$, set $u_i{=}1$.
\emph{(ii) Numeric relative error.} If both sides parse as a single float $\hat{p},g$, set
\begin{equation}
u_i = \exp(-e_{\mathrm{rel}}), \qquad e_{\mathrm{rel}} = \frac{|\hat{p}-g|}{\max(|g|,\,10^{-12})}.
\label{eq:nl_numeric}
\end{equation}
\emph{(iii) Lexical fallback.} Otherwise tokenise both strings with a regex that preserves \LaTeX{} commands, words, numbers, and individual symbols; let $o$ be the multiset overlap and define token precision/recall/F1 as $P_\mathrm{tok}{=}o/|T_\mathrm{pred}|$, $R_\mathrm{tok}{=}o/|T_\mathrm{gold}|$, $F_1^{\mathrm{tok}}{=}2P_\mathrm{tok}R_\mathrm{tok}/(P_\mathrm{tok}{+}R_\mathrm{tok})$ (with $F_1^{\mathrm{tok}}{=}1$ if both empty and $0$ if exactly one is empty). Let $S_\mathrm{char}$ be the \texttt{difflib.SequenceMatcher} ratio on the normalised strings. Then
\begin{equation}
u_i = \min\!\bigl(\tau,\; \alpha\,F_1^{\mathrm{tok}} + \beta\,S_\mathrm{char}\bigr), \qquad \alpha{=}0.6,\ \beta{=}0.4,\ \tau{=}0.8.
\label{eq:nl_textual}
\end{equation}
The cap $\tau$ keeps lexical near-misses strictly below the scores reserved for stages (i)--(ii).

\textbf{Time series.}
The gold and predicted continuations $\mathbf{y},\hat{\mathbf{y}}\in\mathbb{R}^{H}$ are parsed from \texttt{(timestamp,value)} CSVs. With a denominator floor $\varepsilon{=}10^{-2}$ and $\mathcal{I}{=}\{t:|y_t|{>}\varepsilon\}$,
\begin{equation}
\mathrm{sMAPE} = \frac{1}{H}\sum_{t=1}^{H} \frac{2|y_t-\hat{y}_t|}{\max(|y_t|+|\hat{y}_t|,\varepsilon)} \in [0,2], \qquad
\mathrm{MAAPE} = \frac{1}{|\mathcal{I}|}\sum_{t\in\mathcal{I}} \arctan\!\frac{|y_t-\hat{y}_t|}{\max(|y_t|,\varepsilon)} \in [0,\pi/2],
\label{eq:smape_maape}
\end{equation}
($\mathrm{MAAPE}{=}0$ when $\mathcal{I}{=}\emptyset$). The two errors are normalised to $[0,1]$ and combined into
\begin{equation}
u_i = 1 - \frac{1}{2}\!\left(\frac{\mathrm{sMAPE}}{2} + \frac{\mathrm{MAAPE}}{\pi/2}\right) \in [0,1].
\label{eq:ts_utility}
\end{equation}
Combining the two mitigates the pathologies of either metric used alone: $\mathrm{sMAPE}$ saturates symmetrically but is brittle near zero, whereas $\mathrm{MAAPE}$ is well-behaved near zero but less scale-sensitive.

\textbf{Tabular.}
The label and prediction strings are parsed into 1-D arrays via \texttt{DataFrame.values.flatten()} (or \texttt{eval} when the string starts with \texttt{[}). For classification tasks the utility is top-1 accuracy,
\begin{equation}
u_i = \frac{1}{N}\sum_{n=1}^{N} \mathbf{1}\!\bigl[\hat{y}_n = y_n\bigr],
\label{eq:tab_cls}
\end{equation}
implemented via \texttt{sklearn.metrics.accuracy\_score}. For regression tasks the predicted and gold target columns are scored with the same sMAPE+MAAPE combination as time series (Eqs.~\eqref{eq:smape_maape}--\eqref{eq:ts_utility}), treating the $N$ row-wise predictions as an $H$-step forecast. Sharing this rule across modalities keeps numeric-target errors on a common scale, so the cross-modal averages reported in the main paper are directly comparable.

\section{Experiment Details}
\label{app:experiment_details}

\subsection{More Detailed Ablations}
\label{ap:more_ablation}

\begin{table}[!t]
    \centering
    \renewcommand{\arraystretch}{1.16}
    \caption{Overall performance comparison across scientific domains on \textit{EywaBench}. We compare all methods on three dimensions, including utility $(\uparrow)$, inference time $(\downarrow)$, and token consumption $(\downarrow)$. Best results are highlighted in bold and second-best results are underlined. Our proposed methods, EywaAgent, EywaMAS, and EywaOrchestra, achieve strong overall performance while maintaining competitive efficiency.}
    \vspace{-8pt}
    \label{tab:ablation_backend_llm}
    \small
    \resizebox{\textwidth}{!}{
    \begin{tabular}{l|c|ccc|ccc|ccc|c}
        \toprule
        \multirow{2}{*}{\textbf{Method}} &
        \multirow{2}{*}{\textbf{Metrics}} &
        \multicolumn{3}{c}{\textbf{Physical Science}} &
        \multicolumn{3}{c}{\textbf{Life Science}} &
        \multicolumn{3}{c}{\textbf{Social Science}}
        \\
        \cmidrule(lr){3-5}\cmidrule(lr){6-8}\cmidrule(lr){9-11}\cmidrule(lr){12-12}
        & & Material & Energy & Space & Biology & Clinic & Drug & Economy & Business & Infrastructure & Overall\\
        \midrule[-0.4ex]\midrule\addlinespace[-0.000ex]
        \multicolumn{12}{c}{
            \rule{0pt}{1.1em}
            \textbf{\textit{Single-Agent Setting}}
            \rule[-0.3em]{0pt}{1.1em}
        } \\
        [-0.4ex]\midrule\addlinespace[-0.000ex]

        & Utility ($\uparrow$)   & 0.5491 & 0.7980 & 0.5848 & 0.3429 & 0.4010 & 0.4592 & 0.7416 & 0.5560 & 0.5762 & 0.5680\\
        \textbf{EywaAgent (gpt-4.1-nano)}
        & Time ($\downarrow$)  & 9.47 & 20.88 & 10.85 & 31.11 & 23.37 & 31.51 & 23.66 & 13.88 & 12.98 & 19.61\\
        & Tokens ($\downarrow$) & 1654 & 1149 & 1047 & 1337 & 606 & 839 & 1181 & 1036 & 1237 & 1139\\
        [-0.4ex]\midrule\addlinespace[-0.000ex]

        & Utility ($\uparrow$)   & 0.5871 & 0.8390 & 0.6123 & 0.3718 & 0.5085 & 0.6199 & \textbf{0.8048} & \underline{0.7371} & 0.7060 & 0.6558\\
        \textbf{EywaAgent (gpt-5-nano)}
        & Time ($\downarrow$)  & 34.88 & 24.42 & 23.12 & 30.84 & 20.32 & 15.84 & 19.71 & 20.98 & 15.99 & 22.78\\
        & Tokens ($\downarrow$) & 5040 & 3167 & 3329 & 4858 & 2333 & 2210 & 2791 & 2444 & 2248 & 3137\\
        [-0.4ex]\midrule\addlinespace[-0.000ex]

        & Utility ($\uparrow$)   & 0.6272 & 0.8615 & 0.6286 & 0.3670 & 0.4990 & 0.6444 & 0.7790 & 0.7193 & 0.7320 & 0.6640\\
        \textbf{EywaAgent (gpt-5-mini)}
        & Time ($\downarrow$)  & 21.81 & 20.70 & 29.52 & 45.34 & 16.86 & 29.91 & 18.97 & 21.08 & 16.05 & 23.63\\
        & Tokens ($\downarrow$) & 3666 & 2316 & 2541 & 3589 & 1734 & 2527 & 2033 & 1669 & 2092 & 2444\\

        \addlinespace[-0.4ex] \midrule \addlinespace[-0.000ex]
        \multicolumn{12}{c}{
            \rule{0pt}{1.1em}
            \textbf{\textit{Multi-Agent Setting}}
            \rule[-0.3em]{0pt}{1.1em}
        } \\
        [-0.4ex]\midrule\addlinespace[-0.000ex]

        & Utility ($\uparrow$)   & 0.6161 & 0.8305 & 0.6602 & 0.3545 & 0.4468 & 0.5062 & 0.7729 & 0.6404 & 0.6759 & 0.6236\\
        \textbf{EywaMAS (gpt-4.1-nano)}
        & Time ($\downarrow$)  & 31.16 & 39.72 & 27.96 & 42.28 & 36.96 & 31.72 & 38.16 & 63.88 & 39.24 & 39.40\\
        & Tokens ($\downarrow$) & 6592 & 4782 & 3832 & 5284 & 2392 & 3049 & 4328 & 3764 & 4969 & 4421\\
        [-0.4ex]\midrule\addlinespace[-0.000ex]

        & Utility ($\uparrow$)   & \textbf{0.6381} & \textbf{0.8742} & \underline{0.6899} & \textbf{0.3798} & \underline{0.5086} & \underline{0.6248} & \underline{0.7959} & 0.7284 & \textbf{0.7406} & \textbf{0.6761}\\
        \textbf{EywaMAS (gpt-5-nano)}
        & Time ($\downarrow$)  & 77.25 & 75.96 & 72.51 & 111.92 & 59.97 & 59.23 & 68.40 & 58.11 & 46.49 & 72.11\\
        & Tokens ($\downarrow$) & 14529 & 11709 & 11787 & 16502 & 9407 & 8078 & 11044 & 9470 & 8912 & 11214\\
        
        \addlinespace[-0.4ex] \midrule \addlinespace[-0.000ex]
        \multicolumn{12}{c}{
            \rule{0pt}{1.1em}
            \textbf{\textit{Dynamic Orchestration}}
            \rule[-0.3em]{0pt}{1.1em}
        } \\
        [-0.4ex]\midrule\addlinespace[-0.000ex]

        & Utility ($\uparrow$)   & 0.6146 & 0.8259 & 0.5498 & 0.3355 & 0.4567 & 0.6153 & 0.7442 & 0.6725 & 0.6519 & 0.6210\\
        \textbf{EywaOrchestra (gpt-4.1-nano)}
        & Time ($\downarrow$)  & 35.14 & 23.13 & 24.32 & 32.51 & 20.30 & 17.83 & 21.77 & 15.93 & 18.00 & 23.05\\
        & Tokens ($\downarrow$) & 8272 & 5808 & 5846 & 8517 & 4847 & 4912 & 6088 & 4627 & 5317 & 6017\\
        [-0.4ex]\midrule\addlinespace[-0.000ex]

        & Utility ($\uparrow$)   & \underline{0.6249} & \underline{0.8711} & \textbf{0.7187} & 0.3682 & \textbf{0.5159} & \textbf{0.6319} & 0.7830 & \textbf{0.7388} & \underline{0.7298} & \underline{0.6746}\\
        \textbf{EywaOrchestra (gpt-5-nano)}
        & Time ($\downarrow$)  & 61.78 & 39.92 & 75.47 & 67.88 & 45.38 & 45.94 & 49.13 & 34.18 & 28.80 & 48.16\\
        & Tokens ($\downarrow$) & 11535 & 7723 & 10810 & 11315 & 7050 & 6495 & 7117 & 7264 & 6892 & 8335\\
        [-0.4ex]\bottomrule\addlinespace[-0.000ex]

    \end{tabular}
    }
    \vspace{-4pt}
\end{table}

In this section, we provide a more comprehensive backbone ablation that extends the single-agent study in Table~\ref{tab:model_ablation_compact} to also cover \textit{EywaMAS} and \textit{EywaOrchestra}, and reports per-sub-domain results across all nine domains. The full results are shown in Table~\ref{tab:ablation_backend_llm}. We highlight three observations.
 
\textbf{\textit{Eywa} is and consistently compatible and effective across LLM backbones.} Across all three backbones (gpt-4.1-nano, gpt-5-nano, and gpt-5-mini) and all three system settings (single-agent, multi-agent, and dynamic orchestration), \textit{Eywa} delivers strong utility on every sub-domain. For instance, with the relatively weak gpt-4.1-nano backbone, \textit{EywaAgent}, \textit{EywaMAS}, and \textit{EywaOrchestra} reach overall utilities of 0.5680, 0.6236, and 0.6210, respectively, all of which already surpass the single-agent gpt-5-nano baseline reported in Table~\ref{tab:main_comparison_eywabench}. A similar pattern holds when we move to gpt-5-nano, where the relative ranking among single-agent, multi-agent, and dynamic-orchestration variants is preserved. This stability suggests that the gains of \textit{Eywa} come primarily from cross-modality heterogeneity and structured FM--LLM collaboration, rather than from a particular language-model checkpoint.
 
\textbf{\textit{Eywa} benefits from stronger LLM backbones.} As we scale the language model from gpt-4.1-nano to gpt-5-nano, all three settings improve substantially in overall utility: \textit{EywaAgent} improves from 0.5680 to 0.6558 ($+15.5\%$), \textit{EywaMAS} from 0.6236 to 0.6761 ($+8.4\%$), and \textit{EywaOrchestra} from 0.6210 to 0.6746 ($+8.6\%$). The per-domain breakdown shows consistent improvements across physical, life, and social science. This indicates that, although \textit{Eywa} already provides strong baseline performance through its FM--LLM coupling, the gains compound as the underlying language model becomes more capable: stronger LLMs produce more reliable planning, better routing, and more faithful integration of foundation-model outputs into the final answer.
 
\textbf{Diminishing returns from gpt-5-nano to gpt-5-mini suggest a domain-specific bottleneck.} While moving from gpt-4.1-nano to gpt-5-nano yields large utility gains (e.g., $+15.5\%$ for \textit{EywaAgent}), further upgrading the backbone from gpt-5-nano to gpt-5-mini brings only marginal improvements ($0.6558 \rightarrow 0.6640$, i.e., $+1.25\%$ for \textit{EywaAgent}), and on several sub-domains the larger model is in fact slightly worse (e.g., Biology, Clinic, Economy, and Business). This pattern indicates that once the LLM is sufficiently capable of planning, communicating, and routing to the specialized foundation models, additional general-purpose LLM scaling yields diminishing returns. In other words, the remaining headroom in utility is no longer dominated by the LLM's general reasoning ability, but rather by \emph{domain-specific} capability. This observation provides a strong motivation for our \textit{Eywa} to extend agentic systems with domain-specific components.

\subsection{Case Study}
\label{ap:case_study}

We provide representative case studies to qualitatively illustrate how different agentic configurations behave on \textit{EywaBench-V1}. In particular, we focus on how each system parses a text-defined task, handles structured non-linguistic inputs, invokes or fails to invoke specialized predictive capabilities, and realizes the final output under strict formatting constraints.

Case Study A compares a language-only LLM agent with \textit{EywaAgent} on the same task over structured financial signals. As shown in Figure~\ref{fig: case_study_llmagent}, the language-only agent correctly understands the surface-level task interface: it identifies the required output schema, produces the correct number of future timestamps, and returns a dataframe-style response. However, because it only reasons over serialized values, its prediction collapses to a last-value persistence baseline. This illustrates a common failure mode: the LLM can follow instructions and satisfy formatting requirements, but this does not imply that it has performed the underlying domain-specific numerical computation.

Figure~\ref{fig: case_study_eywaagent} shows the corresponding \textit{EywaAgent} behavior on the same task. Instead of treating the serialized sequence as a purely textual pattern, \textit{EywaAgent} uses the LLM to parse the task, configure the model call, and activate the Chronos foundation model through the \textit{Tsaheylu} interface. Chronos then serves as the core predictor, while the LLM verifies the returned forecast in context and formats it into the required dataframe-style response. This case highlights the intended division of labor in \textit{EywaAgent}: the language model provides task understanding and interface control, while the specialized foundation model performs the domain-specific computation.

Case Study B further illustrates the role of \textit{EywaOrchestra}. Unlike \textit{EywaAgent}, which uses a fixed augmented-agent design, \textit{EywaOrchestra} first decides which configuration should be used for the given task. As shown in Figure~\ref{fig:case_study_eywaorchestra}, the planner identifies the input as a Miami house-price prediction problem, formulates it as structured tabular regression, and selects a lightweight single-agent \textit{Eywa} configuration with \texttt{gpt-5-nano} and TabPFN. This example shows that orchestration does not necessarily mean using a more complex multi-agent topology; rather, it means selecting an appropriate configuration for the task. When a specialized tabular foundation model is sufficient, \textit{EywaOrchestra} can avoid unnecessary multi-agent discussion and directly route the task to the most suitable expert.

Overall, these case studies show three complementary points. First, language-only agents may be highly format-compliant while still relying on shallow heuristics for structured numerical tasks. Second, \textit{EywaAgent} makes specialized foundation models usable inside text-defined workflows by bridging natural-language instructions and model-specific inference interfaces. Third, \textit{EywaOrchestra} extends this idea by dynamically selecting the agent setting, foundation model, and execution plan according to the task modality and complexity. Together, these examples provide qualitative evidence for the central motivation of \textit{Eywa} that heterogeneous scientific reasoning requires both language-side task execution and domain-side specialized computation.

\begin{figure}[h]
\centering
\resizebox{\textwidth}{!}{
\begin{tcolorbox}[colback=gray!5!white, colframe=red!60, 
title=Case Study A.1: How LLM Agent Fails a Domain-Specific Problem, boxrule=0.3mm, width=1.2\textwidth, arc=2mm, auto outer arc=true]
\textbf{Core Task.} This case study presents a representative instance from \textit{EywaBench-V1}, centered on a domain-specific numerical prediction task over structured financial signals. The input consists of 50 normalized NASDAQ measurements indexed by consecutive, anonymized timestamps. Although the data are presented in a text-readable format, solving it requires recognizing numerical dynamics in a structured signal and extrapolating future values. To reduce the possibility that language models rely on memorized historical market trends, all timestamps are anonymized and re-indexed during preprocessing. The goal is to estimate the next 10 normalized NASDAQ values.

\vspace{3mm}
\textbf{Agents.}
We compare a language-only LLM agent with \textit{EywaAgent}. The LLM agent performs the task solely through language-based reasoning with gpt-5-nano backbone. In contrast, \textit{EywaAgent} connects a Chronos time-series foundation model with gpt-5-nano through the \textit{Tsaheylu} interface.

\vspace{3mm}
\hrule 
\vspace{3mm}
\textbf{Prompt:} 
\begin{quote}
You are given a normalized univariate financial time series and are asked to forecast the next 10 values. The timestamps have been anonymized and re-indexed to avoid relying on memorized historical market trends.

The input is a structured financial signal with two columns: \texttt{timestamp} and \texttt{value}. It contains 50 consecutive normalized NASDAQ observations with anonymized timestamps. Your task is to infer the values for the next 10 timestamps.
Return the prediction as a pandas-dataframe-style string with exactly two columns: \texttt{timestamp,value}. Do not include explanations, comments, or intermediate reasoning.

\textbf{Input time series:}

\texttt{timestamp,value}\\
\texttt{\ldots}\\

\textbf{Forecast horizon:} 10 timestamps.

\textbf{Required output format:} \texttt{\ldots}\\

\end{quote}

\hrule
\vspace{3mm}

\textbf{LLM Agent Response Trajectory:}

\begin{enumerate}
    \item \textbf{Task-interface parsing.}
    The prompt is interpreted as a domain-specific numerical prediction task with a strict output interface. The input is presented as a dataframe-style sequence, and the answer must be returned in the same dataframe-style format. This creates two immediate constraints: the response should contain exactly 10 future timestamps and only the columns \texttt{timestamp} and \texttt{value}, without any additional explanation.

    \item \textbf{Language-only inference.}
    The agent then inspects the serialized numerical observations. Since it has no access to a specialized forecasting model, it can only infer future values from the surface pattern of the provided sequence. To avoid introducing unsupported oscillations or trends, it chooses a conservative continuation strategy.

    \item \textbf{Persistence-style output generation.}
    Following this strategy, the agent extends the timestamp index by 10 steps and assigns each future timestamp the last observed value. The resulting response is format-compliant, but the prediction itself reduces to a last-value persistence forecast rather than a domain-specific numerical prediction.
\end{enumerate}

\begin{quote}
\small
\texttt{timestamp,value}\\
\textbf{\textit{(The agent keeps the last observed value as the prediction for all future timestamps)}}\\
\texttt{\ldots}
\end{quote}

\textbf{Evaluation on LLM Agent Output:}
The response receives a utility score of $0.605$. The agent correctly follows the required output schema and produces the correct number of future predictions. However, because the language-only agent has no specialized numerical forecasting capability, its final answer reduces to a last-value persistence baseline, limiting its predictive utility.

\end{tcolorbox}}
\caption{
Case study of a language-only LLM agent on a representative \textit{EywaBench-V1} instance. Although the agent correctly parses the task interface and produces a format-compliant response with the required forecast horizon, its prediction collapses to a last-value persistence baseline. This illustrates the limitation of language-only reasoning on structured domain-specific numerical tasks and motivates the need for \textit{EywaAgent} to invoke specialized foundation models through the \textit{Tsaheylu} interface.
}
\label{fig: case_study_llmagent}
\end{figure}

\begin{figure}[h]
\centering
\resizebox{\textwidth}{!}{
\begin{tcolorbox}[colback=gray!5!white, colframe=orange!60, 
title=Case Study A.2: How EywaAgent Solves a Domain-Specific Problem with Specialized Foundation Model, boxrule=0.3mm, width=1.2\textwidth, arc=2mm, auto outer arc=true]
\textbf{Core task and prompts are the same as previous A.1 case study.} 

\vspace{3mm}
\textbf{Agents.}
We compare a language-only LLM agent with \textit{EywaAgent}. The LLM agent performs the task solely through language-based reasoning with gpt-5-nano backbone. In contrast, \textit{EywaAgent} connects a Chronos time-series foundation model with gpt-5-nano through the \textit{Tsaheylu} interface.

\vspace{3mm}
\hrule
\vspace{3mm}

\textbf{Why Not Directly Use the Foundation Model?}
The Chronos foundation model is the main predictive component in this case, but it cannot independently execute the full text-defined task. It does not parse natural-language instructions, interpret the required dataframe-style output schema, decide which part of the prompt corresponds to the numerical input, or format the answer according to the user's textual constraints. Directly applying Chronos would require manually extracting the numerical sequence, specifying the forecast horizon, running model inference, and post-processing the prediction into the requested output format. After our extension, \textit{EywaAgent} automates this missing interface.

\vspace{3mm}
\hrule
\vspace{3mm}

\textbf{EywaAgent Response Trajectory.}
\begin{enumerate}
    \item \textbf{Task-interface parsing by the LLM.} 
    The \texttt{gpt-5-nano} backbone first parses the user prompt and identifies that the task is not a general language reasoning problem, but a domain-specific numerical prediction problem over a structured financial signal. It extracts the input schema, the forecast horizon, and the required output format.

    \item \textbf{Foundation-model activation through \textit{Tsaheylu}.}
    Instead of extrapolating directly from the serialized values, \textit{EywaAgent} activates the connected Chronos foundation model through the \textit{Tsaheylu} interface. The interface converts the dataframe-style input into the numerical sequence format expected by Chronos, configures the forecast horizon, and delegates the core prediction step to the specialized foundation model.

    \item \textbf{Specialized prediction by the foundation model:}
    Chronos performs the main predictive computation over the structured numerical signal. Unlike the language-only agent, which falls back to a persistence-style heuristic, the foundation model uses its pretrained time-series inductive bias to generate a nontrivial forecast for the next 10 steps.

    \begin{quote}
    \small
    \texttt{timestamp,value}\\
    \textit{(Chronos-generated predictions are aligned with the next 10 timestamps)}\\
    \texttt{\ldots}
    \end{quote}

    \item \textbf{Context-aware verification and output realization by the LLM.}
    The Chronos predictions are returned to \textit{EywaAgent}, where the LLM performs a context-aware check beyond simple format validation. It verifies the forecast horizon, timestamp alignment, numerical scale, and consistency with the prompt constraints, making adjustments when needed. The final prediction is then formatted as the required pandas-dataframe-style response. Thus, Chronos serves as the core predictor, while the LLM handles task interpretation, model-call configuration, contextual verification, and final answer realization.
\end{enumerate}

\textbf{Evaluation on EywaAgent Output:}
The response obtains a higher utility score of $0.701$ because it preserves the required output format while using a specialized foundation model for the core predictive computation, rather than relying on a language-level persistence heuristic.

\end{tcolorbox}}
\caption{
Case study of \textit{EywaAgent} on the same \textit{EywaBench-V1} instance as Case Study A.1. Unlike the language-only agent, \textit{EywaAgent} enables the Chronos foundation model to serve as the core predictor while using the LLM to parse the text-defined task, configure the model call through the \textit{Tsaheylu} interface, verify the returned prediction in context, and realize the final dataframe-style response. The improved utility score shows the benefit of combining specialized foundation-model computation with language-side task execution.
}
\label{fig: case_study_eywaagent}
\end{figure}

\begin{figure}[h]
\centering
\resizebox{\textwidth}{!}{
\begin{tcolorbox}[colback=gray!5!white, colframe=pink!90, 
title=Case Study B.1: EywaOrchestra Example 1, boxrule=0.3mm, width=1.2\textwidth, arc=2mm, auto outer arc=true]

\textbf{Core Task.}
This case study presents a representative \textit{EywaBench-V1} instance on Miami house price. The input is a structured real-estate table containing property attributes. The target column contains masked entries, and the system is asked to predict these missing sale prices and return them as a numpy-array-style string.

\vspace{3mm}

\textbf{Prompt:} Please refer to Apendix \ref{app:prompts} for details.

\vspace{3mm}
\hrule
\vspace{3mm}

\textbf{Planner Orchestration Stage.}
\begin{enumerate}
    \item \textbf{Task diagnosis.}
    The planner parses the text-defined task and identifies it as a Miami house-price prediction problem, formulated as a structured tabular regression task. It extracts the input modality, target column, masked query rows, output size, and required numpy-array-style response format.

    \item \textbf{Configuration selection.}
    Based on the detected tabular structure, the planner decides that a full multi-agent discussion is unnecessary. Instead, it selects a lightweight single-agent \textit{Eywa} configuration with \texttt{gpt-5-nano} as the language backbone and TabPFN as the specialized tabular foundation model:
    \[
    \{\texttt{eywa}: \texttt{true},\ 
    \texttt{setting}: \texttt{single-agent},\ 
    \texttt{model}: \texttt{gpt-5-nano},\ 
    \texttt{foundation\_model}: \texttt{TabPFN}\}.
    \]

    \item \textbf{Execution plan.}
    The planner decomposes the problem into three executable steps: extract observed training rows and masked query rows, invoke TabPFN to estimate the missing values, and format the resulting predictions as the required numpy array.
\end{enumerate}

\textbf{Execution Stage.}
\begin{enumerate}
    \item \textbf{Structured input extraction.}
    The execution agent converts the serialized table into structured tabular data, separates rows with observed target values from rows with masked targets, and identifies the features used for prediction.

    \item \textbf{Foundation-model inference.}
    The selected tabular foundation model performs the core regression step. It predicts each missing value based on the corresponding row features.

    \item \textbf{Context-aware verification.}
    The LLM checks whether the returned predictions match the requested output size, whether the values are plausible under the scale of the observed column, and whether the prediction array respects the order of the masked rows.

    \item \textbf{Final response realization.}
    After verification, \textit{EywaOrchestra} returns only the predicted labels in the required numpy-array-style format, satisfying the text-defined output constraint while relying on the tabular foundation model for the core numerical computation.
\end{enumerate}

\begin{quote}
\small
\texttt{Output: [}\textit{tabular-model prediction for row 1}\texttt{, }
\textit{tabular-model prediction for row 2}\texttt{, }
\texttt{\ldots}\texttt{]}
\end{quote}

\textbf{Evaluation on EywaOrchestra Output:}
The response receives a utility score of \(0.853\). \textit{EywaOrchestra} correctly identifies the task as Miami house-price prediction and selects a single-agent \textit{Eywa} configuration with \texttt{gpt-5-nano} and TabPFN. The resulting prediction preserves the required numpy-array-style output format while using TabPFN for the core tabular regression computation.

\end{tcolorbox}}
\caption{
Case study of \textit{EywaOrchestra} on a representative \textit{EywaBench-V1} instance. The planner first diagnoses the task modality, selects a tabular predictive expert, and determines an execution topology. The execution stage then extracts structured inputs, invokes the selected foundation model for the core regression computation, and uses the LLM for contextual verification and final numpy-array-style response realization.
}
\label{fig:case_study_eywaorchestra}
\end{figure}

\subsection{Detailed Utility and Token Consumption}

\begin{figure*}[t]
  \centering
  % Shared legend on top
  \includegraphics[width=0.98\linewidth]{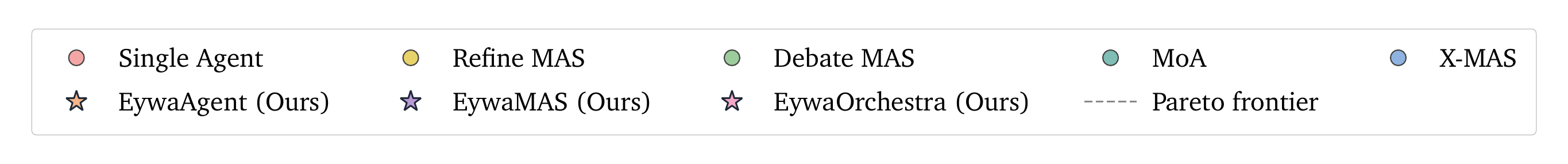}\\[4pt]

  \begin{subfigure}[t]{0.32\linewidth}
    \centering
    \includegraphics[width=\linewidth]{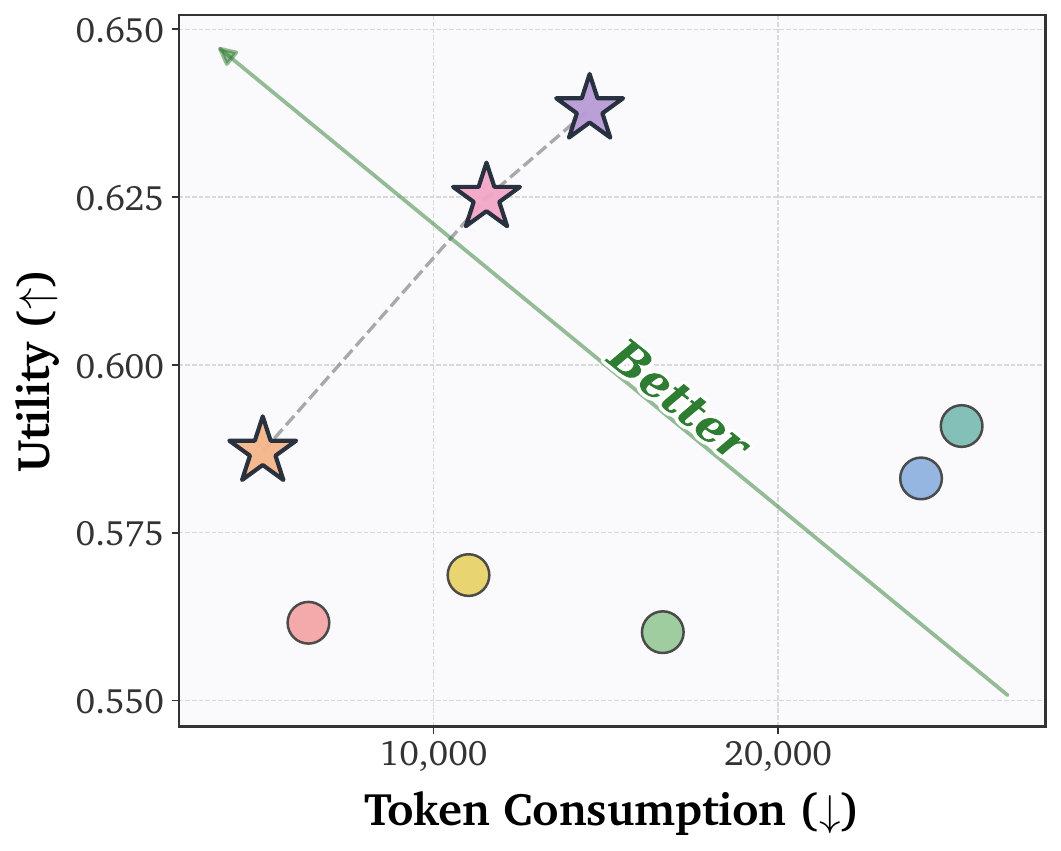}
    \caption{Material}
  \end{subfigure}\hfill
  \begin{subfigure}[t]{0.32\linewidth}
    \centering
    \includegraphics[width=\linewidth]{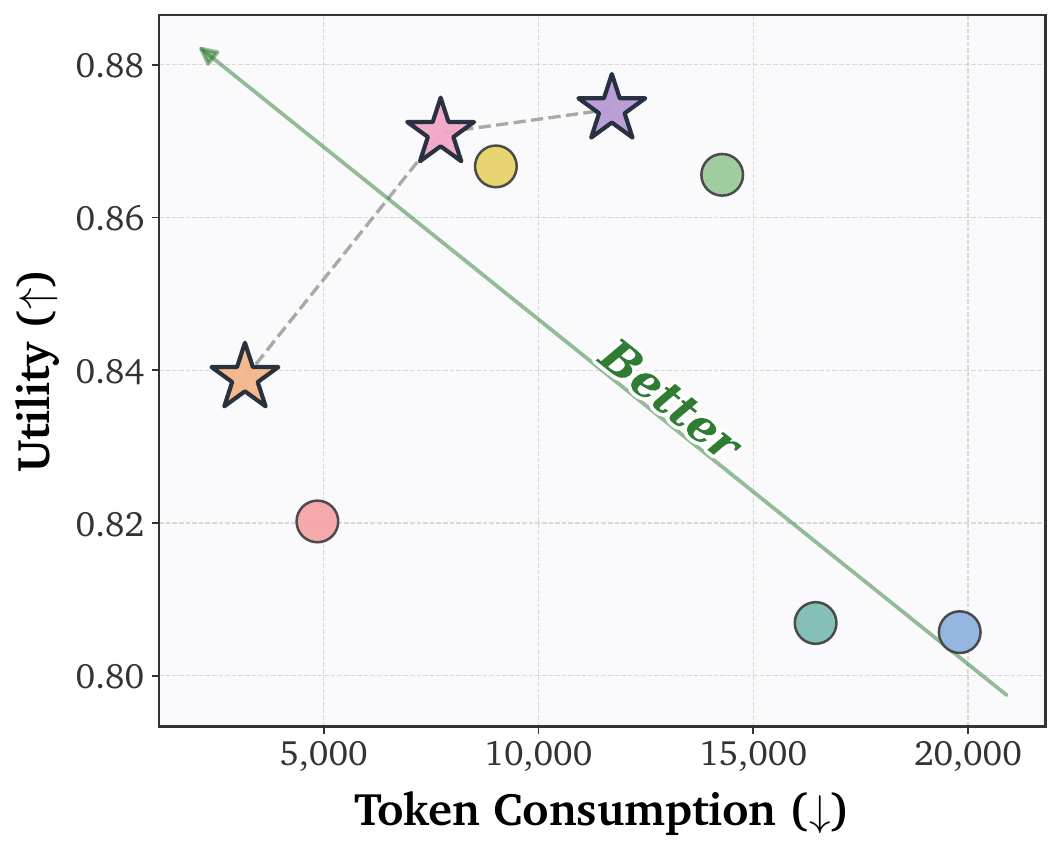}
    \caption{Energy}
  \end{subfigure}\hfill
  \begin{subfigure}[t]{0.32\linewidth}
    \centering
    \includegraphics[width=\linewidth]{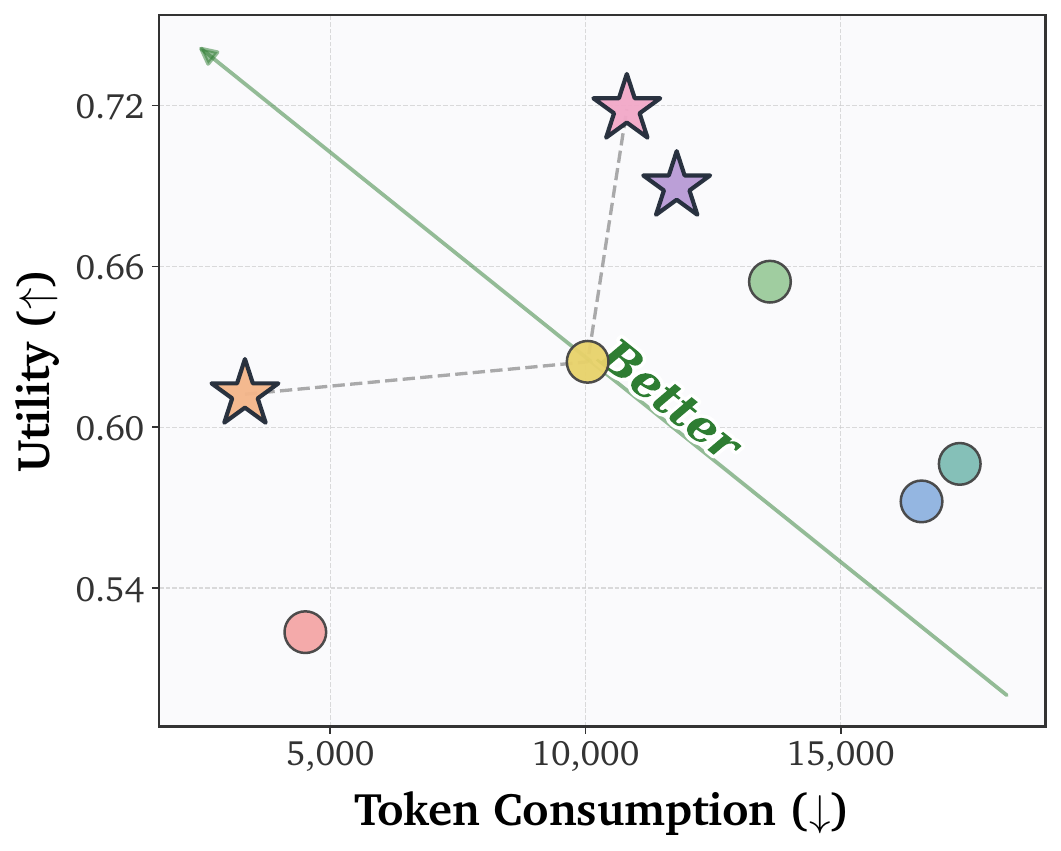}
    \caption{Space}
  \end{subfigure}\\[4pt]

  \begin{subfigure}[t]{0.32\linewidth}
    \centering
    \includegraphics[width=\linewidth]{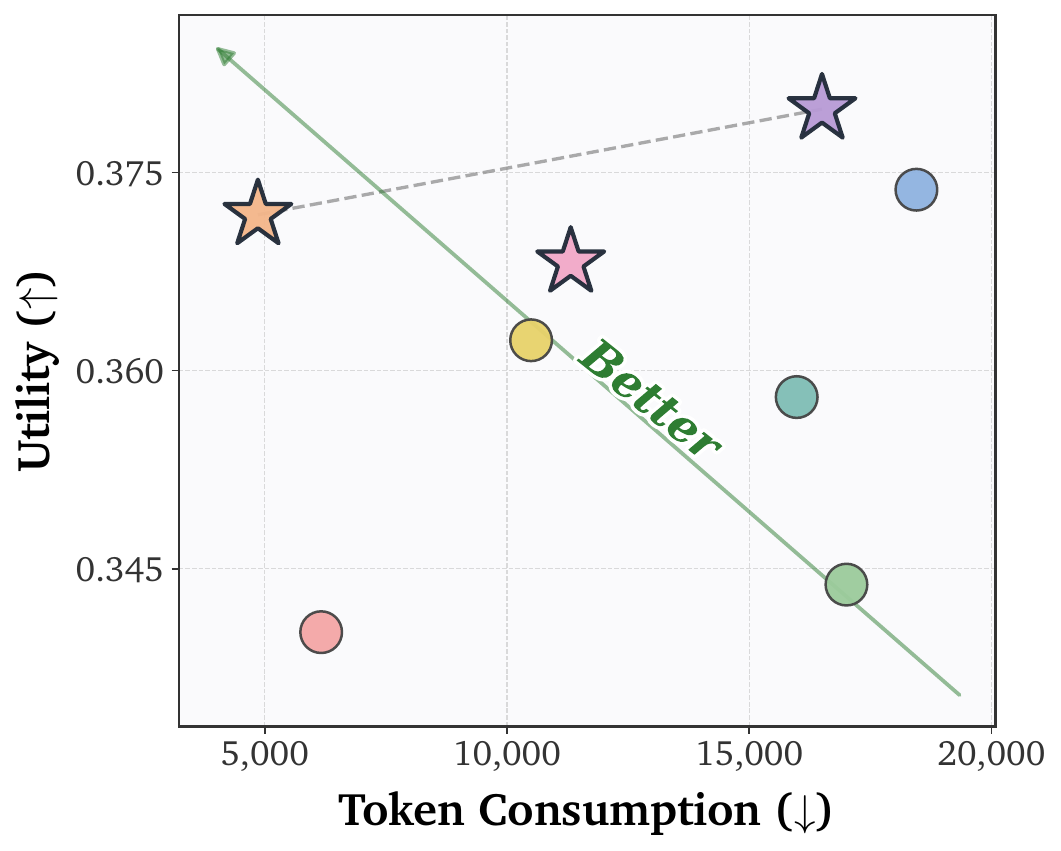}
    \caption{Biology}
  \end{subfigure}\hfill
  \begin{subfigure}[t]{0.32\linewidth}
    \centering
    \includegraphics[width=\linewidth]{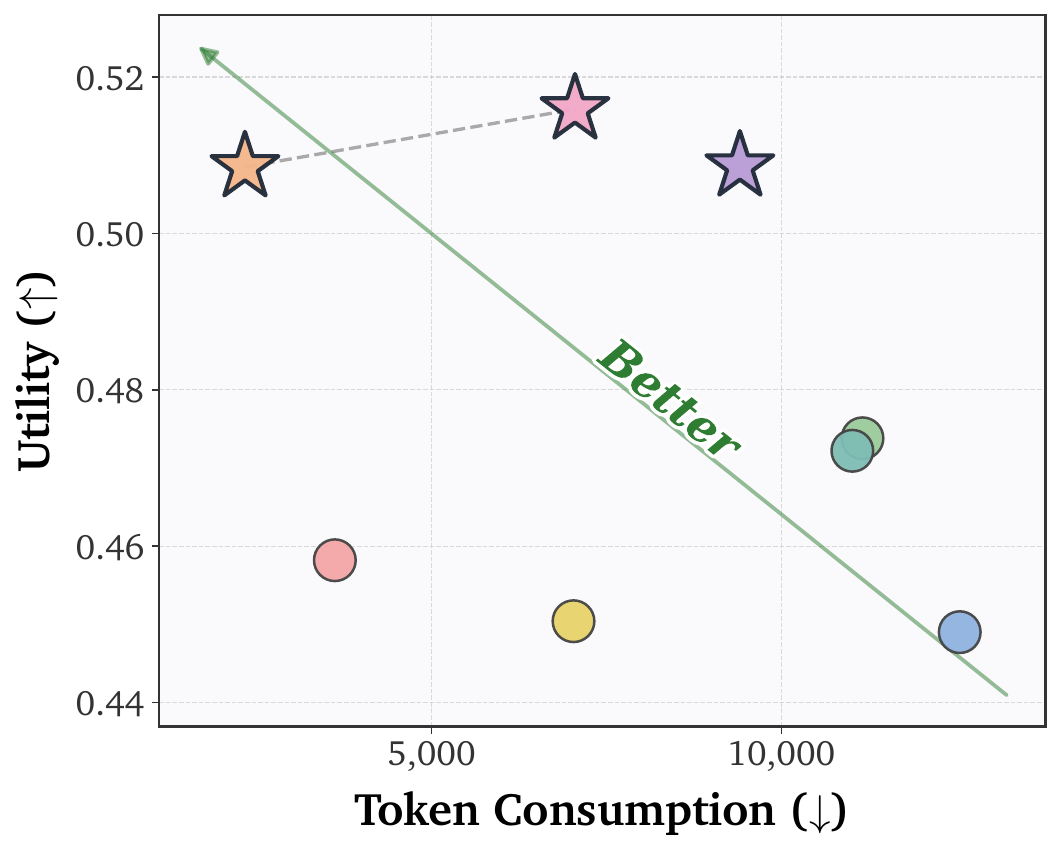}
    \caption{Clinic}
  \end{subfigure}\hfill
  \begin{subfigure}[t]{0.32\linewidth}
    \centering
    \includegraphics[width=\linewidth]{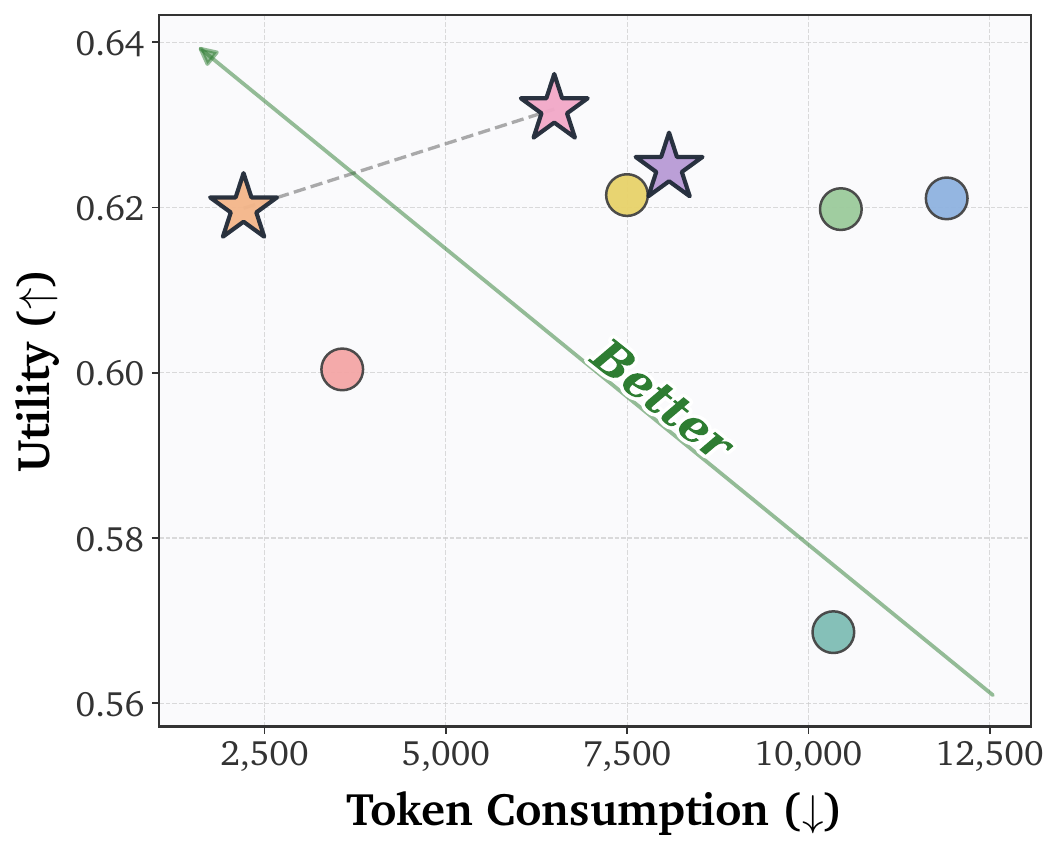}
    \caption{Drug}
  \end{subfigure}\\[4pt]

  \begin{subfigure}[t]{0.32\linewidth}
    \centering
    \includegraphics[width=\linewidth]{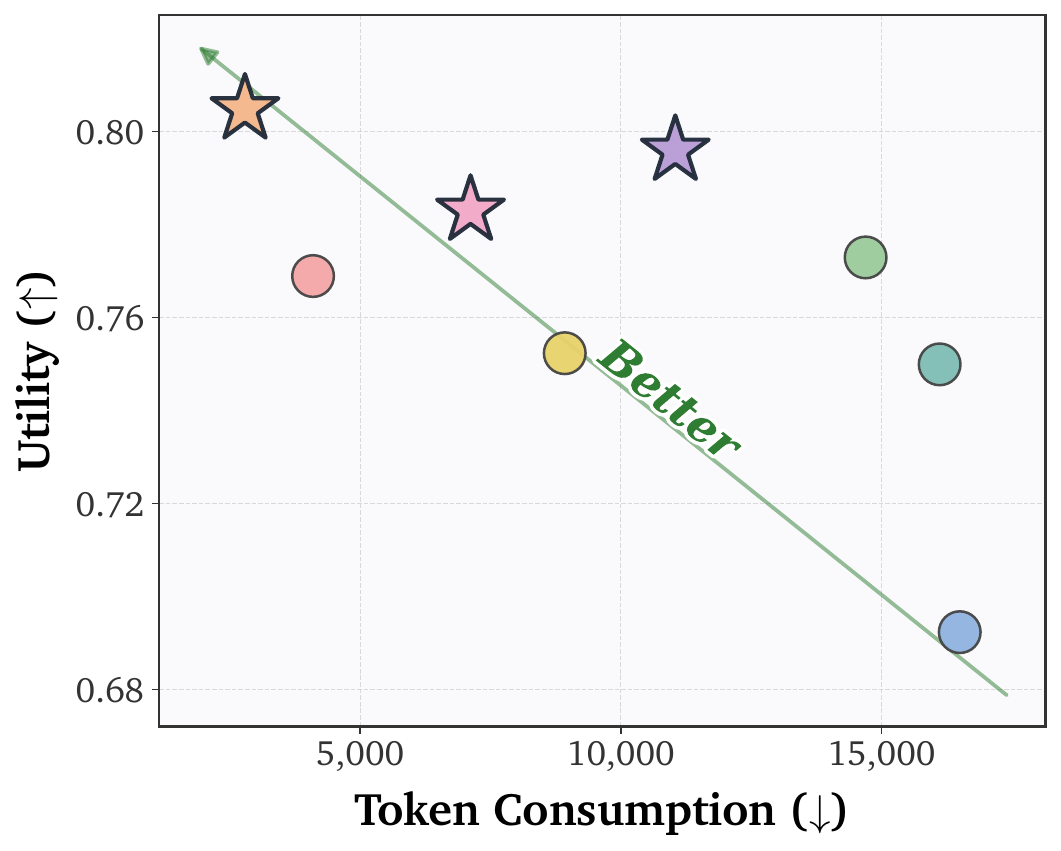}
    \caption{Economy}
  \end{subfigure}\hfill
  \begin{subfigure}[t]{0.32\linewidth}
    \centering
    \includegraphics[width=\linewidth]{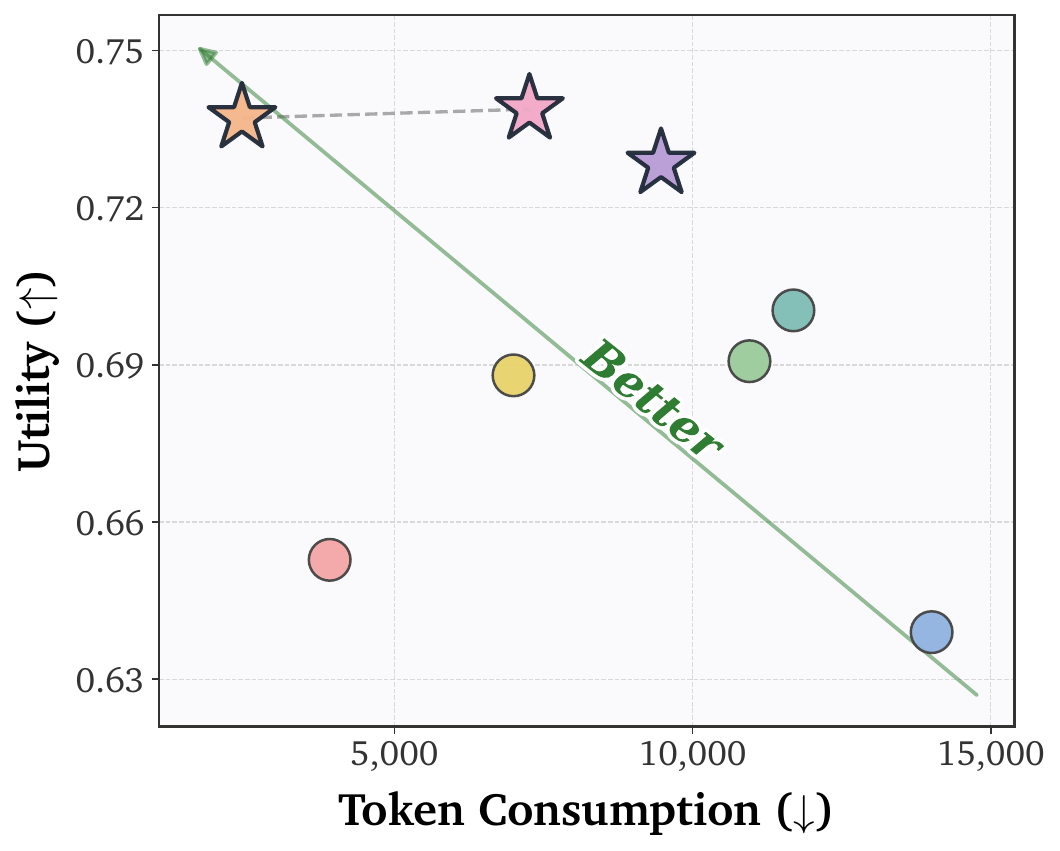}
    \caption{Business}
  \end{subfigure}\hfill
  \begin{subfigure}[t]{0.32\linewidth}
    \centering
    \includegraphics[width=\linewidth]{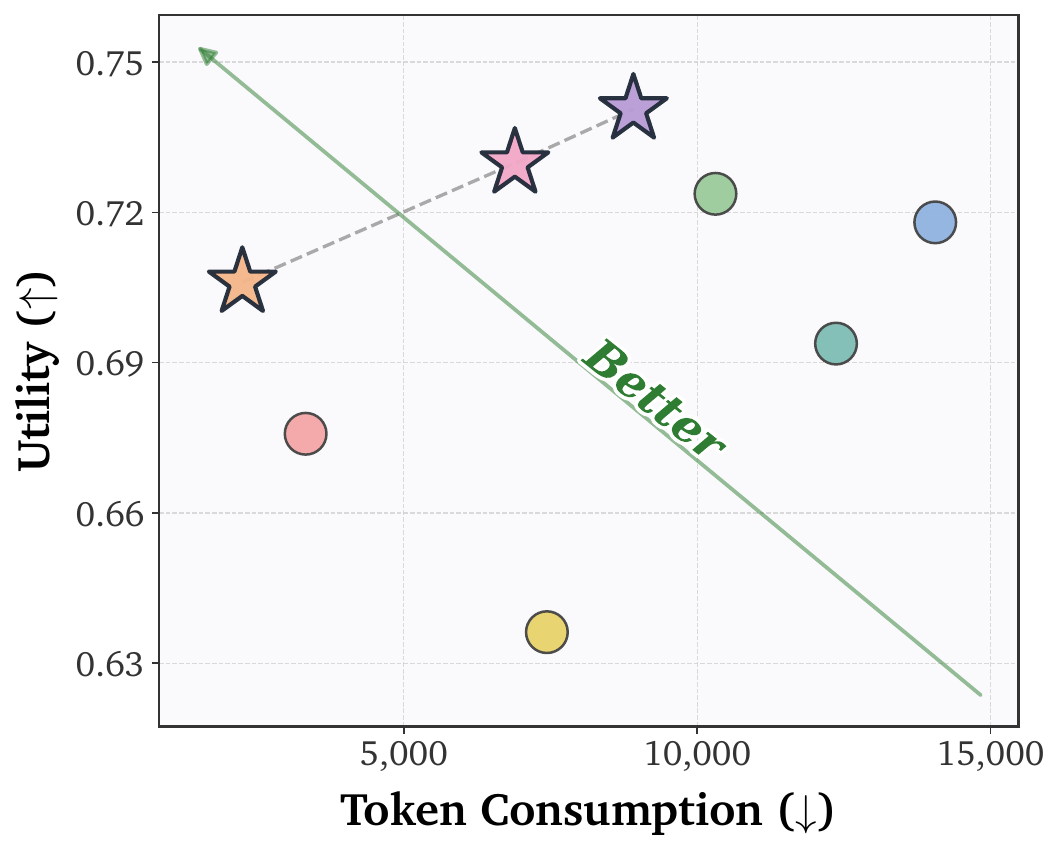}
    \caption{Infrastructure}
  \end{subfigure}

\caption{Utility vs.\ token consumption across nine scientific domains on \textit{EywaBench}. In every panel, the \textcolor[HTML]{2E7D32}{\textbf{green arrow}} indicates the preferred direction of the trade-off, pointing toward higher utility with fewer tokens (i.e., the upper-left of each plot); methods closer to the arrow head are more desirable. The grey dashed line traces the \emph{Pareto frontier}, connecting methods that are not dominated by any other method along this trade-off. Across all nine domains, our methods (EywaAgent, EywaMAS, and EywaOrchestra) consistently sit on or close to the Pareto frontier, achieving the strongest utility while spending substantially fewer tokens than competing multi-agent baselines.}
\label{fig:tradeoff_per_domain}
\vspace{-3mm}
\end{figure*}

To complement the overall utility-vs.-token plot in Figure~\ref{fig:tradeoff_per_domain} of the main text, we present the per-sub-domain trade-off in Figure~\ref{fig:tradeoff_per_domain}, covering all nine sub-domains across physical, life, and social science. Each panel places competing methods in a 2D space defined by token consumption (x-axis) and utility (y-axis), with the upper-left corner being the most desirable. We highlight three observations.
 
\textbf{All three \textit{Eywa} variants sit on or close to the Pareto frontier in every sub-domain.} Across all nine panels, \textit{EywaAgent}, \textit{EywaMAS}, and \textit{EywaOrchestra} are consistently among the points that define the Pareto frontier, while homogeneous and heterogeneous LLM-only multi-agent baselines (\textit{Refine}, \textit{Debate}, \textit{MoA}, \textit{X-MAS}) are typically dominated. For example, on Material, \textit{Debate}, \textit{MoA}, and \textit{X-MAS} all spend $14{,}000$--$25{,}000$ tokens while achieving lower utility than \textit{EywaMAS}; a similar pattern is seen on Energy, Biology, and Drug. This indicates that the \textit{Eywa} family does not trade tokens for utility. Instead, it pushes the Pareto frontier outward by routing computation to specialized foundation models rather than spending it on additional LLM exchanges.
 
\textbf{\textit{EywaAgent} is the most token-efficient point on the frontier.} Across the nine sub-domains, \textit{EywaAgent} uses on average $3{,}137$ tokens per task, compared with $4{,}469$ for the single-agent gpt-5-nano baseline and $8{,}673$--$16{,}537$ for the various MAS baselines. As a result, \textit{EywaAgent} typically anchors the lower-left region of the frontier: it delivers utility comparable to or above LLM-only MAS baselines while consuming a small fraction of their tokens. This is consistent with our claim that, when the foundation model already encodes the necessary domain prior, an additional debate or refinement loop among LLMs is largely redundant.
 
\textbf{\textit{EywaMAS} and \textit{EywaOrchestra} occupy complementary points on the frontier.} \textit{EywaMAS} pushes utility further by adding cross-modality multi-agent coordination, attaining the highest utility on most domains, at a moderate token cost that remains below all LLM-only MAS baselines. \textit{EywaOrchestra}, in contrast, dynamically selects the topology and the foundation model on a per-sample basis, which lets it reach utility close to or matching \textit{EywaMAS} while reducing average token consumption from $11{,}214$ to $8{,}335$ ($-26\%$). Moreover, \textit{EywaOrchestra} does not rely on expert configuration of the MAS. This positions \textit{EywaOrchestra} at the upper-left corner of the frontier, providing a favorable point in the trade-off when both quality and budget matter, and confirming that adaptive orchestration is an effective lever for navigating the utility-cost trade-off across heterogeneous scientific tasks.

\section{Prompt Templates}
\label{app:prompts}

This section provides the prompt templates used in \textit{EywaBench-V1} and \textit{EywaOrchestra}. 
% We separate the prompts into two categories. The first is a general task-execution prompt, which presents each benchmark instance to an agent. The second is an orchestration-planner prompt, which is used by \textit{EywaOrchestra} to select an execution configuration before solving the task. This separation decouples \emph{what problem should be solved} from \emph{how the problem should be routed and executed}.

\paragraph{General task-execution prompt.}
For all tasks in \textit{EywaBench-V1}, we adopt a unified task-execution template. The template specifies the task role, optional MCP-based model/tool context, structured input field, expected output size, and task-specific response format. This design provides a consistent interface across heterogeneous task modalities, while still allowing modality-specific information to be injected through specialized input tags and additional instructions. This prompt is illustrated in Prompt \ref{fig:prompt_template_general}.

\paragraph{Planner prompt for \textit{EywaOrchestra}.}
Unlike fixed agent baselines, \textit{EywaOrchestra} first invokes a planner to determine how a task should be executed. The planner receives a task description, domain, and task type, and returns a structured JSON configuration. This configuration specifies whether to use single-agent or multi-agent execution, whether to enable an \textit{Eywa}-augmented agent, which foundation model should be invoked when applicable, which multi-agent topology should be used, and how the participating agents should be instantiated.

The planner prompt formulates orchestration as a constrained configuration-generation problem rather than a direct task-solving problem. The hard JSON constraint ensures that planner outputs can be parsed automatically and passed to the downstream execution engine without manual intervention. Meanwhile, the explicit fields for \texttt{setting}, \texttt{multi\_agent\_type}, \texttt{foundation\_model}, and \texttt{agents} make the orchestration decision transparent and auditable.

% This design enables \textit{EywaOrchestra} to adapt across heterogeneous tasks. For language-centric reasoning problems, the planner can choose language-only execution or multi-agent collaboration. For structured prediction problems, the planner can route the task to an \textit{Eywa}-augmented agent equipped with an appropriate domain-specific foundation model. Thus, the planner prompt serves as the interface between high-level task diagnosis and executable heterogeneous-agent composition.

\begin{figure}[h]
\centering
\resizebox{\textwidth}{!}{
\begin{tcolorbox}[
colback=gray!5!white,
colframe=blue!70,
title=Prompt Template B.1: General Prompt Structure,
boxrule=0.3mm,
width=1.2\textwidth,
arc=2mm,
auto outer arc=true
]

\textbf{General Prompt.} For all tasks in \textit{EywaBench-V1}, we adopt a unified prompt template that specifies the task role, optional tool/model context, structured input field, expected output size, and task-specific output format. This design allows different task types to share the same high-level prompting interface while still preserving modality-specific instructions through specialized input tags and output constraints.

\begin{quote}
You are an expert in \texttt{\{task\}}. 
\texttt{\{mcp\_server\_description\}}. 
\texttt{\{additional\_instructions\}}

\medskip

\texttt{<\{input\_tag\}>}

\texttt{\{input\_data\}}

\texttt{</\{input\_tag\}>}

\medskip

\texttt{<output\_size>}

\texttt{\{output\_size\}}

\texttt{</output\_size>}

\medskip

\texttt{\{output\_format\}}
\end{quote}

\textbf{Template fields.}
The placeholder \texttt{\{task\}} specifies the task type.
The field \texttt{\{mcp\_server\_description\}} provides the available model/tool context when an external foundation model is accessible through the MCP server. 
The field \texttt{\{additional\_instructions\}} contains task-specific guidance, such as explaining masked targets in tabular prediction tasks. 
The pair of tags \texttt{<\{input\_tag\}>} and \texttt{</\{input\_tag\}>} explicitly marks the input. 
Finally, \texttt{\{output\_size\}} and \texttt{\{output\_format\}} define the expected output length and response schema.

\end{tcolorbox}}
\caption{
General prompt template used in \textit{EywaBench-V1}. The template separates task identity, model/tool context, modality-specific input, output size, and response format, enabling a unified prompting interface across heterogeneous scientific tasks.
}
\label{fig:prompt_template_general}
\end{figure}

\begin{figure}[h]
\centering
\begin{tcolorbox}[
colback=gray!5!white,
colframe=blue!70,
title=Prompt Template B.2: Planner Prompt for \textit{EywaOrchestra},
boxrule=0.3mm,
width=\textwidth,
arc=2mm,
auto outer arc=true
]

% \textbf{Planner Prompt Template.}

\footnotesize
\begin{verbatim}
You are an orchestration planner for the Eywa Agentic System.

Your job is to choose an execution configuration for a single task.

Available LLM models and descriptions:
- ......

Available foundation models and descriptions:
- ......

Supported multi-agent topology pool:
- ......

Input task:
- Task Description: {task_description}
- Domain: {domain}
- Task Type: {task_type}

Hard constraints:
- Output must be valid JSON only (no markdown, no code fence, no extra text).
- ......
- If "setting" is "single-agent":
  - "model" must be a valid model string.
  - "multi_agent_type" must be null.
  - "foundation_model" should be in the available foundation models or null.
  - "agents" must be an empty list [].
- If "setting" is "multi-agent":
  - "model" must be null.
  - "multi_agent_type" must be in the topology pool.
  - "foundation_model" should be in the available foundation models or null.
  - "agents" must be a non-empty list of valid agent specs.

Output format:
{
  "eywa": true or false,
  "setting": "single-agent" or "multi-agent",
  "model": <llm_model> or null,
  "multi_agent_type": <multi_agent_topology> or null,
  "foundation_model": <foundation_model> or null,
  "agents": [<agent_spec_1>, <agent_spec_2>, ...] or []
}
\end{verbatim}

\end{tcolorbox}
\caption{
Planner prompt template used by \textit{EywaOrchestra}. Given the task description, domain, and task type, the planner outputs a structured JSON configuration that specifies whether to use single-agent or multi-agent execution, whether to enable \textit{Eywa}, which foundation model to invoke, and how to instantiate the participating agents.
}
\label{fig:prompt_template_planner}
\end{figure}

\end{document}